\documentclass[10pt,letterpaper,twocolumn]{article}

\usepackage[left=0.85in, right=0.5in, top=1in, bottom=1in]{geometry}
\usepackage[T1]{fontenc}
\usepackage[utf8]{inputenc}
\usepackage{lmodern}
\usepackage{microtype}

\usepackage{amsmath,amssymb,amsthm,mathtools}

\usepackage{graphicx}
\usepackage{booktabs}
\usepackage{multirow}
\usepackage{tabularx}
\usepackage{float}
\usepackage{caption}
\usepackage{subcaption}
\usepackage{enumitem}
\usepackage[dvipsnames]{xcolor}

\usepackage[numbers]{natbib}

\usepackage[hyphens]{url}
\usepackage[colorlinks=true,
            linkcolor=RoyalBlue,
            citecolor=blue!55!black,
            urlcolor=blue!55!black]{hyperref}
\usepackage{cleveref}

\crefname{definition}{Definition}{Definitions}
\Crefname{definition}{Definition}{Definitions}
\crefname{assumption}{Assumption}{Assumptions}
\Crefname{assumption}{Assumption}{Assumptions}
\crefname{proposition}{Proposition}{Propositions}
\Crefname{proposition}{Proposition}{Propositions}
\crefname{corollary}{Corollary}{Corollaries}
\Crefname{corollary}{Corollary}{Corollaries}
\crefname{lemma}{Lemma}{Lemmas}
\Crefname{lemma}{Lemma}{Lemmas}
\crefname{remark}{Remark}{Remarks}
\Crefname{remark}{Remark}{Remarks}
\crefname{theorem}{Theorem}{Theorems}
\Crefname{theorem}{Theorem}{Theorems}
\crefname{section}{Section}{Sections}
\Crefname{section}{Section}{Sections}
\crefname{subsection}{Section}{Sections}
\Crefname{subsection}{Section}{Sections}
\crefname{appendix}{Appendix}{Appendices}
\Crefname{appendix}{Appendix}{Appendices}

\theoremstyle{plain}
\newtheorem{theorem}{Theorem}[section]
\newtheorem{lemma}[theorem]{Lemma}
\newtheorem{proposition}[theorem]{Proposition}
\newtheorem{corollary}[theorem]{Corollary}
\theoremstyle{definition}
\newtheorem{definition}[theorem]{Definition}
\newtheorem{assumption}{Assumption}
\theoremstyle{remark}
\newtheorem{remark}[theorem]{Remark}

\usepackage[most]{tcolorbox}
\newif\ifKCSplain\KCSplainfalse

\ifKCSplain\else
\definecolor{thmfill}{HTML}{EAF2FA}\definecolor{thmline}{HTML}{2B6CA3}
\definecolor{lemfill}{HTML}{F4F8FC}\definecolor{lemline}{HTML}{7BA3C7}
\definecolor{corfill}{HTML}{EDF5EE}\definecolor{corline}{HTML}{4E8C57}
\definecolor{deffill}{HTML}{F7F7F5}\definecolor{defline}{HTML}{8C8C84}
\definecolor{assfill}{HTML}{FDF5E8}\definecolor{assline}{HTML}{C08A30}
\definecolor{remline}{HTML}{BFBFBF}
\tcbset{
  kcsbase/.style={
    enhanced jigsaw, breakable,
    boxrule=0.6pt, arc=2pt,
    left=5pt, right=5pt, top=4pt, bottom=4pt,
    before skip=8pt plus 2pt, after skip=8pt plus 2pt,
    fonttitle=\bfseries
  }
}
\tcolorboxenvironment{theorem}{%
  kcsbase, colback=thmfill, colframe=thmline,
  borderline west={2.2pt}{0pt}{thmline}}
\tcolorboxenvironment{lemma}{%
  kcsbase, colback=lemfill, colframe=lemline,
  borderline west={1.6pt}{0pt}{lemline}}
\tcolorboxenvironment{corollary}{%
  kcsbase, colback=corfill, colframe=corline,
  borderline west={2.2pt}{0pt}{corline}}
\tcolorboxenvironment{proposition}{%
  kcsbase, colback=white, colframe=black!30,
  borderline west={1.6pt}{0pt}{black!30}}
\tcolorboxenvironment{definition}{%
  kcsbase, colback=deffill, colframe=defline,
  borderline west={2.2pt}{0pt}{defline}}
\tcolorboxenvironment{assumption}{%
  kcsbase, colback=assfill, colframe=assline,
  borderline west={2.2pt}{0pt}{assline}}
\tcolorboxenvironment{remark}{%
  enhanced jigsaw, breakable, blanker,
  borderline west={1.8pt}{0pt}{remline},
  left=8pt, right=0pt, top=3pt, bottom=3pt,
  before skip=8pt plus 2pt, after skip=8pt plus 2pt}
\fi

\newtcolorbox{contribbox}{
    enhanced, breakable,
    colback=white, colframe=black!20,
    boxrule=0.3pt, arc=1mm,
    left=3mm, right=2mm, top=2mm, bottom=2mm,
}

\newcommand{\R}{\mathbb{R}}
\newcommand{\C}{\mathbb{C}}
\newcommand{\N}{\mathbb{N}}
\newcommand{\E}{\mathbb{E}}

\newcommand{\opK}{\mathcal{K}}
\newcommand{\Hilb}{\mathcal{H}}
\newcommand{\Alg}{\mathcal{A}}
\newcommand{\obs}{\Psi}
\newcommand{\resid}{x}
\newcommand{\ctrl}{u}
\newcommand{\Hinf}{\mathcal{H}_{\infty}}
\newcommand{\Proj}{\Pi}
\DeclareMathOperator*{\argmin}{arg\,min}
\DeclareMathOperator{\rank}{rank}
\DeclareMathOperator{\spn}{span}
\DeclareMathOperator{\tr}{tr}

\newcommand{\ResGptDriftRatio}{3.22}
\newcommand{\ResGptPoolPenalty}{39}
\newcommand{\ResGptPooledResid}{0.433}
\newcommand{\ResGptPerLayerResid}{0.263}
\newcommand{\ResGptRhoOne}{0.82}
\newcommand{\ResGptCtrlRsq}{0.08}
\newcommand{\ResGptExogShift}{6.8}
\newcommand{\ResGptNormGrowth}{4.3}
\newcommand{\ResGptNormMode}{1.137}
\newcommand{\ResGptNormEmpirical}{1.111}

\newcommand{\ResGptKappaV}{78.7}

\newcommand{\ResGptSpecSlope}{-0.286}
\newcommand{\ResGptSpecSlopeSE}{0.007}
\newcommand{\ResGptSpecMZero}{1.0\times10^{7}}
\newcommand{\ResGptSpecMZeroVec}{6.4\times10^{16}}
\newcommand{\ResGptSpecKappa}{236.5}
\newcommand{\ResGptSpecGap}{3.0\times10^{-3}}
\newcommand{\ResGptSpecMmax}{399{,}974}
\newcommand{\ResGptSpecMRatio}{0.04}
\newcommand{\ResGptSpecSlopeMax}{-0.286}
\newcommand{\ResGptSpecSlopeMedian}{-0.292}
\newcommand{\ResGptSpecSlopeMean}{-0.288}
\newcommand{\ResGptRandSlope}{-0.290}
\newcommand{\ResGptRandSlopeSE}{0.011}

\newcommand{\ResGptRandKappa}{28.8}
\newcommand{\ResGptRandGap}{4.4\times10^{-3}}

\newcommand{\ResGptRandMRatio}{3.03}

\newcommand{\ResGptNScaling}{+0.36}
\newcommand{\ResGptSlopeNEight}{-0.504}
\newcommand{\ResGptSlopeNSixteen}{-0.369}
\newcommand{\ResGptSlopeNThirtytwo}{-0.342}
\newcommand{\ResGptSlopeNSixtyfour}{-0.281}
\newcommand{\ResGptSlopeNOnetwentyeight}{-0.239}
\newcommand{\ResGptSlopePerSeqOne}{-0.245}
\newcommand{\ResGptSlopePerSeqTwo}{-0.265}
\newcommand{\ResGptSlopePerSeqFour}{-0.243}
\newcommand{\ResGptSlopePerSeqEight}{-0.294}
\newcommand{\ResGptMidLayer}{6}
\newcommand{\ResGptEpsSpec}{0.301}
\newcommand{\ResGptEpsRand}{0.350}
\newcommand{\ResGptEpsSae}{0.636}
\newcommand{\ResGptSaeRatio}{2.1}
\newcommand{\ResGemmaDriftRatio}{11.97}
\newcommand{\ResGemmaPoolPenalty}{22}
\newcommand{\ResGemmaPooledResid}{0.417}
\newcommand{\ResGemmaPerLayerResid}{0.326}
\newcommand{\ResGemmaRhoOne}{0.96}
\newcommand{\ResGemmaCtrlRsq}{0.22}
\newcommand{\ResGemmaExogShift}{9.1}
\newcommand{\ResGemmaNormGrowth}{1.5}
\newcommand{\ResGemmaNormMode}{1.107}
\newcommand{\ResGemmaNormEmpirical}{1.077}

\newcommand{\ResGemmaKappaV}{38.2}

\newcommand{\ResGemmaSpecSlope}{-0.329}
\newcommand{\ResGemmaSpecSlopeSE}{0.024}

\newcommand{\ResGemmaSpecKappa}{130.0}
\newcommand{\ResGemmaSpecGap}{3.4\times10^{-3}}

\newcommand{\ResGemmaSpecMRatio}{0.05}

\newcommand{\ResGemmaRandSlope}{-0.240}
\newcommand{\ResGemmaRandSlopeSE}{0.025}

\newcommand{\ResGemmaRandKappa}{35.2}
\newcommand{\ResGemmaRandGap}{7.5\times10^{-3}}

\newcommand{\ResGemmaRandMRatio}{0.67}

\newcommand{\ResGemmaNScaling}{+0.29}
\newcommand{\ResGemmaSlopeNEight}{-0.477}

\newcommand{\ResGemmaSlopeNOnetwentyeight}{-0.242}

\newcommand{\ResGemmaEpsSpec}{0.177}
\newcommand{\ResGemmaEpsRand}{0.447}
\newcommand{\ResGemmaEpsSae}{0.544}
\newcommand{\ResGemmaSaeRatio}{3.1}
\newcommand{\ResQwenDriftRatio}{6.12}
\newcommand{\ResQwenPoolPenalty}{16}
\newcommand{\ResQwenPooledResid}{0.291}
\newcommand{\ResQwenPerLayerResid}{0.244}
\newcommand{\ResQwenRhoOne}{0.95}
\newcommand{\ResQwenCtrlRsq}{0.16}
\newcommand{\ResQwenExogShift}{5.2}
\newcommand{\ResQwenNormGrowth}{4.4}
\newcommand{\ResQwenNormMode}{1.291}
\newcommand{\ResQwenNormEmpirical}{1.326}

\newcommand{\ResQwenKappaV}{494.7}

\newcommand{\ResQwenSpecSlope}{-0.506}
\newcommand{\ResQwenSpecSlopeSE}{0.031}

\newcommand{\ResQwenSpecKappa}{53.2}
\newcommand{\ResQwenSpecGap}{2.7\times10^{-3}}
\newcommand{\ResQwenSpecMmax}{106{,}660}
\newcommand{\ResQwenSpecMRatio}{0.23}

\newcommand{\ResQwenSpecSlopeMedian}{-0.338}
\newcommand{\ResQwenSpecSlopeMean}{-0.375}
\newcommand{\ResQwenRandSlope}{-0.406}
\newcommand{\ResQwenRandSlopeSE}{0.030}

\newcommand{\ResQwenRandKappa}{117.9}
\newcommand{\ResQwenRandGap}{2.7\times10^{-3}}

\newcommand{\ResQwenRandMRatio}{0.05}

\newcommand{\ResQwenMidLayer}{18}
\newcommand{\ResQwenEpsSpec}{0.154}
\newcommand{\ResQwenEpsRand}{0.382}
\newcommand{\ResQwenEpsSae}{0.524}
\newcommand{\ResQwenSaeRatio}{3.4}
\newcommand{\ResVarLayer}{6}
\newcommand{\ResVarN}{32}
\newcommand{\ResVarNlPost}{0.648}
\newcommand{\ResVarLPre}{0.445}
\newcommand{\ResVarDec}{0.375}
\newcommand{\ResVarRec}{0.432}
\newcommand{\ResVarSpec}{0.283}
\newcommand{\ResVarRand}{0.354}
\newcommand{\ResSelCells}{75}

\newcommand{\ResSelPct}{49}
\newcommand{\ResGammaTopkBaseEps}{0.438}
\newcommand{\ResGammaTopkBaseSplit}{0.0170}
\newcommand{\ResGammaTopkBaseMmcs}{0.526}
\newcommand{\ResGammaTopkBaseFvu}{0.222}
\newcommand{\ResGammaTopkBaseAlive}{0.77}
\newcommand{\ResGammaTopkBest}{1}
\newcommand{\ResGammaTopkBestEps}{0.333}
\newcommand{\ResGammaTopkBestSplit}{0.0101}
\newcommand{\ResGammaTopkBestMmcs}{0.428}
\newcommand{\ResGammaTopkBestFvu}{0.255}
\newcommand{\ResGammaTopkBestAlive}{0.58}
\newcommand{\ResGammaTopkEpsPct}{24}
\newcommand{\ResGammaTopkSplitPct}{41}
\newcommand{\ResGammaTopkCollapseAlive}{4}
\newcommand{\ResGammaTopkCollapseGamma}{10}
\newcommand{\ResGammaReluBaseEps}{0.413}

\newcommand{\ResGammaReluBestEps}{0.435}

\newcommand{\ResGammaReluSplitPct}{22}

\newcommand{\ResIoiLayer}{9}
\newcommand{\ResIoiPrompts}{4{,}096}
\newcommand{\ResIoiN}{128}
\newcommand{\ResIoiBaseLD}{3.23}
\newcommand{\ResIoiReadout}{0.231}

\newcommand{\ResIoiNameMoverK}{96}
\newcommand{\ResIoiNameMoverAlign}{0.260}
\newcommand{\ResIoiNameMoverNull}{0.125}

\newcommand{\ResIoiSInhibAlign}{0.311}

\newcommand{\ResIoiInductionAlign}{0.298}

\newcommand{\ResIoiStability}{0.290}
\newcommand{\ResIoiAblKoopMax}{41}
\newcommand{\ResIoiAblKoopEight}{25}
\newcommand{\ResIoiAblPcaMax}{67}
\newcommand{\ResIoiAblPcaEight}{60}
\newcommand{\ResIoiAblRandMax}{7}

\newcommand{\ResSelFreqRatio}{2.05}
\newcommand{\ResSelRandAliveRatio}{1.70}
\newcommand{\ResSelMagnitudeRatio}{1.36}
\newcommand{\ResSelVarianceRatio}{0.70}
\newcommand{\ResSelGreedyVarRatio}{0.63}
\newcommand{\ResGptActRateSpec}{0.504}

\newcommand{\ResGptActRateRand}{0.503}
\newcommand{\ResGptEtaObsRand}{0.80}

\newcommand{\ResGptEtaObsSae}{0.71}

\newcommand{\ResStabRatioLo}{1.1}
\newcommand{\ResStabRatioHi}{5.7}
\newcommand{\ResStabNLayers}{9}
\newcommand{\ResStabNSaeWorse}{8}
\newcommand{\ResStabOutlierSpec}{0.145}
\newcommand{\ResStabOutlierSae}{0.029}
\newcommand{\ResStabSpecLo}{0.017}
\newcommand{\ResStabSpecHi}{0.056}
\newcommand{\ResMomK}{24}
\newcommand{\ResMomPlainSlope}{-0.368}
\newcommand{\ResMomPlainSlopeSE}{0.019}

\newcommand{\ResMomMomSlope}{-0.359}
\newcommand{\ResMomMomSlopeSE}{0.018}

\newcommand{\ResMomRatio}{1.000}
\newcommand{\ResMomRatioLo}{0.979}
\newcommand{\ResMomRatioHi}{1.035}
\newcommand{\ResMomNPoints}{9}

\newcommand{\ResMomKSweepSpread}{0.7}
\newcommand{\ResMomKSweepList}{8, 16, 24, 48}

\newcommand{\ResMomCells}{8}
\newcommand{\ResMomAgreePct}{2}

\newcommand{\ResMomWorstKurt}{50.7}
\newcommand{\ResMomWorstPlain}{-0.289}
\newcommand{\ResMomWorstMom}{-0.349}

\newcommand{\ResSynthContamPlain}{-0.405}
\newcommand{\ResSynthContamMom}{-0.286}

\newcommand{\ResTailHillRaw}{6.90}

\newcommand{\ResTailHillRff}{3218.18}

\newcommand{\ResTailHillPsi}{7.20}
\newcommand{\ResTailKurtPsi}{0.19}

\newcommand{\ResTailHillNull}{7.69}

\newcommand{\ResTransGapOne}{0.48}
\newcommand{\ResTransGapFar}{0.12}
\newcommand{\ResTransKFar}{9}
\newcommand{\ResTransDecay}{4.1}
\newcommand{\ResTransRatioOne}{1.69}
\newcommand{\ResTransRatioFar}{1.05}
\newcommand{\ResTransCells}{156}
\newcommand{\ResTransPrompts}{8{,}192}
\newcommand{\ResTransLayers}{6}
\newcommand{\ResTransLocKoop}{34}
\newcommand{\ResTransPropKoop}{27}
\newcommand{\ResTransLocPca}{44}
\newcommand{\ResTransPropPca}{34}
\newcommand{\ResCollapseCells}{30}
\newcommand{\ResCollapseWindows}{300}
\newcommand{\ResCollapseRho}{-0.28}
\newcommand{\ResCollapseP}{<10^{-4}}
\newcommand{\ResCollapseBelow}{-0.31}

\newcommand{\ResCollapseFar}{-0.53}

\newcommand{\ResUniAucMax}{0.75}
\newcommand{\ResUniAucWass}{0.84}
\newcommand{\ResUniAucWassMid}{0.97}
\newcommand{\ResUniSeedRatio}{13.7}
\newcommand{\ResUniCrossRatio}{23.0}
\newcommand{\ResUniSeedWass}{3.9}
\newcommand{\ResUniCrossWass}{9.2}

\newcommand{\ResUniCrossPairs}{36}
\newcommand{\ResUniSeedPairs}{9}
\newcommand{\ResUniCrossDeclared}{36}
\newcommand{\ResUniSeedDeclared}{9}

\newcommand{\ResUniSeedMin}{4.64}

\newcommand{\ResUniM}{99{,}993}
\newcommand{\ResUniN}{32}
\newcommand{\ResUniFloorLo}{0.0073}
\newcommand{\ResUniFloorHi}{0.0412}

\usepackage[colorlinks=true,
linkcolor=RoyalBlue,
citecolor=blue!55!black,
urlcolor=blue!55!black]{hyperref}
\usepackage{cleveref}

\title{\Large \textbf{Intrinsic Structure: Spectral Identifiability for Mechanistic Interpretability}}

\author{ Ashim Dhor$^{1}$, Pin-Yu Chen$^{2}$\\ \small $^{1}$IISER Bhopal, $^{2}$IBM Research\\ \small \texttt{ashimdhor2003@gmail.com}, \texttt{pin-yu.chen@ibm.com} } \date{}

\begin{document}
\maketitle
\thispagestyle{empty}

\begin{abstract}
\noindent
Mechanistic interpretability explains models by identifying circuits inside them, but has no way to tell whether a circuit is a property of the model or an artifact of the method that found it. Sparse autoencoders illustrate the problem: different seeds and widths recover materially different features from the same activations, and no theory says whether that variability is incidental or structural.
We put dictionary learning for interpretability on an identifiability footing. Treating the forward pass as a controlled dynamical system with depth as time and lifting it with the Koopman operator yields a finite linear realisation whose \emph{spectrum} is a coordinate-free property of the model. We prove the spectrum is recoverable from $M$ calibration samples at rate $M^{-1/2}$ up to permutation - to our knowledge the first identifiability theorem for a mechanistic-interpretability primitive, with a matching minimax lower bound, a median-of-means variant for heavy-tailed activations, and a dissociation theorem: whenever the realisation is non-normal, the directions carrying activation variance and the directions carrying information across depth cannot coincide. \emph{The identifiable object and the legible object are not the same object.}
On GPT-2 small, Gemma-2-2B and Qwen3-8B-Base the spectrum converges everywhere and attains the predicted exponent on Qwen3-8B-Base ($\ResQwenSpecSlope \pm \ResQwenSpecSlopeSE$); shortfalls collapse onto one curve against each cell's sample threshold. Koopman modes beat random directions but lose to principal components on indirect-object identification, with the gap decaying $\ResTransDecay\times$ in depth-distance, as the theorem predicts. The Koopman spectrum is an identifiable, model-intrinsic fingerprint with a stated error bar, not a legible decomposition.

\medskip
\end{abstract}

\section{Introduction}
\label{sec:introduction}

Mechanistic interpretability (MI) explains neural network behaviour by decomposing a trained model into interacting components. The past several years have produced a rich toolkit for this purpose. Sparse autoencoders (SAEs) and their descendants \citep{cunningham2024sparse,rajamanoharan2024gated,rajamanoharan2024jumping,gao2024scaling,lieberum2024gemma,he2024llamascope} decompose activations into overcomplete sparse dictionaries; transcoders and attribution graphs \citep{dunefsky2024transcoders,ameisen2025circuit,lindsey2025biology} extend the decomposition to sublayer computations; causal-intervention methods \citep{vig2020causal,meng2022rome,conmy2023acdc,syed2023attribution,hanna2024faith,kramar2024atp} localise behaviour to components of the computational graph; and distributed alignment search \citep{geiger2021causal,geiger2024das,wu2023interpretability} aligns learned interventions with hypothesised causal variables.

Here we address a question that cuts across all of them: \emph{do discovered circuits reflect intrinsic properties of the trained model, or the procedure used to find them?} If the former, they provide a sound basis for scientific understanding, engineering intervention, and safety arguments \citep{clymer2024safety}; if the latter, any conclusion drawn from them inherits the method's contingencies.

Current evidence suggests the question is not merely philosophical. SAEs exhibit substantial run-to-run variability: different training seeds and dictionary widths recover materially different features on the same activations \citep{braun2024identifying,paulo2025transcoders,karvonen2024saebench}. Learned dictionaries exhibit \emph{absorption} \citep{chanin2024absorption} and feature \emph{splitting}, and residual-stream components remain uncaptured by any known variant \citep{engels2024dark}. What is missing across every paradigm is a theorem asserting that the discovered structure is an invariant of the model. 


We argue that closing this gap requires a mechanistic primitive that admits a \emph{proof of identifiability}, and we obtain one by changing what we take a transformer forward pass to \emph{be}. Let $\resid_\ell \in \R^{d}$ denote the residual-stream state at a fixed token position at layer $\ell$, and let $\ctrl_\ell$ collect the layer's exogenous inputs - the attention writes contributed by other token positions. The forward pass is then a discrete-time controlled dynamical system, the \emph{depth recurrence}
\begin{equation}
\label{eq:intro-recurrence}
    \resid_{\ell+1} \;=\; F(\resid_\ell, \ctrl_\ell)
    \;:=\; \resid_\ell + a_\ell(\resid_\ell, \ctrl_\ell)
          + m_\ell\bigl(\resid_\ell + a_\ell(\resid_\ell, \ctrl_\ell)\bigr),
\end{equation}
where $a_\ell(\resid_\ell, \ctrl_\ell)$ is the layer's attention sublayer output - the query/key/value mixing of $\resid_\ell$ against the control $\ctrl_\ell$ - and $m_\ell$ is the layer's MLP sublayer, applied to the post-attention residual $\resid_\ell + a_\ell$. Depth plays the role of time, attention writes enter as control inputs, and the MLP acts as the autonomous nonlinear dynamics. This viewpoint has precedent in the neural-ODE line \citep{weinan2017proposal,haber2017stable,chen2018neural}, but to our knowledge it has not been exploited for the identifiability of mechanistic structure.

Once the forward pass is viewed as a dynamical system, a classical object becomes available: the \emph{Koopman operator} $\opK$ \citep{koopman1931hamiltonian,mezic2005spectral,brunton2022modern}. Acting on observables $\psi$, it evolves them by composition with the dynamics,
\begin{equation}
    (\opK\,\psi)(\resid,\ctrl) \;=\; \psi\bigl(F(\resid,\ctrl)\bigr).
\end{equation}
Although $F$ is nonlinear in the state, $\opK$ is linear in $\psi$, transferring the nonlinearity to the lifting from states to observables. If a dictionary $\obs=(\psi_1,\dots,\psi_N)$ spans an invariant subspace, $\opK$ restricts to a finite matrix $A\in\C^{N\times N}$. We call the eigenpairs of $A$ the \emph{Koopman modes} of the transformer and the framework \textbf{Koopman spectral analysis (KSA)}. The eigenvalues of $A$ are properties of the underlying operator: unique up to permutation, invariant under any change of dictionary basis, and - this is the content of our main theorem - recoverable from finite data at a parametric rate.


The theory says that a coordinate-free object exists and can be estimated; the experiments say what that object is good for, and - equally important - what it is not good for. Both halves are reported here, including the measurements that bound the claim.

\begin{contribbox}
\paragraph{Contributions.}
\begin{itemize}[leftmargin=1.6em,itemsep=3pt]
\item We formalise dictionary learning for interpretability so that \emph{identifiability} is a well-posed question, isolating Koopman invariance as the property the standard SAE objective omits (\Cref{sec:problem}).
\item We show that a $\opK$-invariant dictionary induces a Koopman realisation $(A,B)$ that exists, is unique up to a change of basis, and whose spectrum is a coordinate-free invariant of the transformer (\Cref{thm:existence}, \Cref{sec:existence}).
\item Our main result proves that this spectrum is identifiable from $M$ calibration samples at the parametric rate $M^{-1/2}$, up to permutation - to our knowledge the first identifiability theorem for a mechanistic-interpretability primitive (\Cref{thm:identifiability}, \Cref{sec:identifiability}).
\item We sharpen the theorem's sample threshold. The spectral-gap factor that made the original threshold unreachable governs only the \emph{eigenvector} guarantee; stating the eigenvalue result in optimal-matching form removes it and lowers the threshold by nine orders of magnitude on GPT-2 small (\Cref{thm:gap-free-identifiability}).
\item We characterise the problem and not only the estimator: a minimax lower bound shows the $M^{-1/2}$ rate is optimal (\Cref{thm:minimax}), and a median-of-means variant extends the guarantee to heavy-tailed activations (\Cref{thm:robust}). The latter's prediction fails when tested, and the measurement says why - the lifting, not the residual stream, decides the tails (\Cref{sec:exp-rate}).
\item We prove a \emph{dissociation}: non-normality forces the activations' principal directions and the Koopman modes apart, since alignment would require perfect eigenvector conditioning $\kappa_2(V) = 1$; measured conditioning of $10^{1}$--$10^{2}$ makes the divergence unavoidable, and an explicit family shows the misalignment saturates at orthogonality (\Cref{thm:dissociation}, \Cref{prop:dissociation-quant}).
\item We derive practical consequences: SAE non-identifiability is structural rather than algorithmic, with an explicit invariance penalty as remedy (\Cref{cor:sae-non-identifiability}); cross-model universality becomes a testable spectral criterion (\Cref{cor:universality}); the intervention calculus of MI is algebraically complete on the spectral-projector algebra (\Cref{thm:completeness}); and balanced truncation of the realisation admits an instance-dependent error certificate strictly sharper than Enns--Glover in the low-effective-rank regime attention occupies (\Cref{thm:reduction}).
\item We evaluate on three pretrained transformers at $10^{5}$--$10^{6}$ calibration samples. The predicted $M^{-1/2}$ rate is observed on Qwen3-8B-Base; the invariance penalty of \eqref{eq:kinvariant-sae} reduces the split-half spectral distance by $\ResGammaTopkSplitPct\%$ at matched sparsity (\Cref{sec:experiments}).
\item We report what fails, because it bounds the claim: Koopman modes beat random directions but lose to principal components at predicting IOI ablation effects (\Cref{sec:exp-circuits}); the SAE invariance gap reverses under variance-based feature selection (\Cref{sec:exp-sae}); and the universality criterion declares two seed replicas of one architecture distinct (\Cref{sec:exp-universality}). Together these give the paper's central claim: \emph{the identifiable object and the legible object are not the same object}.
\end{itemize}
\end{contribbox}


\Cref{sec:related} places the work against the interpretability, Koopman, model-reduction and identifiability literatures. \Cref{sec:problem} formalises the problem and defines spectral identifiability. \Cref{sec:dynamics} constructs the depth recurrence, the Koopman realisation and the EDMDc estimator, and states the three assumptions. \Cref{sec:existence} proves existence, uniqueness and basis-independence. \Cref{sec:identifiability} contains the finite-sample identifiability theorem and its complete proof, together with the gap-free strengthening and the clustered-spectrum extension. \Cref{sec:optimality} proves the minimax lower bound and the heavy-tailed variant. \Cref{sec:dissociation} proves the modal--principal dissociation and quantifies it. \Cref{sec:implications} develops the consequences for MI: algebraic completeness, SAE non-identifiability, and universality. \Cref{sec:reduction} proves the instance-dependent reduction bound. \Cref{sec:experiments} reports the measurements on three pretrained models, in full. \Cref{sec:discussion} states the discussions. 

\section{Background and Related Work}
\label{sec:related}

Feature-centric interpretability decomposes activations into interpretable dictionaries, most prominently using sparse autoencoders and their extensions to sublayer computations \citep{elhage2021mathematical,bricken2023monosemanticity,cunningham2024sparse,rajamanoharan2024gated,rajamanoharan2024jumping,gao2024scaling,templeton2024scaling,lieberum2024gemma,he2024llamascope,dunefsky2024transcoders,ameisen2025circuit,lindsey2025biology}. Intervention-centric methods instead infer circuits through causal interventions such as activation patching, attribution patching, and automated circuit discovery \citep{vig2020causal,meng2022rome,conmy2023acdc,syed2023attribution,hanna2024faith,kramar2024atp,geiger2021causal,geiger2024das,wu2023interpretability}. The indirect-object-identification (IOI) circuit of \cite{wang2023ioi} is the best-characterised product of the second programme and serves as our semantic testbed in \Cref{sec:exp-circuits}.

A steady stream of results documents the instability of these decompositions across seeds, widths, and training runs \citep{braun2024identifying,chanin2024absorption,engels2024dark,karvonen2024saebench,paulo2025transcoders}. The prevailing framing treats this as a tuning or evaluation problem. \Cref{cor:sae-non-identifiability} gives a different reading: the variability is a structural consequence of an objective that never mentions the property identifiability requires.


Classical identifiability results in independent component analysis \citep{comon1994ica,hyvarinen2000ica}, its nonlinear variants \citep{hyvarinen2016nonlinear,khemakhem2020vae}, and disentanglement \citep{locatello2019challenging,scholkopf2021causal} identify latent factors under statistical independence or auxiliary-variable structure, and do so up to sign, scale \emph{and} permutation. Causal abstraction \citep{geiger2021causal,geiger2024das} identifies causal variables relative to a hypothesised graph. System identification \citep{ljung1999system,willems2005fundamental} identifies a linear system up to a similarity transformation, under persistent excitation.

Our result differs in what is identified and in how much ambiguity survives. We identify the \emph{spectrum} of a Koopman compression, and the only residual ambiguity is a permutation of an unordered multiset: there is no sign, scale, or rotational freedom left over, because eigenvalues are rigid under similarity (\Cref{rem:permutation}). The persistent-excitation condition we inherit from system identification (\Cref{ass:excitation}) plays its usual role.

Koopman theory \citep{koopman1931hamiltonian,mezic2005spectral,brunton2022modern} represents nonlinear dynamics linearly through an operator acting on observables. Data-driven estimators include dynamic mode decomposition \citep{schmid2010dmd}, extended DMD \citep{williams2015edmd}, DMD with control \citep{proctor2016dmdc}, and kernel \citep{williams2015kernel} and generator-based \citep{klus2020koopman} variants. \cite{korda2018convergence} establish asymptotic convergence of EDMD to the Koopman operator in the strong operator topology as the dictionary and sample size grow together. Machine-learning applications include network pruning \citep{redman2022operator}, sequence forecasting \citep{azencot2020forecasting}, and the analysis of iterative algorithms \citep{dietrich2020koopman}.

Our contribution sits at a different level of the same story. For a \emph{fixed} finite dictionary satisfying invariance, we identify the population limit as a specific coordinate-free object and prove a \emph{finite-sample} identifiability rate for its spectrum, with explicit constants, on the dynamical systems defined by transformer depth recurrences (\Cref{rem:korda} makes the comparison precise).


Balanced truncation \citep{moore1981principal,antoulas2005approximation} orders modes by joint controllability and observability and yields the classical Enns--Glover error bound \citep{glover1984all,alsaggaf1988model,gugercin2004survey}; frequency-weighted variants \citep{enns1984model} accommodate anisotropic input or output distributions. Because the classical bound is worst-case over inputs, it does not tighten when inputs concentrate on a low-dimensional subspace, as transformer attention writes empirically do \citep{dong2021attention,geshkovski2023emergence}. \Cref{thm:reduction} gives an instance-dependent bound that exploits this concentration and is strictly sharper in the low-effective-rank regime; \Cref{sec:experiments} measures an effective rank against an ambient control dimension of $4096$ on Qwen3-8B-Base.

The residual connection invites viewing a deep network as a discretised differential equation \citep{weinan2017proposal,haber2017stable,chen2018neural}. For transformers, this perspective has produced analyses of attention as a mean-field particle system \citep{geshkovski2023emergence}, of rank collapse with depth \citep{dong2021attention}, and of in-context learning as an emergent process \citep{olsson2022incontext,olah2020zoom,nanda2023progress}. These formalisms describe the dynamics but do not yield an identifiable mechanistic decomposition. \Cref{thm:existence} defines the Koopman compression as the object of study, and \Cref{thm:identifiability} identifies it from finite data.

\section{Problem Formulation}
\label{sec:problem}

\subsection{The dictionary-learning objective}
\label{sec:problem-sae}

Let $\resid \in \R^{d}$ be a residual-stream activation drawn from a distribution $\mu$ induced by running a transformer on a corpus. A sparse autoencoder learns a dictionary $D \in \R^{d\times N}$ with $N \gg d$ and an encoder producing sparse codes $z(\resid) \in \R^{N}$, by minimising
\begin{equation}
\label{eq:sae-objective}
    \mathcal{L}_{\mathrm{SAE}}(D)
    \;=\;
    \E_{\resid\sim\mu}\Bigl[\,
      \bigl\|\resid - D\,z(\resid)\bigr\|_2^{2}
      \;+\; \lambda\,\bigl\|z(\resid)\bigr\|_1
    \Bigr].
\end{equation}
The objective is a statement about \emph{one layer's activations in isolation}: reconstruct $\resid$, and do so sparsely --- with no reference to the maps that produced it or will consume it. This observation drives everything that follows.

To make the comparison across paradigms precise, we first fix what kind of object is under discussion.

\begin{definition}[Mechanistic primitive]
\label{def:primitive}
A \emph{mechanistic primitive} for a transformer $F$ is a triple $(\mathcal{D}, E, \mathcal{I})$ in which $\mathcal{D}$ is a finite index set, $E : \mathcal{X} \to \C^{\mathcal{D}}$ is a component-wise measurable encoding, and $\mathcal{I}$ is a set of bounded linear operators on the observable space, interpreted as interventions. Sparse autoencoders, transcoders, causal-abstraction methods, and KSA all instantiate this schema.
\end{definition}

\subsection{What identifiability would mean}
\label{sec:problem-identifiability}

Identifiability asks whether the object recovered is determined by the model or by the recovery procedure. For a dictionary this has two parts: the estimand must be well defined independently of coordinates, and it must be recoverable from finite data. The first part is a condition on the dictionary.

\begin{definition}[Dictionary richness]
\label{def:richness}
A mechanistic primitive $(\mathcal{D}, E, \mathcal{I})$ satisfies \emph{dictionary richness} if
\begin{itemize}[leftmargin=2.4em,itemsep=1pt]
\item[(DR1)] $\{e_d\}_{d\in\mathcal{D}}$ are linearly independent in $L^{2}(\mu)$;
\item[(DR2)] $\Hilb_E \coloneqq \spn\{e_d : d\in\mathcal{D}\}$ is \emph{$\opK$-invariant}: for every $\psi \in \Hilb_E$ and $\nu$-a.e.\ control $\ctrl$, the function $\resid \mapsto \psi(F(\resid,\ctrl))$ again lies in $\Hilb_E$.
\end{itemize}
\end{definition}

(DR1) is ordinary linear independence and is satisfied by essentially any trained dictionary. (DR2) is the substantive condition, and it is exactly the one \eqref{eq:sae-objective} does not mention: it demands that the span be closed under the dynamics, so that evolving a feature by one layer keeps it inside the dictionary. Neither condition presupposes KSA; both are properties of the primitive itself.

The second part is a statement about estimation.

\begin{definition}[Spectral identifiability]
\label{def:spectral-identifiability}
Let $\hat A_M$ be an estimator of the realisation $A$ from $M$ calibration samples. The primitive is \emph{spectrally identifiable at rate} $r(M)$ if there is a permutation $\pi_M$ of $\{1,\dots,N\}$ such that, with high probability,
\begin{equation}
\label{eq:def-identifiability}
    \max_{k}\bigl|\lambda_k(\hat A_M) - \lambda_{\pi_M(k)}(A)\bigr| \;=\; O\bigl(r(M)\bigr).
\end{equation}
\end{definition}

Permutation is the right and only quotient: eigenvalues form an unordered multiset, so no estimator recovers a labelling, and nothing weaker need be quotiented out. This is markedly stronger than what is available for SAEs, whose features are identified at best up to permutation \emph{and} sign \emph{and} scale, and in practice not at all.

Two things must now be supplied. \Cref{sec:dynamics} constructs a primitive that satisfies \Cref{def:richness} and exhibits the estimand; \Cref{sec:existence} shows the estimand is well defined; and \Cref{sec:identifiability} shows it is recoverable at the parametric rate, making \Cref{def:spectral-identifiability} non-vacuous.

\section{Transformer Depth Dynamics and the Koopman Realisation}
\label{sec:dynamics}

This section constructs the primitive that satisfies \Cref{def:richness}: the controlled depth recurrence, the observable space and controlled Koopman operator, the finite-dimensional realisation, the EDMDc estimator, the spectral decomposition and the definition of a Koopman circuit, and the three assumptions under which everything later is proved. 

\subsection{The depth recurrence as a controlled dynamical system}
\label{sec:dynamics-recurrence}

We consider decoder-only transformers with pre-normalisation and RMSNorm, the architectural configuration used by Llama~3 \citep{grattafiori2024llama3} and Gemma~2 \citep{team2024gemma2} and descended from the original transformer of \cite{vaswani2017attention}. Let $L \in \N$ denote the number of layers, $d \in \N$ the residual-stream dimension, $H \in \N$ the number of attention heads per layer, and $T \in \N$ the number of tokens in a prompt. For each prompt and each token position $t \in \{1,\dots,T\}$, let $\resid_\ell^{(t)} \in \R^{d}$ denote the residual-stream state at layer $\ell \in \{0,\dots,L\}$. The layer update decomposes as
\begin{align}
    \widetilde{\resid}_\ell^{(t)}
        &= \resid_\ell^{(t)}
          + \underbrace{\mathrm{Attn}_\ell\!\bigl(\mathrm{RMSNorm}(\resid_\ell^{(1:T)})\bigr)^{(t)}}_{\displaystyle a_\ell^{(t)}}, \label{eq:setup-attn}\\[2pt]
    \resid_{\ell+1}^{(t)}
        &= \widetilde{\resid}_\ell^{(t)}
          + \underbrace{\mathrm{MLP}_\ell\!\bigl(\mathrm{RMSNorm}(\widetilde{\resid}_\ell^{(t)})\bigr)}_{\displaystyle m_\ell^{(t)}}. \label{eq:setup-mlp}
\end{align}

Fix a token position $t$ and drop the superscript. The attention output $a_\ell$ depends on the residual-stream states at \emph{other} token positions in the same layer, which are themselves determined by their own trajectories. From the vantage point of the trajectory $\{\resid_\ell\}_{\ell=0}^{L}$ at token $t$, therefore, $a_\ell$ enters as an exogenous input. We denote this input by $\ctrl_\ell \in \R^{p}$ (with $p = d$ under the choice $\ctrl_\ell = a_\ell$; the ablation $\ctrl_\ell = (a_\ell, m_\ell)$ is considered separately).

\begin{definition}[Controlled depth recurrence]
\label{def:recurrence}
Fix a token position $t \in \{1,\dots,T\}$. Let $\mathcal{X} \subseteq \R^{d}$ be the state space and $\mathcal{U} \subseteq \R^{p}$ the control space. The \emph{controlled depth recurrence} at token $t$ is the discrete-time dynamical system
\begin{equation}
\label{eq:setup-recurrence}
    \resid_{\ell+1} = F(\resid_\ell,\ctrl_\ell) = \resid_\ell + \ctrl_\ell + m\!\bigl(\resid_\ell + \ctrl_\ell\bigr),
    \qquad \ell = 0,\dots,L-1,
\end{equation}
where $m(\cdot) = \mathrm{MLP}(\mathrm{RMSNorm}(\cdot))$ and $F : \mathcal{X} \times \mathcal{U} \to \mathcal{X}$ is the layer map.
\end{definition}

\begin{remark}[Cross-token coupling]
\label{rem:cross-token}
Cross-token dependence enters through the empirical distribution of $\{\ctrl_\ell\}$, not through $F$. This is the DMDc convention for exogenous inputs \citep{proctor2016dmdc}. Analyses that require modelling cross-token dynamics jointly can be handled by lifting to the $T$-token joint state $\resid_\ell \in \R^{Td}$; we do not require this for our results. The cost of the convention is measured directly in \Cref{sec:exp-exogeneity}.
\end{remark}

\paragraph{Data-generating process.}
The calibration corpus induces an empirical distribution over trajectories $\{(\resid_\ell,\ctrl_\ell)\}_{\ell=0}^{L-1}$. We assume this distribution converges, as the corpus grows, to a joint measure $\mu \otimes \nu$ on $\mathcal{X} \times \mathcal{U}$, where $\mu$ is an invariant measure for the state process and $\nu$ is a stationary distribution for the controls. Existence of such measures under mild regularity is standard \citep{brunton2022modern}; we work with them as given. Because residual-stream norms grow with depth, $\mu$ is in practice the layer marginal $\mu_\ell$ rather than one depth-invariant law; \Cref{sec:exp-stationarity} measures the growth and shows it is recovered as a Koopman mode once the dictionary can express it.

\subsection{Observables, the Koopman operator, and the finite-dimensional realisation}
\label{sec:dynamics-koopman}

The Koopman operator lifts the nonlinear map $F$ to a linear operator on an infinite-dimensional space of \emph{observables} \citep{koopman1931hamiltonian}.

\begin{definition}[Observable space]
\label{def:observable-space}
The observable space is $\Hilb \coloneqq L^{2}(\mathcal{X},\mu)$, the space of measurable $\psi : \mathcal{X} \to \C$ with $\int|\psi|^{2}\,d\mu < \infty$, equipped with the inner product $\langle \psi,\varphi\rangle_{\Hilb} = \int \psi\,\bar{\varphi}\,d\mu$. A \emph{vector-valued observable} is a stacking $\obs = (\psi_1,\dots,\psi_N)^{\top} : \mathcal{X} \to \C^{N}$ with each $\psi_i \in \Hilb$.
\end{definition}

\begin{definition}[Controlled Koopman operator]
\label{def:koopman}
The \emph{controlled Koopman operator} $\opK$ associated with $F$ sends an observable to its composition with the layer map:
\begin{equation}
\label{eq:setup-koopman}
    (\opK\,\psi)(\resid,\ctrl) \;\coloneqq\; \psi\!\bigl(F(\resid,\ctrl)\bigr),
    \qquad \psi \in \Hilb.
\end{equation}
Although $F$ is nonlinear in $\resid$, $\opK$ is linear in $\psi$.
\end{definition}

We write $\opK_{\ctrl}$ for the family of operators obtained by fixing the control, $(\opK_{\ctrl}\psi)(\resid) = \psi(F(\resid,\ctrl))$, and $\bar\opK \coloneqq \E_{\ctrl\sim\nu}[\opK_\ctrl]$ for the control average. Below, ``the Koopman operator'' without qualification means $\bar\opK$.

\begin{definition}[Koopman-invariant subspace]
\label{def:invariance}
Let $\Hilb_N \coloneqq \spn\{\psi_1,\dots,\psi_N\} \subseteq \Hilb$ be the linear span of a dictionary $\obs$. $\Hilb_N$ is \emph{$\opK$-invariant} if for every $\psi \in \Hilb_N$ and every $\ctrl \in \mathcal{U}$, the function $\resid \mapsto \psi(F(\resid,\ctrl))$ belongs to $\Hilb_N$ as a function of $\resid$.
\end{definition}

$\opK$-invariance is a substantive condition on the dictionary: it requires the finite-dimensional span to be closed under the Koopman flow. When it holds, the operator restricts to a finite-dimensional linear map $\opK|_{\Hilb_N} : \Hilb_N \to \Hilb_N$ admitting a matrix representation on any basis.

\begin{definition}[Koopman realisation]
\label{def:realisation}
Assume $\Hilb_N$ is $\opK$-invariant. The \emph{Koopman realisation} of $F$ on $\obs$ is the pair $(A,B) \in \C^{N\times N} \times \C^{N \times p}$ satisfying
\begin{equation}
\label{eq:setup-realisation}
    \obs\bigl(F(\resid,\ctrl)\bigr)
    \;=\; A\,\obs(\resid) \;+\; B\,\ctrl
    \qquad (\mu\otimes\nu)\text{-a.e.}
\end{equation}
Here $A$ is the matrix representation of $\opK|_{\Hilb_N}$ in the basis $\{\psi_i\}$ and $B$ encodes the linear response of $\obs$ to the control input. When $\Hilb_N$ is not $\opK$-invariant, we take $(A,B)$ to be the population least-squares approximation
\begin{equation}
\label{eq:setup-realisation-ls}
    (A,B) \;\coloneqq\; \argmin_{A',B'}\;
    \E_{(\resid,\ctrl) \sim \mu\otimes\nu}
    \bigl\|\obs\bigl(F(\resid,\ctrl)\bigr) - A'\,\obs(\resid) - B'\,\ctrl\bigr\|_{2}^{2},
\end{equation}
and denote the projection residual by $\varepsilon(\resid,\ctrl) \coloneqq \obs(F(\resid,\ctrl)) - A\,\obs(\resid) - B\,\ctrl$.
\end{definition}

The realisation is defined at the level of an abstract basis of $\Hilb_N$. Different bases yield different matrix representations related by similarity: if $\obs'(\resid) = T\,\obs(\resid)$ for invertible $T \in \C^{N \times N}$, then $A' = TAT^{-1}$ and $B' = TB$. \Cref{thm:existence} formalises that these are the \emph{only} ambiguities in the definition of $(A,B)$.

\begin{definition}[Read-out]
\label{def:readout}
A \emph{read-out} is a linear operator $C \in \C^{q \times N}$ that maps the observable at any layer to a $q$-dimensional target: $y_\ell = C\,\obs(\resid_\ell) \in \C^{q}$. Common choices include $C = W_{U}^{\top}P$ (projection onto specified logit directions) or $C$ encoding a probe direction.
\end{definition}

The triple $(A,B,C)$ constitutes the KSA realisation, a discrete-time linear time-invariant (LTI) system on $\C^{N}$. We denote its input--output map by $G : \{\ctrl_\ell\} \mapsto \{y_\ell\}$ and its transfer function by $\hat{G}(z) = C\,(zI - A)^{-1}\,B$, with $\|\hat{G}\|_{\Hinf}$ the Hardy $\Hinf$-norm on the closed unit disk. \Cref{sec:reduction} bounds the error of truncating this system to $r \ll N$ modes.

\begin{figure*}
\centering
\includegraphics[width=\textwidth]{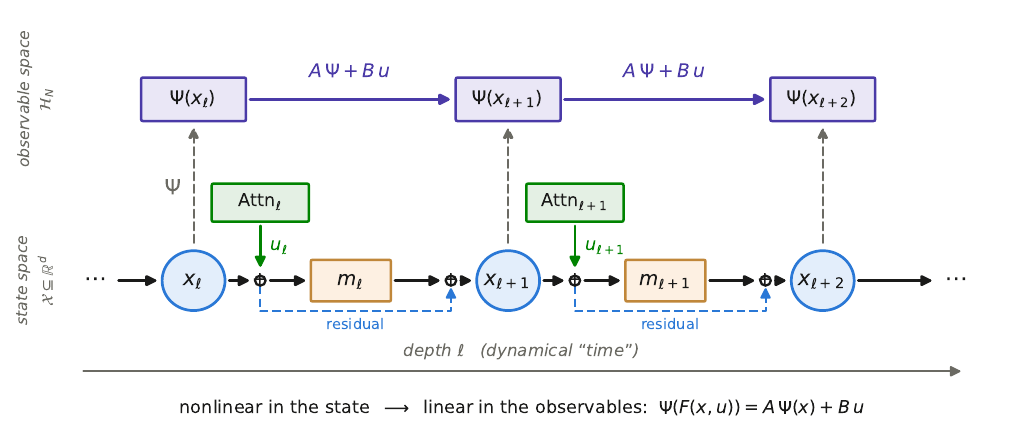}
\caption{Transformer forward pass as a controlled dynamical system (bottom) and its Koopman lift (top). With a Koopman-invariant dictionary the lifted dynamics become exactly linear, $\obs(F(\resid,\ctrl)) = A\,\obs(\resid) + B\,\ctrl$. The basis-independent spectrum of $A$ is the identifiable object (\Cref{thm:existence}).}
\label{fig:schematic}
\end{figure*}

\subsection{The EDMDc estimator}
\label{sec:dynamics-edmdc}

The Koopman realisation $(A,B)$ is not directly accessible; we estimate it from calibration data using the extended dynamic mode decomposition with control (EDMDc) of \cite{williams2015edmd,proctor2016dmdc}.

\paragraph{Calibration data.}
Let $\mathcal{D}$ be a calibration corpus consisting of $M$ layer--token samples $\{(\resid_\ell^{(i)},\ctrl_\ell^{(i)})\}_{i=1}^{M}$ obtained by running the transformer on a corpus of prompts and caching residual-stream states and control inputs at the analysis position. Assemble the \emph{snapshot matrices}
\begin{align}
\label{eq:setup-snapshots}
    X   &\coloneqq \bigl[\,\obs(\resid_0^{(1)}),\dots,\obs(\resid_{L-1}^{(1)}),\obs(\resid_0^{(2)}),\dots\bigr] \in \C^{N\times M}, \\
    Y   &\coloneqq \bigl[\,\obs(\resid_1^{(1)}),\dots,\obs(\resid_{L}^{(1)}),\obs(\resid_1^{(2)}),\dots\bigr] \in \C^{N\times M}, \\
    \Xi &\coloneqq \bigl[\,\ctrl_0^{(1)},\dots,\ctrl_{L-1}^{(1)},\ctrl_0^{(2)},\dots\bigr] \in \C^{p\times M}.
\end{align}
The columns of $Y$ are the next-layer observables corresponding to the columns of $X$; $\Xi$ collects the control inputs applied between $X$ and $Y$.

\begin{definition}[EDMDc estimator]
\label{def:edmdc}
Fix a regularisation parameter $\gamma \geq 0$. The \emph{EDMDc estimator} on $M$ samples is
\begin{equation}
\label{eq:setup-edmdc}
\begin{split}
    (\hat A_M,\hat B_M)
    \;\coloneqq\; \argmin_{A' \in \C^{N\times N},\, B' \in \C^{N \times p}}\;
    &\|Y - A'X - B'\Xi\|_{F}^{2} \\
    &+ \gamma\bigl(\|A'\|_{F}^{2} + \|B'\|_{F}^{2}\bigr).
\end{split}
\end{equation}

When $\gamma = 0$ and the stacked snapshot matrix $\bigl[\begin{smallmatrix}X\\\Xi\end{smallmatrix}\bigr]$ has full row rank $N+p$, the solution is
\begin{equation}
\label{eq:setup-edmdc-closed}
    [\hat A_M,\;\hat B_M]
    \;=\; Y\,\begin{bmatrix}X\\ \Xi\end{bmatrix}^{\!\dagger}.
\end{equation}
\end{definition}

The estimator is the empirical counterpart of the population minimiser \eqref{eq:setup-realisation-ls}. Under $\opK$-invariance and persistent excitation it converges to $(A,B)$; the finite-sample rate is the subject of \Cref{thm:identifiability}. All experiments use $\gamma = 0$, the unregularised pseudo-inverse solution.

\subsection{Spectral decomposition and the definition of a Koopman circuit}
\label{sec:dynamics-spectrum}

We now attach the interpretive apparatus to the estimated realisation.

\paragraph{Eigendecomposition.}
Assume $\hat A_M \in \C^{N \times N}$ is diagonalisable (the Jordan-form extension is discussed in \Cref{rem:jordan}). There exist right eigenvectors $\{v_k\}_{k=1}^{N}$ and left eigenvectors $\{\phi_k\}_{k=1}^{N}$ in $\C^{N}$ such that
\begin{align}
\label{eq:setup-eig-right}
    \hat A_M \, v_k &= \lambda_k\, v_k, \\
\label{eq:setup-eig-left}
    \phi_k^{\top}\, \hat A_M &= \lambda_k\, \phi_k^{\top},
\end{align}
biorthogonally normalised so that $\phi_j^{\top} v_k = \delta_{jk}$. The eigendecomposition is $\hat A_M = \sum_{k=1}^{N} \lambda_k\, v_k \phi_k^{\top}$.

\begin{definition}[Spectral projector]
\label{def:spectral-projector}
The \emph{spectral projector} associated with the eigenpair $(\lambda_k,v_k,\phi_k)$ is
\begin{equation}
\label{eq:setup-projector}
    \Proj_k \;\coloneqq\; v_k\,\phi_k^{\top} \;\in\; \C^{N\times N}.
\end{equation}
The family $\{\Proj_k\}_{k=1}^{N}$ satisfies $\Proj_j\Proj_k = \delta_{jk}\Proj_k$ and $\sum_k\Proj_k = I_N$.
\end{definition}

\paragraph{Koopman eigenfunctions.}
Corresponding to each left eigenvector $\phi_k \in \C^{N}$ is a scalar observable $\varphi_k \in \Hilb_N$ obtained by contraction with the dictionary,
\begin{equation}
\label{eq:setup-eigenfunction}
    \varphi_k(\resid) \;\coloneqq\; \phi_k^{\top}\,\obs(\resid),
\end{equation}
which is a \emph{Koopman eigenfunction} of $\opK|_{\Hilb_N}$ in the classical sense: on autonomous dynamics ($\ctrl = 0$), $\varphi_k(F(\resid,0)) = \lambda_k\,\varphi_k(\resid)$ $\mu$-a.e. We distinguish notationally between the eigenfunction $\varphi_k : \mathcal{X}\to\C$ and its left-eigenvector representative $\phi_k \in \C^{N}$, a coordinate vector in the dictionary basis.

\begin{definition}[Koopman circuit]
\label{def:circuit}
A \emph{Koopman circuit} associated with the realisation $(\hat A_M,\hat B_M,C)$ is a quadruple $(\lambda_k, \phi_k, v_k, \Proj_k)$ consisting of an eigenvalue $\lambda_k \in \C$, its left eigenvector $\phi_k$ (equivalently its eigenfunction $\varphi_k$), its right eigenvector $v_k$ (the \emph{Koopman mode}), and its spectral projector $\Proj_k = v_k\phi_k^{\top}$.
\end{definition}

The interpretive taxonomy attached to the eigenvalue is the classical one: $|\lambda_k|>1$ amplifying, $|\lambda_k|\approx 1$ transport, $|\lambda_k|<1$ decaying, $\arg\lambda_k \neq 0$ rotational. Modes live in observable space $\C^N$; \Cref{sec:exp-readout} gives the fitted linear map back to residual-stream coordinates $\R^{d}$ that every mode-level empirical claim depends on.

\subsection{Assumptions}
\label{sec:dynamics-assumptions}

All results below rest on three assumptions, stated here and discussed immediately after.

\begin{assumption}[Dictionary $\opK$-invariance]
\label{ass:invariance}
The dictionary span $\Hilb_N = \spn\{\psi_1,\dots,\psi_N\}$ is $\opK$-invariant in the sense of \Cref{def:invariance}. Equivalently, for every $\psi \in \Hilb_N$ and $\nu$-almost every $\ctrl \in \mathcal{U}$, the composition $\psi \circ F(\cdot,\ctrl)$ belongs to $\Hilb_N$. This is (DR2) of \Cref{def:richness}.
\end{assumption}

\begin{assumption}[Persistent excitation]
\label{ass:excitation}
There exists $\eta > 0$ such that the empirical control covariance $\Sigma_M \coloneqq \frac{1}{M}\Xi\Xi^{\top} \in \C^{p \times p}$ satisfies
\begin{equation}
\label{eq:setup-excitation}
    \lambda_{\min}\!\bigl(\Sigma_M\bigr) \;\geq\; \eta
\end{equation}
with probability at least $1 - \delta_M$, where $\delta_M \to 0$ as $M \to \infty$. Equivalently, the control sequence excites all $p$ input directions in the limit.
\end{assumption}

\begin{assumption}[Spectral separation]
\label{ass:separation}
There exists $\Delta > 0$ such that the eigenvalues of the true Koopman compression $A$ are pairwise separated:
\begin{equation}
\label{eq:setup-separation}
    \bigl|\lambda_i(A) - \lambda_j(A)\bigr| \;\geq\; \Delta
    \qquad \text{for all } i \neq j.
\end{equation}
\end{assumption}

\begin{remark}[On \Cref{ass:invariance}]
\label{rem:invariance}
\Cref{ass:invariance} is the substantive richness condition. In practice, transformer-aligned dictionaries built from the model's own principal directions and random Fourier features \emph{approximate} the condition well but do not satisfy it exactly. The theorems are stated under exact invariance; \Cref{rem:approximate-invariance} and \Cref{rem:ident-approx} give the extension to the approximate case, in which projection-residual bias terms enter the bounds additively. \Cref{sec:exp-sae} measures the residual $\|\varepsilon\|$ for public SAEs directly rather than assuming it small.
\end{remark}

\begin{remark}[On \Cref{ass:excitation}]
\label{rem:excitation}
Persistent excitation is the classical identifiability requirement in system identification \citep{ljung1999system,willems2005fundamental}. It ensures the empirical control Gramian $\Sigma_M$ is invertible for sufficiently large $M$, which is necessary for the input-to-observable map $B$ to be recovered. For transformer analyses the condition follows if the calibration corpus is sufficiently diverse across attention patterns.
\end{remark}

\begin{remark}[On \Cref{ass:separation}]
\label{rem:separation}
Spectral separation is the strongest assumption. It excludes eigenvalue degeneracies that would render individual modes non-identifiable, though \emph{clusters} of near-degenerate modes remain identifiable via the cluster-level version of the theorem (\Cref{cor:cluster-ident}). \Cref{thm:gap-free-identifiability} shows the eigenvalue guarantee survives dropping this assumption entirely, at the cost of a dimensional factor.
\end{remark}

\subsection{Three modelling conventions, all measured}
\label{sec:dynamics-conventions}

Three modelling choices define the estimand, and two of them constrain what it can be. We state them here and measure each of them in \Cref{sec:experiments}, rather than leaving them implicit.

\emph{First, the realisation is depth-indexed.} Since $m_\ell$ is layer-dependent, we fit $\hat A_\ell$ separately at each analysed layer and the spectrum is a family $\ell\mapsto\sigma(A_\ell)$. This is resolved rather than assumed: adjacent-layer spectral distance exceeds the split-half resolution floor by median factors of $\ResGptDriftRatio$, $\ResGemmaDriftRatio$ and $\ResQwenDriftRatio$ across the suite, and a pooled $A$ predicts held-out transitions $\ResGptPoolPenalty\%$, $\ResGemmaPoolPenalty\%$ and $\ResQwenPoolPenalty\%$ worse (\Cref{sec:exp-depth-homogeneity}). \Cref{thm:identifiability} is proved at fixed $\ell$ and needs no depth-homogeneity.

\emph{Second, the realisation is defined relative to the layer marginal.} Residual-stream norm grows with depth ($\ResGptNormGrowth\times$ on GPT-2 small across the analysed range), so a single invariant law $\mu$ is not available and the per-layer fit enforces the marginal $\mu_\ell$ instead. Adding $\log\|\resid_\ell\|$ to the dictionary then recovers that growth as its own real mode on all three models ($\ResGptNormMode$, $\ResGemmaNormMode$, $\ResQwenNormMode$ against measured growth $\ResGptNormEmpirical$, $\ResGemmaNormEmpirical$, $\ResQwenNormEmpirical$) --- a prediction that could have failed and did not (\Cref{sec:exp-stationarity}).

\emph{Third, and least innocuous, exogeneity of the control is a choice.} The write $\ctrl_\ell$ is computed from the same residual stream, so treating it as exogenous is a modelling convention rather than a fact. The largest canonical correlation between $\obs(\resid_\ell)$ and $\ctrl_\ell$ is $\rho_1 = \ResGptRhoOne$ (GPT-2 small), $\ResGemmaRhoOne$ (Gemma-2-2B) and $\ResQwenRhoOne$ (Qwen3-8B-Base), though the controls are mostly unexplained by the state ($R^2 = \ResGptCtrlRsq$, $\ResGemmaCtrlRsq$, $\ResQwenCtrlRsq$). Residualising the control against the lifted state (Frisch--Waugh--Lovell) shifts the spectrum by $\ResGptExogShift\times$, $\ResGemmaExogShift\times$ and $\ResQwenExogShift\times$ the split-half floor, so the two conventions yield genuinely different estimands at our resolution. We report the naive realisation as primary and the residualised version alongside (\Cref{sec:exp-exogeneity}).

\section{Existence and Uniqueness of the Koopman Realisation}
\label{sec:existence}

This section establishes the foundational claim on which everything later rests: given a $\opK$-invariant dictionary, the object we call ``the Koopman realisation'' - the matrix $A$ associated with the operator's action on the dictionary span - exists, is uniquely determined up to a change of basis, and has a spectrum $\sigma(A)$ that is a coordinate-free property of the pair $(F,\Hilb_N)$. Uniqueness up to similarity is the correct invariance property: it says the eigenvalues, spectral projectors and modal structure of $A$ are attributes of the transformer and dictionary, not of the particular ordering or scaling of the basis vectors $\{\psi_i\}$. Without this, \Cref{def:spectral-identifiability} would have no estimand to speak of.

The proof proceeds in three preparatory lemmas followed by consolidation. \Cref{lem:pointwise} shows that under \Cref{ass:invariance} the pointwise Koopman matrix $A_{\ctrl}$ exists and is unique for each control value. \Cref{lem:basis-change} establishes the similarity transformation under change of basis. \Cref{lem:spectrum-embed} establishes the spectral inclusion $\sigma(A) \subseteq \sigma(\bar\opK)$ via standard operator-theoretic facts about invariant subspaces. We then assemble the theorem and derive a corollary formalising the coordinate-free invariants of Koopman circuits.

\subsection{Regularity conditions and statement}
\label{sec:existence-statement}

We work throughout under \Cref{ass:invariance} and the linear independence condition
\begin{equation}
\label{eq:existence-R1}
\text{(R1)}\qquad
    G_{\obs} \;\coloneqq\; \E_{\mu}\!\bigl[\obs(\resid)\,\obs(\resid)^{*}\bigr]
    \;\succ\; 0,
\end{equation}
which asserts that $\{\psi_i\}_{i=1}^{N}$ are linearly independent as elements of $L^{2}(\mu)$ - that is, (DR1) of \Cref{def:richness}. We further impose the standard state--control decorrelation condition
\begin{equation}
\label{eq:existence-R2}
\text{(R2)}\qquad
    \E_{(\resid,\ctrl)\sim\mu\otimes\nu}\!\bigl[\obs(\resid)\,\ctrl^{*}\bigr]
    \;=\; \E_{\mu}[\obs(\resid)]\;\E_{\nu}[\ctrl]^{*},
\end{equation}
which follows automatically when $\resid$ and $\ctrl$ are independent under the joint measure. Without loss of generality we centre the controls, $\E_{\nu}[\ctrl] = 0$, so that (R2) becomes $\E[\obs(\resid)\,\ctrl^{*}] = 0$.

\begin{theorem}[Existence and uniqueness of the Koopman realisation]
\label{thm:existence}
Assume \Cref{ass:invariance} together with \eqref{eq:existence-R1} and \eqref{eq:existence-R2}. Then:
\begin{enumerate}
\item[(a)] \emph{(Existence and uniqueness.)} There exists a unique matrix $A \in \C^{N\times N}$ satisfying
\begin{equation}
\label{eq:existence-A-def}
    A\,\obs(\resid)
    \;=\;
    \E_{\ctrl \sim \nu}\!\bigl[\,\obs\bigl(F(\resid,\ctrl)\bigr)\,\bigr]
    \qquad \mu\text{-a.e.}\ \resid \in \mathcal{X}.
\end{equation}
Equivalently, $A$ is the matrix representation of the bounded linear operator $\bar{\opK}|_{\Hilb_N}$ on the basis $\{\psi_i\}_{i=1}^{N}$.

\item[(b)] \emph{(Basis-change equivariance.)} Under any change of basis $\obs'(\resid) = T\,\obs(\resid)$ with $T \in GL(N,\C)$, the realisation transforms as $A' = T A T^{-1}$.

\item[(c)] \emph{(Spectral invariance.)} The spectrum $\sigma(A) \subseteq \C$ is invariant under basis change and constitutes a coordinate-free invariant of the pair $(F,\Hilb_N)$.

\item[(d)] \emph{(Spectral embedding.)} $\sigma(A) \subseteq \sigma(\bar{\opK})$.

\item[(e)] \emph{(Agreement with the least-squares realisation.)} The matrix $A$ coincides with the population minimiser $A_{\mathrm{LS}}$ of the least-squares problem \eqref{eq:setup-realisation-ls}.
\end{enumerate}
\end{theorem}

Parts (a)--(c) constitute the existence and uniqueness statement and make \Cref{def:spectral-identifiability} well posed by supplying an estimand independent of the dictionary's basis. Part~(d) is the spectral-inclusion claim: the recovered eigenvalues are genuine Koopman eigenvalues of the transformer's depth dynamics, not artefacts of the finite truncation. Part~(e) connects the abstract compression to the estimable object of \Cref{def:edmdc}: it is $A_{\mathrm{LS}}$ that EDMDc targets, and (e) shows the two coincide under the stated conditions.

\subsection{Preparatory lemmas}
\label{sec:existence-lemmas}

\begin{lemma}[Pointwise Koopman matrix]
\label{lem:pointwise}
Under \Cref{ass:invariance} and \eqref{eq:existence-R1}, for each $\ctrl \in \mathcal{U}$ there exists a unique matrix $A_{\ctrl} \in \C^{N\times N}$ such that
\begin{equation}
\label{eq:existence-Au}
    \obs\bigl(F(\resid,\ctrl)\bigr) \;=\; A_{\ctrl}\,\obs(\resid)
    \qquad \mu\text{-a.e.}\ \resid \in \mathcal{X}.
\end{equation}
The matrix admits the closed form
\begin{equation}
\label{eq:existence-Au-closed}
    A_{\ctrl}
    \;=\;
    \E_{\mu}\!\bigl[\obs\bigl(F(\resid,\ctrl)\bigr)\,\obs(\resid)^{*}\bigr]\,G_{\obs}^{-1}.
\end{equation}
\end{lemma}

\begin{proof}
Fix $\ctrl \in \mathcal{U}$. By \Cref{ass:invariance}, for each $i \in \{1,\dots,N\}$ the function $\resid \mapsto \psi_i\bigl(F(\resid,\ctrl)\bigr)$ belongs to $\Hilb_N$. Since $\{\psi_j\}_{j=1}^{N}$ is a basis of $\Hilb_N$ under (R1), there is a unique tuple $(a_{i1},\dots,a_{iN}) \in \C^{N}$ with $\psi_i(F(\resid,\ctrl)) = \sum_{j=1}^{N} a_{ij}\,\psi_j(\resid)$ $\mu$-a.e. Stacking these coefficients into the $i$-th row of $A_{\ctrl}$ yields \eqref{eq:existence-Au}.

To derive \eqref{eq:existence-Au-closed}, right-multiply \eqref{eq:existence-Au} by $\obs(\resid)^{*}$ and take expectation under $\mu$:
\begin{equation}
    \E_{\mu}\!\bigl[\obs(F(\resid,\ctrl))\,\obs(\resid)^{*}\bigr]
    \;=\; A_{\ctrl}\,\E_{\mu}\!\bigl[\obs(\resid)\,\obs(\resid)^{*}\bigr]
    \;=\; A_{\ctrl}\,G_{\obs}.
\end{equation}
Since $G_{\obs} \succ 0$ by (R1), $A_{\ctrl}$ is uniquely determined as \eqref{eq:existence-Au-closed}.
\end{proof}

\begin{lemma}[Basis-change equivariance]
\label{lem:basis-change}
Under a change of dictionary $\obs'(\resid) = T\,\obs(\resid)$ with $T \in GL(N,\C)$, the pointwise Koopman matrix transforms as $A'_{\ctrl} = T\,A_{\ctrl}\,T^{-1}$ for every $\ctrl \in \mathcal{U}$. Consequently the averaged matrix $A = \E_{\nu}[A_{\ctrl}]$ transforms as $A' = T\,A\,T^{-1}$.
\end{lemma}

\begin{proof}
Under the new basis the Gram matrix is $G_{\obs'} = \E_{\mu}[T\,\obs\,\obs^{*}\,T^{*}] = T\,G_{\obs}\,T^{*}$. Applying \eqref{eq:existence-Au-closed} to the new dictionary,
\begin{align*}
    A'_{\ctrl}
    &= \E_{\mu}\!\bigl[T\,\obs(F(\resid,\ctrl))\,(T\,\obs(\resid))^{*}\bigr]\,(T\,G_{\obs}\,T^{*})^{-1} \\
    &= T\,\E_{\mu}\!\bigl[\obs(F(\resid,\ctrl))\,\obs(\resid)^{*}\bigr]\,T^{*}\,(T^{*})^{-1}\,G_{\obs}^{-1}\,T^{-1} \\
    &= T\,A_{\ctrl}\,T^{-1}.
\end{align*}
Averaging over $\nu$ preserves the similarity: $A' = T\,A\,T^{-1}$.
\end{proof}

\begin{lemma}[Spectral embedding]
\label{lem:spectrum-embed}
Under \Cref{ass:invariance}, the finite-dimensional subspace $\Hilb_N$ is $\bar{\opK}$-invariant, and the spectrum of the restriction obeys
\begin{equation}
\label{eq:existence-embed}
    \sigma\!\bigl(\bar{\opK}|_{\Hilb_N}\bigr) \;\subseteq\; \sigma(\bar{\opK}).
\end{equation}
Moreover the matrix $A$ of \eqref{eq:existence-A-def} is a matrix representation of $\bar{\opK}|_{\Hilb_N}$, so $\sigma(A) = \sigma(\bar{\opK}|_{\Hilb_N})$, whence $\sigma(A) \subseteq \sigma(\bar{\opK})$.
\end{lemma}

\begin{proof}
Fix $\psi \in \Hilb_N$. By \Cref{ass:invariance}, for each $\ctrl \in \mathcal{U}$ we have $\opK_{\ctrl}\,\psi \in \Hilb_N$. Because $\Hilb_N$ is finite-dimensional and hence closed, the Bochner integral $\bar{\opK}\,\psi = \int \opK_{\ctrl}\,\psi\,d\nu(\ctrl)$ takes values in $\Hilb_N$ (the integrand is $\nu$-measurable and uniformly bounded in the finite-dimensional norm). This establishes $\bar{\opK}$-invariance of $\Hilb_N$.

Since $\Hilb_N$ is finite-dimensional it is complemented in $\Hilb$: $\Hilb = \Hilb_N \oplus \Hilb_N^{\perp}$. With respect to this decomposition, $\bar{\opK}$ has block-triangular form
\begin{equation}
    \bar{\opK}
    \;=\;
    \begin{pmatrix}
        \bar{\opK}|_{\Hilb_N} & * \\
        0                     & \bar{\opK}|_{\Hilb_N^{\perp}\to\Hilb_N^{\perp}}
    \end{pmatrix},
\end{equation}
because the $(2,1)$ block vanishes by invariance. A standard result in operator theory then gives $\sigma(\bar{\opK}) = \sigma(\bar{\opK}|_{\Hilb_N}) \cup \sigma(\bar{\opK}|_{\Hilb_N^{\perp}\to\Hilb_N^{\perp}})$ \citep[Ch.~III, \S4]{kato1995perturbation}, so $\sigma(\bar{\opK}|_{\Hilb_N}) \subseteq \sigma(\bar{\opK})$, which is \eqref{eq:existence-embed}.

Finally, expressing $\bar{\opK}|_{\Hilb_N}$ in the basis $\{\psi_i\}_{i=1}^{N}$ produces a matrix whose action satisfies \eqref{eq:existence-A-def}, and this matrix is precisely $A$. Since spectra of finite-dimensional operators equal spectra of their matrix representations, $\sigma(A) = \sigma(\bar{\opK}|_{\Hilb_N})$.
\end{proof}

\subsection{Proof of \texorpdfstring{\Cref{thm:existence}}{Theorem 5.1}}
\label{sec:existence-proof}

\begin{proof}[Proof of \Cref{thm:existence}]
Existence of $A$ satisfying \eqref{eq:existence-A-def} follows from \Cref{lem:pointwise} by averaging over $\ctrl \sim \nu$: for $\mu$-a.e.\ $\resid \in \mathcal{X}$,
\begin{equation}
    \E_{\ctrl\sim\nu}\!\bigl[\obs(F(\resid,\ctrl))\bigr]
    \;=\; \E_{\nu}[A_{\ctrl}\,\obs(\resid)]
    \;=\; \E_{\nu}[A_{\ctrl}]\,\obs(\resid)
    \;=\; A\,\obs(\resid),
\end{equation}
where $A \coloneqq \E_{\nu}[A_{\ctrl}]$. The average is well defined in $\C^{N\times N}$ because the family $\{A_{\ctrl}\}_{\ctrl\in\mathcal{U}}$ is $\nu$-measurable by the closed form \eqref{eq:existence-Au-closed} and bounded on compact subsets of $\mathcal{U}$; the Bochner integral therefore converges. Uniqueness of $A$ follows by right-multiplying \eqref{eq:existence-A-def} by $\obs(\resid)^{*}$, taking expectation under $\mu$, and applying $G_{\obs} \succ 0$. This proves (a).

Part (b) is \Cref{lem:basis-change}. Part (c) follows because $\sigma(TAT^{-1}) = \sigma(A)$ for any $T \in GL(N,\C)$, so the spectrum depends only on the similarity class of $A$ and not on the choice of basis for $\Hilb_N$. Part (d) is \Cref{lem:spectrum-embed}.

It remains to prove (e). The population least-squares realisation $(A_{\mathrm{LS}}, B_{\mathrm{LS}})$ of \Cref{def:realisation} solves
\begin{equation}
    \min_{A',B'}\;
    \E_{(\resid,\ctrl)\sim\mu\otimes\nu}
    \bigl\|\obs(F(\resid,\ctrl)) - A'\,\obs(\resid) - B'\,\ctrl\bigr\|_{2}^{2}.
\end{equation}
The first-order optimality conditions of this quadratic problem are
\begin{align*}
    A_{\mathrm{LS}}\,G_{\obs} + B_{\mathrm{LS}}\,\E[\ctrl\,\obs(\resid)^{*}]
        &= \E\!\bigl[\obs(F(\resid,\ctrl))\,\obs(\resid)^{*}\bigr],\\
    A_{\mathrm{LS}}\,\E[\obs(\resid)\,\ctrl^{*}] + B_{\mathrm{LS}}\,G_{\ctrl}
        &= \E\!\bigl[\obs(F(\resid,\ctrl))\,\ctrl^{*}\bigr],
\end{align*}
where $G_{\ctrl} \coloneqq \E_{\nu}[\ctrl\ctrl^{*}]$. By (R2) and the centring $\E_{\nu}[\ctrl] = 0$, the cross-terms $\E[\ctrl\,\obs(\resid)^{*}]$ and $\E[\obs(\resid)\,\ctrl^{*}]$ vanish, so the equations decouple. Solving the first for $A_{\mathrm{LS}}$ and using \Cref{lem:pointwise},
\begin{align*}
    A_{\mathrm{LS}}
    &= \E\!\bigl[\obs(F(\resid,\ctrl))\,\obs(\resid)^{*}\bigr]\,G_{\obs}^{-1} \\
    &= \E_{\nu}\!\bigl[\,\E_{\mu}[A_{\ctrl}\,\obs(\resid)\,\obs(\resid)^{*}]\,\bigr]\,G_{\obs}^{-1} \\
    &= \E_{\nu}[A_{\ctrl}\,G_{\obs}]\,G_{\obs}^{-1}
    \;=\; \E_{\nu}[A_{\ctrl}]
    \;=\; A.
\end{align*}
This proves $A_{\mathrm{LS}} = A$ and completes the proof.
\end{proof}

\subsection{Coordinate-free invariants of Koopman circuits}
\label{sec:existence-consequences}

\Cref{thm:existence} implies that every spectral attribute of $A$ preserved under similarity is a well-defined invariant of the pair $(F,\Hilb_N)$.

\begin{corollary}[Coordinate-free invariants of the Koopman realisation]
\label{cor:invariants}
Under the hypotheses of \Cref{thm:existence}, the following are coordinate-free invariants of $(F,\Hilb_N)$:
\begin{enumerate}[leftmargin=2.4em,itemsep=1pt]
\item[(i)] the multiset of eigenvalues $\{\lambda_k\}_{k=1}^{N}$ of $A$, with algebraic multiplicities;
\item[(ii)] for each distinct eigenvalue $\lambda$, the spectral projector $\Proj_{\lambda}$ onto the corresponding generalised eigenspace, in the sense that under basis change $T$ it transforms as $\Proj'_{\lambda} = T\,\Proj_{\lambda}\,T^{-1}$;
\item[(iii)] the number and sizes of Jordan blocks associated with each eigenvalue;
\item[(iv)] the algebra $\Alg(A) \subseteq \C^{N\times N}$ generated by the spectral projectors under composition and complex linear combination, considered as an abstract commutative $*$-algebra.
\end{enumerate}
\end{corollary}

\begin{proof}
Under $A \mapsto TAT^{-1}$: eigenvalues are preserved, $\sigma(TAT^{-1}) = \sigma(A)$. Riesz projectors $\Proj_{\lambda} = \frac{1}{2\pi i}\oint_{\Gamma_\lambda} (zI - A)^{-1}\,dz$ around a small contour $\Gamma_\lambda$ enclosing $\lambda$ transform as $T\Proj_{\lambda}T^{-1}$, because $(zI - TAT^{-1})^{-1} = T(zI-A)^{-1}T^{-1}$. Jordan block counts are invariants of similarity classes. Finally, the algebra generated by $\{T\Proj_{\lambda}T^{-1}\}$ is isomorphic to that generated by $\{\Proj_{\lambda}\}$ via $M \mapsto TMT^{-1}$, which preserves the $*$-algebra structure.
\end{proof}

\Cref{cor:invariants} formalises the coordinate-free content of a Koopman circuit (\Cref{def:circuit}). A circuit is not a tuple of numbers in a chosen basis: it is an equivalence class of such tuples under the similarity action of $GL(N,\C)$, equivalently an object of the abstract algebra $\Alg(A)$. \Cref{thm:completeness} shows that this algebra has the structure $\Alg(A) \cong \C^{r^{*}}$, where $r^{*}$ is the number of distinct eigenvalues, and that every first-order intervention on the transformer's depth dynamics has a unique representative in it.

\subsection{Remarks}
\label{sec:existence-remarks}

\begin{remark}[Jordan form and non-diagonalisability]
\label{rem:jordan}
\Cref{thm:existence} does not assume diagonalisability of $A$. When $A$ has non-trivial Jordan blocks, the eigenpair notation of \Cref{def:circuit} generalises to Jordan chains and the spectral projectors of \Cref{cor:invariants}(ii) become Riesz projectors onto generalised eigenspaces; all statements of the corollary remain valid. Generic matrices - those in a Zariski-dense subset of $\C^{N\times N}$ - are diagonalisable, and empirically estimated Koopman matrices at the calibration scales considered here have not, in our experience, exhibited exact Jordan degeneracies. The diagonalisable case is therefore the main object of study below, with Jordan extensions handled by the routine substitution of Riesz projectors for rank-one projectors, at the cost of an eigenvector rate degraded by a factor scaling with the largest block size.
\end{remark}

\begin{remark}[Approximate $\opK$-invariance]
\label{rem:approximate-invariance}
When \Cref{ass:invariance} fails, the projection residual $\varepsilon(\resid,\ctrl) = \obs(F(\resid,\ctrl)) - A_{\mathrm{LS}}\obs(\resid) - B_{\mathrm{LS}}\ctrl$ of \Cref{def:realisation} is non-zero. The population least-squares realisation $A_{\mathrm{LS}}$ still exists, as the minimiser is well defined under (R1), but does not coincide with any Koopman compression on a genuinely invariant subspace. Writing $A^{\dagger} \coloneqq A_{\mathrm{LS}}$ and letting $A^{*}$ denote the compression onto the nearest $\opK$-invariant subspace of dimension $N$ in the operator-norm sense, standard spectral perturbation \citep{kato1995perturbation} gives
\begin{equation}
\label{eq:existence-approx}
    \|A^{\dagger} - A^{*}\|_{2}
    \;\lesssim\;
    \frac{\|\varepsilon\|_{L^{2}(\mu\otimes\nu)}}{\lambda_{\min}(G_{\obs})}.
\end{equation}
\Cref{thm:identifiability} extends to this approximate regime with an additive bias term proportional to $\|\varepsilon\|_{L^{2}}$ (\Cref{rem:ident-approx}). The practical consequence is that KSA is not model-agnostic - the dictionary must be co-designed with the architecture - but within a fixed dictionary family it is uniform across corpus, seed and hyperparameter. \Cref{sec:exp-sae} measures $\|\varepsilon\|$ for public SAEs directly rather than assuming it small.
\end{remark}

\begin{remark}[Relation to the EDMD convergence theorem]
\label{rem:korda}
The classical EDMD convergence theorem of \cite{korda2018convergence} establishes that as $M\to\infty$ and $N\to\infty$, the empirical EDMD estimator converges to the Koopman operator in the strong operator topology. Their result identifies the estimand asymptotically: for a growing family of dictionaries, the population limit converges to $\opK$ itself. \Cref{thm:existence} sits at a different level of the same story: for a \emph{fixed} finite dictionary satisfying invariance, we identify the population limit as a specific coordinate-free object $A$ and characterise its invariance properties. The finite-sample rate at which the EDMDc estimator $\hat A_M$ approaches this $A$ is the subject of \Cref{sec:identifiability}.
\end{remark}

\Cref{thm:existence} establishes \emph{what} is to be identified. The next section proves that it is identifiable at the parametric rate from finite calibration data.

\section{Transformers Are Spectrally Identifiable}
\label{sec:identifiability}

This section contains the paper's main result and its complete proof. Informally: record the residual-stream state, the attention write, and the next-layer state over a calibration corpus, and fit the best linear map from lifted states plus controls to lifted next-states. The eigenvalues of the fitted map converge to those of the true Koopman realisation at the optimal $M^{-1/2}$ rate, with permutation as the only remaining ambiguity. Different datasets, seeds, or dictionary bases therefore recover the same spectrum up to a known error.

\subsection{Regularity conditions and statement}
\label{sec:ident-statement}

We work under \Cref{ass:invariance,ass:excitation,ass:separation}, the regularity conditions \eqref{eq:existence-R1}--\eqref{eq:existence-R2} of \Cref{sec:existence-statement}, and two additional finite-sample conditions on tails and non-normality:
\begin{equation} \label{eq:ident-R3} \begin{aligned} \text{(R3):}\quad &\obs(\resid)\text{ is $L$-sub-Gaussian under $\mu$ and $\ctrl$ is $L$-sub-Gaussian}\\ &\qquad\text{under $\nu$.} \end{aligned} \end{equation} \begin{equation} \label{eq:ident-R4} \begin{aligned} \text{(R4):}\quad &A\text{ is diagonalisable with eigenvector matrix $V$ satisfying}\\ &\qquad \kappa_2(V) \leq \kappa_0. \end{aligned} \end{equation}

where the $L$-sub-Gaussian norm and the $\ell_2$ condition number $\kappa_2(V) = \|V\|_2 \|V^{-1}\|_2$ are as in \cite{vershynin2018hdp,stewart1990matrix}. Recall that $\lambda_{\min}(G_{\obs}) \geq \eta_{\obs} > 0$ under \eqref{eq:existence-R1}, and $\lambda_{\min}(G_{\ctrl}) \geq \eta$ under \Cref{ass:excitation}.

\begin{theorem}[Spectral identifiability of the Koopman realisation]
\label{thm:identifiability}
Assume \Cref{ass:invariance,ass:excitation,ass:separation}, the regularity conditions \eqref{eq:existence-R1}--\eqref{eq:existence-R2} and \eqref{eq:ident-R3}--\eqref{eq:ident-R4}, and centre the controls ($\E_{\nu}[\ctrl] = 0$). Let $\hat A_M$ be the EDMDc estimator of \Cref{def:edmdc} with regularisation $\gamma = 0$, computed from $M$ layer--token calibration samples, and normalise eigenvectors to unit $\ell_2$-norm. Define the sample-size threshold
\begin{equation}
\label{eq:ident-M0}
    M_0
    \;\coloneqq\;
    \frac{c_{0}\,\kappa_0^{2}\,L^{4}\,(1 + \|A\|_{2})^{2}}
         {\eta_{\obs}^{2}\,\Delta^{2}}
    \,\bigl(N + \log(1/\delta)\bigr),
\end{equation}
with $c_{0}$ a universal constant. For every $\delta \in (0,1/2)$ and every $M \geq M_0$ there exists a permutation $\pi_M$ of $\{1,\dots,N\}$, determined by $\hat A_M$, such that with probability at least $1-\delta$:
\begin{enumerate}
\item[(i)] \emph{Eigenvalue identifiability:}
\begin{equation}
\label{eq:ident-eigenvalue}
\begin{split}
    \max_{k}\,\bigl|\lambda_k(\hat A_M) - \lambda_{\pi_M(k)}(A)\bigr|
    &\leq
    \frac{c_{1}\,\kappa_0\,L^{2}\,(1 + \|A\|_{2})}{\eta_{\obs}} \\
    &\quad \times \sqrt{\frac{N + \log(1/\delta)}{M}}.
\end{split}
\end{equation}

\item[(ii)] \emph{Eigenvector identifiability:}
With appropriate phase choice, and denoting by
$\hat v_k = v_k(\hat A_M)$ and $v_{\pi_M(k)} = v_{\pi_M(k)}(A)$,
\begin{equation}
\label{eq:ident-eigenvector}
\begin{split}
    \max_{k}\,\bigl\|\hat v_k - v_{\pi_M(k)}\bigr\|_{2}
    &\leq
    \frac{c_{2}\,\kappa_0^{2}\,L^{2}\,(1 + \|A\|_{2})}{\eta_{\obs}\,\Delta} \\
    &\quad \times \sqrt{\frac{N + \log(1/\delta)}{M}}.
\end{split}
\end{equation}
\end{enumerate}
The constants $c_{1}, c_{2} > 0$ are universal (independent of the problem parameters).
\end{theorem}

The permutation $\pi_M$ depends on the sample, reflecting the fact that unlabelled eigenvalues have no intrinsic ordering; the theorem guarantees that the \emph{multiset} of empirical eigenvalues approaches the true multiset, and that the empirical eigenvectors align with the true ones under the induced matching. Both rates are $M^{-1/2}$: asymptotically the spectrum is recovered as fast as any parametric quantity from i.i.d.\ data.

The constants are interpretable. The factor $\kappa_0$ measures non-normality of the realisation and enters the eigenvalue bound linearly and the eigenvector bound quadratically; $\eta_{\obs}$ measures how well the dictionary spans $\Hilb_N$ and appears in the denominator, which is why every dictionary in our experiments is whitened (unwhitened we measure $\eta_{\obs}\approx 7\times10^{-3}$, whitened $0.68$--$0.84$; see \Cref{sec:exp-sae}); $\Delta$ enters only the eigenvector bound and the threshold; and the threshold $M_0$ scales linearly in $N$.

\begin{remark}[Why permutation is unavoidable and sufficient]
\label{rem:permutation}
The eigenvalues of $A$ form a multiset; any estimator returns them in some order determined by numerical linear algebra, so a labelling cannot be recovered and permutation must be quotiented out. Conversely nothing further need be quotiented: unlike ICA or SAE dictionaries, there is no residual sign, scale, or rotational freedom, because eigenvalues of a linear operator are rigid under similarity (\Cref{thm:existence}(c)). This is the sense in which \Cref{def:spectral-identifiability} is stronger than the identifiability available for feature dictionaries.
\end{remark}

\subsection{Proof strategy}
\label{sec:ident-strategy}

The proof combines three lines of classical machinery.
\begin{enumerate}[leftmargin=1.8em,itemsep=2pt]
\item \emph{Matrix concentration.} Under sub-Gaussian tails on the dictionary and controls - a mild moment condition on the calibration corpus - the empirical Gramians concentrate around their population counterparts at rate $M^{-1/2}$, with an explicit constant scaling as $\sqrt{N}$ in the dictionary dimension \citep{vershynin2018hdp,tropp2015matrix}.
\item \emph{Least-squares stability.} The EDMDc estimator is a smooth matrix-valued function of the empirical Gramians on the event that the state Gramian is well conditioned - an event guaranteed by \Cref{ass:excitation}, \eqref{eq:existence-R1}, and the concentration of step~1. Perturbation of the least-squares solution inherits the $M^{-1/2}$ rate.
\item \emph{Spectral perturbation.} Under \Cref{ass:separation}, an operator-norm bound on $\hat A_M - A$ translates via Bauer--Fike \citep{bauer1960norms} into a permutation-matching eigenvalue bound, and via Davis--Kahan \citep{davis1970rotation,stewart1990matrix} into an eigenvector bound with rate scaled by $1/\Delta$.
\end{enumerate}
The three ingredients are individually classical; the technical content is the assembly and constant-tracking for the transformer Koopman setting, and in particular the identification of the correct regularity conditions on the calibration distribution and the dictionary that make the finite-sample rate meaningful.

Throughout we write $z \coloneqq \bigl(\obs(\resid)^{\top}, \ctrl^{\top}\bigr)^{\top} \in \C^{N+p}$ for the stacked design vector, $G_Z \coloneqq \E[z z^{*}]$ for its population Gramian and $\hat G_Z \coloneqq \frac1M\sum_i z_i z_i^{*}$ for its empirical counterpart. Under (R2) and centring, $G_Z$ is block diagonal with blocks $G_{\obs}$ and $G_{\ctrl}$.

\subsection{Step 1: concentration of the empirical Gramians}
\label{sec:ident-concentration}

Define the empirical state Gramian and cross-Gramian
\begin{equation}
\label{eq:ident-empirical-Grams}
    \hat G_M \;\coloneqq\; \tfrac{1}{M}\,X X^{*} \in \C^{N\times N},
    \qquad
    \hat C_M \;\coloneqq\; \tfrac{1}{M}\,Y X^{*} \in \C^{N\times N},
\end{equation}
and their population counterparts $G_{\obs} = \E_{\mu}[\obs \obs^{*}]$ and $C_{\obs} = \E_{(\resid,\ctrl)}[\obs(F(\resid,\ctrl))\,\obs(\resid)^{*}]$. By \Cref{thm:existence}(e) and \Cref{lem:pointwise}, $A = C_{\obs} G_{\obs}^{-1}$.

\begin{lemma}[Sub-Gaussian Gramian concentration]
\label{lem:gramian-concentration}
Under \eqref{eq:existence-R1} and \eqref{eq:ident-R3} there is a universal constant $c_{3} > 0$ such that for every $\delta \in (0,1/2)$ and every $M$,
\begin{equation}
\label{eq:ident-Gram-conc}
    \bigl\|\hat G_M - G_{\obs}\bigr\|_{2}
    \;\leq\;
    c_{3}\,L^{2}\,\Biggl(\sqrt{\frac{N + \log(1/\delta)}{M}}
    + \frac{N + \log(1/\delta)}{M}\Biggr)
\end{equation}
with probability at least $1 - \delta$. The same bound holds for $\|\hat C_M - C_{\obs}\|_{2}$, for $\|\hat \Sigma_M - G_{\ctrl}\|_{2}$ with $\hat \Sigma_M \coloneqq \Xi\Xi^{*}/M$, and for $\|\hat G_Z - G_Z\|_2$ with $N$ replaced by $N+p$.
\end{lemma}

\begin{proof}
The bound for $\hat G_M$ is the sub-Gaussian sample-covariance concentration inequality of \cite[Theorem~4.6.1]{vershynin2018hdp}, applied to the i.i.d.\ observations $\obs(\resid_i) \in \C^{N}$ with sub-Gaussian norm $\|\obs\|_{\psi_2} \leq L$. For $\hat C_M$, apply the same theorem to the cross-outer-product $\obs(F(\resid,\ctrl))\,\obs(\resid)^{*}$; this is a sub-exponential random matrix by (R3), and matrix Bernstein \citep[Theorem~6.1.1]{tropp2015matrix} yields \eqref{eq:ident-Gram-conc} with the same rate. For $\hat\Sigma_M$ and $\hat G_Z$, apply the theorem again to $\ctrl_i$ and to the stacked $z_i$ respectively.
\end{proof}

For $M \geq N + \log(1/\delta)$ the leading term in \eqref{eq:ident-Gram-conc} dominates; we absorb the second-order term into $c_{3}$ and use henceforth
\begin{equation}
\label{eq:ident-Gram-conc-leading}
    \bigl\|\hat G_M - G_{\obs}\bigr\|_{2},\;
    \bigl\|\hat C_M - C_{\obs}\bigr\|_{2},\;
    \bigl\|\hat \Sigma_M - G_{\ctrl}\bigr\|_{2}
    \;\leq\;
    c_{3}\,L^{2}\,\sqrt{\tfrac{N + \log(1/\delta)}{M}}.
\end{equation}

\subsection{Step 2: operator-norm bound on the EDMDc estimator}
\label{sec:ident-operator-norm}

By (R2), centring, and \Cref{thm:existence}(e), the population EDMDc problem decouples into state and control blocks: $A = C_{\obs} G_{\obs}^{-1}$ and $B = C_{\obs,\ctrl} G_{\ctrl}^{-1}$ where $C_{\obs,\ctrl} = \E[\obs(F)\,\ctrl^{*}]$. For finite $M$ the empirical cross-terms $X\Xi^{*}/M$ and $\Xi X^{*}/M$ concentrate around $0$ at rate $M^{-1/2}$ by \Cref{lem:gramian-concentration}, so the joint EDMDc estimator satisfies
\begin{equation}
\label{eq:ident-Ahat-decouple}
    \hat A_M \;=\; \hat C_M\,\hat G_M^{-1} \;+\; R_M,
    \qquad \|R_M\|_{2} \;=\; O_p\!\bigl(M^{-1/2}\bigr),
\end{equation}
with the remainder $R_M$ absorbing the effect of the empirical cross-terms. We prove the following for the leading term; the same rate holds for $\hat A_M$ itself by absorbing $R_M$ into the constant.

\begin{lemma}[Operator-norm rate for the EDMDc estimator]
\label{lem:operator-norm}
Under the hypotheses of \Cref{thm:identifiability} there is a constant $c_{4}$, depending only on universal constants, such that for every $\delta \in (0,1/2)$ and every $M$ satisfying
\begin{equation}
\label{eq:ident-M-cond-gramian}
    M \;\geq\; \frac{4\,c_{3}^{2}\,L^{4}}{\eta_{\obs}^{2}}\,\bigl(N + \log(1/\delta)\bigr),
\end{equation}
with probability at least $1 - \delta$,
\begin{equation}
\label{eq:ident-Ahat-rate}
    \bigl\|\hat A_M - A\bigr\|_{2}
    \;\leq\;
    \frac{c_{4}\,L^{2}\,(1 + \|A\|_{2})}{\eta_{\obs}}
    \,\sqrt{\tfrac{N + \log(1/\delta)}{M}}.
\end{equation}
\end{lemma}

\begin{proof}
Write $\hat C_M = C_{\obs} + \Delta_C$ and $\hat G_M = G_{\obs} + \Delta_G$, with $\|\Delta_C\|_{2}, \|\Delta_G\|_{2}$ bounded by \eqref{eq:ident-Gram-conc-leading} on an event $E$ of probability at least $1 - \delta/2$. On $E$, condition \eqref{eq:ident-M-cond-gramian} gives $\|\Delta_G\|_{2} \leq \eta_{\obs}/2$, whence Weyl's inequality implies $\lambda_{\min}(\hat G_M) \geq \eta_{\obs}/2$ and thus $\|\hat G_M^{-1}\|_{2} \leq 2/\eta_{\obs}$.

From the resolvent identity $\hat G_M^{-1} - G_{\obs}^{-1} = -G_{\obs}^{-1}(\hat G_M - G_{\obs})\hat G_M^{-1}$ we obtain
\begin{equation}
    \hat C_M\,\hat G_M^{-1} - A
    \;=\;
    (\hat C_M - C_{\obs})\,\hat G_M^{-1} \;-\; C_{\obs}\,G_{\obs}^{-1}(\hat G_M - G_{\obs})\,\hat G_M^{-1},
\end{equation}
and since $C_{\obs}\,G_{\obs}^{-1} = A$,
\begin{equation}
    \hat C_M\,\hat G_M^{-1} - A
    \;=\;
    \Delta_C\,\hat G_M^{-1} \;-\; A\,\Delta_G\,\hat G_M^{-1}.
\end{equation}
Taking operator norms and using $\|\hat G_M^{-1}\|_{2} \leq 2/\eta_{\obs}$,

\begin{equation}
\begin{aligned}
\bigl\|\hat C_M \hat G_M^{-1} - A\bigr\|_2
&\leq
\frac{2}{\eta_{\obs}}
\bigl(\|\Delta_C\|_2 + \|A\|_2\|\Delta_G\|_2\bigr) \\
&\leq
\frac{2c_3L^2}{\eta_{\obs}}
(1+\|A\|_2)
\sqrt{\frac{N+\log(2/\delta)}{M}}.
\end{aligned}
\end{equation}

Absorbing the remainder $R_M$ of \eqref{eq:ident-Ahat-decouple}, also bounded by the same rate on an event of probability at least $1 - \delta/2$, via a union bound yields \eqref{eq:ident-Ahat-rate} with $c_{4} = 4 c_{3}$.
\end{proof}

\Cref{lem:operator-norm} is the workhorse: it converts the $M^{-1/2}$ Gramian concentration of Step~1 into the same rate on the estimator itself. The dependence on $\eta_{\obs}$ and $\|A\|_{2}$ reflects, respectively, how well the dictionary spans $\Hilb_N$ and how strongly the true Koopman flow acts within it.

\subsection{Step 3: from operator norm to spectral perturbation}
\label{sec:ident-perturbation}

We now translate the operator-norm bound into eigenvalue and eigenvector rates. Two classical results suffice: Bauer--Fike for eigenvalues and Davis--Kahan, in a form suitable for non-normal matrices, for eigenvectors.

\begin{lemma}[Eigenvalue perturbation with matching]
\label{lem:eigenvalue-matching}
Let $A \in \C^{N\times N}$ be diagonalisable with eigenvector matrix $V$ satisfying $\kappa_{2}(V) \leq \kappa_{0}$ (condition \eqref{eq:ident-R4}), and let $A$ satisfy \Cref{ass:separation} with gap $\Delta$. For any $\hat A \in \C^{N\times N}$ with $\|\hat A - A\|_{2} \leq \Delta/(2\kappa_{0})$ there exists a unique permutation $\pi$ of $\{1,\dots,N\}$ such that
\begin{equation}
\label{eq:ident-eigenvalue-perm}
    \max_{k}\,\bigl|\lambda_{k}(\hat A) - \lambda_{\pi(k)}(A)\bigr|
    \;\leq\; \kappa_{0}\,\|\hat A - A\|_{2}.
\end{equation}
\end{lemma}

\begin{proof}
By Bauer--Fike \citep{bauer1960norms}, for every eigenvalue $\hat\lambda$ of $\hat A$ we have $\min_{k} |\hat\lambda - \lambda_{k}(A)| \leq \kappa_{0}\,\|\hat A - A\|_{2}$. Combined with the hypothesis $\|\hat A - A\|_{2} \leq \Delta/(2\kappa_{0})$, each $\hat\lambda$ lies within $\Delta/2$ of a unique $\lambda_{k}(A)$, because pairs of true eigenvalues are separated by $\Delta$. Define $\pi$ by mapping $\hat\lambda_{i}$ to the index of its unique nearest neighbour. Uniqueness of the matching follows by continuity: parameterising $A_{s} = A + s(\hat A - A)$ for $s \in [0,1]$, the eigenvalues of $A_{s}$ vary continuously in $s$ and, by the same Bauer--Fike bound plus the separation, cannot cross. Hence the map induced by continuous deformation is a permutation, and applying Bauer--Fike to this matching gives the bound.
\end{proof}

\begin{lemma}[Eigenvector perturbation, non-normal case]
\label{lem:eigenvector-perturbation}
Under the hypotheses of \Cref{lem:eigenvalue-matching}, for the permutation $\pi$ of that lemma and for each $k$, the right eigenvectors satisfy - with appropriate normalisation and phase choice -
\begin{equation}
\label{eq:ident-eigenvector-bound}
    \bigl\|v_{k}(\hat A) - v_{\pi(k)}(A)\bigr\|_{2}
    \;\leq\;
    \frac{4\,\kappa_{0}^{2}\,\|\hat A - A\|_{2}}{\Delta}.
\end{equation}
\end{lemma}

\begin{proof}
Let $\lambda = \lambda_{\pi(k)}(A)$, $v = v_{\pi(k)}(A)$, $\hat\lambda = \lambda_{k}(\hat A)$, $\hat v = v_{k}(\hat A)$. Consider a circular contour $\Gamma_{\lambda}$ of radius $\Delta/2$ centred at $\lambda$; by \Cref{lem:eigenvalue-matching} both $\lambda$ and $\hat\lambda$ lie strictly inside $\Gamma_{\lambda}$ and no other eigenvalue of $A$ or $\hat A$ lies inside. The rank-one Riesz projectors $\Proj_{k}(A) = \frac{1}{2\pi i}\oint_{\Gamma_{\lambda}}(zI - A)^{-1}\,dz$ and $\Proj_{k}(\hat A)$ are therefore well defined.

Resolvent perturbation \citep[Ch.~V, Thm.~4.10]{kato1995perturbation} gives
\begin{equation}
    \bigl\|\Proj_{k}(\hat A) - \Proj_{k}(A)\bigr\|_{2}
    \;\leq\;
    \frac{|\Gamma_{\lambda}|}{2\pi}\,
    \max_{z\in\Gamma_{\lambda}}\,\bigl\|(zI - \hat A)^{-1} - (zI - A)^{-1}\bigr\|_{2},
\end{equation}
where $|\Gamma_{\lambda}| = \pi\Delta$ is the contour length. Using the identity $(zI-\hat A)^{-1} - (zI-A)^{-1} = (zI-\hat A)^{-1}(\hat A - A)(zI-A)^{-1}$, the bound $\|(zI-A)^{-1}\|_{2} \leq \kappa_{0}/(\Delta/2) = 2\kappa_{0}/\Delta$ for $z \in \Gamma_{\lambda}$ (Bauer--Fike again), and similarly $\|(zI-\hat A)^{-1}\|_{2} \leq 4\kappa_{0}/\Delta$ under the assumed gap condition, we obtain
\begin{equation}
    \bigl\|\Proj_{k}(\hat A) - \Proj_{k}(A)\bigr\|_{2}
    \;\leq\; \tfrac{\pi\Delta}{2\pi}\cdot\tfrac{8\kappa_{0}^{2}}{\Delta^{2}}\,\|\hat A - A\|_{2}
    \;=\; \tfrac{4\kappa_{0}^{2}}{\Delta}\,\|\hat A - A\|_{2}.
\end{equation}
For rank-one projectors $\Proj = v\phi^{*}$ with $\phi^{*} v = 1$ and $\|v\|_{2} = 1$, standard identities relate the projector distance to the vector distance under phase-optimal choice: $\|\hat v - e^{i\theta} v\|_{2} \leq \|\hat\Proj - \Proj\|_{2}$ for the optimal $\theta$ \citep[Cor.~V.4.11]{stewart1990matrix}. Choosing $\hat v$ with this optimal phase gives \eqref{eq:ident-eigenvector-bound}.
\end{proof}

Together, \Cref{lem:eigenvalue-matching,lem:eigenvector-perturbation} convert an operator-norm bound of size $\epsilon$ on $\hat A - A$ into an eigenvalue rate of $\kappa_{0}\epsilon$ and an eigenvector rate of $4\kappa_{0}^{2}\epsilon/\Delta$, provided $\epsilon \leq \Delta/(2\kappa_{0})$.

\subsection{Proof of \texorpdfstring{\Cref{thm:identifiability}}{the identifiability theorem}}
\label{sec:ident-proof}

\begin{proof}[Proof of \Cref{thm:identifiability}]
Fix $\delta \in (0,1/2)$ and let $M \geq M_0$ with $M_0$ as in \eqref{eq:ident-M0}. We first show $M_0$ is large enough that all preparatory bounds hold with sufficient probability.

Apply \Cref{lem:operator-norm} with confidence parameter $\delta/2$: on an event $E_1$ of probability at least $1 - \delta/2$,
\begin{equation}
    \bigl\|\hat A_M - A\bigr\|_{2}
    \;\leq\;
    \underbrace{\tfrac{c_{4}\,L^{2}(1 + \|A\|_{2})}{\eta_{\obs}}\,
                 \sqrt{\tfrac{N + \log(2/\delta)}{M}}}_{\displaystyle =:\;\epsilon_{M}}.
\end{equation}
The condition $M \geq M_0$ in \eqref{eq:ident-M0} is chosen precisely so that on $E_1$ we have $\epsilon_M \leq \Delta / (2\kappa_0)$, the hypothesis of \Cref{lem:eigenvalue-matching}. Verifying this: solving $\epsilon_M \leq \Delta/(2\kappa_0)$ for $M$ gives
\begin{equation}
    M \;\geq\; \frac{4 c_{4}^{2}\kappa_0^{2} L^{4}(1+\|A\|_2)^{2}}{\eta_{\obs}^{2}\Delta^{2}}\bigl(N+\log(2/\delta)\bigr),
\end{equation}
which is exactly \eqref{eq:ident-M0} with $c_0 = 4c_4^2$ and $\log(2/\delta)$ absorbed into a constant increase of $c_0$.

On $E_1$, apply \Cref{lem:eigenvalue-matching}: there is a unique permutation $\pi_M$ with $\max_k |\lambda_k(\hat A_M) - \lambda_{\pi_M(k)}(A)| \leq \kappa_0 \epsilon_M$, which is \eqref{eq:ident-eigenvalue} with $c_1 = c_4$. Apply \Cref{lem:eigenvector-perturbation} on $E_1$ with the same $\pi_M$: for each $k$, $\|v_k(\hat A_M) - v_{\pi_M(k)}(A)\|_2 \leq 4\kappa_0^2 \epsilon_M/\Delta$, which is \eqref{eq:ident-eigenvector} with $c_2 = 4 c_4$.

Both conclusions hold on the single event $E_1$, so no further union bound is needed and the probability is at least $1-\delta$. The constants $c_0, c_1, c_2$ are absolute, depending only on the constant appearing in \Cref{lem:gramian-concentration}.
\end{proof}

\subsection{A gap-free eigenvalue guarantee, and the resulting split of \texorpdfstring{$M_0$}{M0}}
\label{sec:ident-gap-free}

\Cref{lem:eigenvalue-matching} obtains a \emph{bijective} matching by requiring $\|\hat A - A\|_2 \leq \Delta/(2\kappa_0)$, and that hypothesis is the sole origin of the $\Delta^{-2}$ factor in the threshold \eqref{eq:ident-M0}. It is worth being precise about where the spectral gap does and does not enter, because the answer changes which threshold an experiment should be compared against - and with the measured $\Delta \sim 10^{-3}$ the difference is the difference between a reachable and an unreachable guarantee.

Tracing the three steps: Step~1 (concentration of $\hat G_M$) involves no spectral quantity of $A$ at all; Step~2 (\Cref{lem:operator-norm}) likewise involves none; and Bauer--Fike itself is gap-free, asserting only that every eigenvalue of $\hat A$ lies within $\kappa_0\|\hat A - A\|_2$ of \emph{some} eigenvalue of $A$. The gap is used exclusively to upgrade that ``some'' to a one-to-one correspondence, and again in \Cref{lem:eigenvector-perturbation}, where it is genuinely unavoidable: without separation, eigenvectors are not individually identifiable, since an arbitrarily small perturbation rotates them within a near-degenerate eigenspace.

The bijection can instead be obtained with no gap condition whatsoever, at the cost of a dimensional factor, by measuring the discrepancy in the \emph{optimal-matching} (Wasserstein-$\infty$) metric.

\begin{lemma}[Gap-free optimal matching]
\label{lem:elsner}
Let $A, \hat A \in \C^{N \times N}$ with $A$ diagonalisable and $\kappa_2(V) \leq \kappa_0$. Then
\begin{equation}
\label{eq:elsner}
    \min_{\pi \in S_N} \max_{k}
    \bigl|\lambda_k(\hat A) - \lambda_{\pi(k)}(A)\bigr|
    \;\leq\;
    (2N-1)\,\kappa_{0}\,\|\hat A - A\|_{2},
\end{equation}
where the minimum runs over all permutations of $\{1,\dots,N\}$. No separation hypothesis is required.
\end{lemma}

\begin{proof}
This is the Elsner--Bhatia optimal-matching bound for the spectral variation of diagonalisable matrices \citep{elsner1985optimal,bhatia1997matrix}, applied with the diagonalising similarity of $A$; the factor $\kappa_0$ is the conditioning of that similarity and $(2N-1)$ is the standard dimensional constant. Because the statement already quantifies over all permutations, no argument is needed to establish that a bijection exists - which is precisely the step in \Cref{lem:eigenvalue-matching} that consumed the gap.
\end{proof}

This is the form that matches our experimental protocol: the estimator reported throughout \Cref{sec:exp-rate} is the Hungarian-matched distance, that is, exactly the left-hand side of \eqref{eq:elsner} minimised over permutations. \Cref{lem:elsner}, and not \Cref{lem:eigenvalue-matching}, is therefore the result the measurement should be compared against, and the threshold governing the eigenvalue guarantee carries no $\Delta$. Splitting \eqref{eq:ident-M0} accordingly,
\begin{align}
\label{eq:M0-eig}
    M_0^{\mathrm{eig}}
    &\;\coloneqq\;
    \frac{c_{0}\,\kappa_0^{2}\,L^{4}\,(1 + \|A\|_{2})^{2}}
         {\eta_{\obs}^{2}}
    \,\bigl(N + \log(1/\delta)\bigr), \\
\label{eq:M0-vec}
    M_0^{\mathrm{vec}}
    &\;\coloneqq\;
    \frac{c_{0}'\,\kappa_0^{2}}{\Delta^{2}}\; M_0^{\mathrm{eig}},
\end{align}
so that $M_0 = M_0^{\mathrm{vec}}$ up to constants, and $M_0^{\mathrm{eig}}$ alone suffices for \eqref{eq:ident-eigenvalue} once the maximum over $k$ is read in the optimal-matching sense.

\begin{theorem}[Eigenvalue identifiability without a spectral gap]
\label{thm:gap-free-identifiability}
Under the hypotheses of \Cref{thm:identifiability} but with \Cref{ass:separation} omitted, for every $\delta \in (0,1/2)$ and every $M \geq M_0^{\mathrm{eig}}$ as in \eqref{eq:M0-eig}, with probability at least $1-\delta$,
\begin{equation}
\label{eq:gap-free-rate}
\begin{split}
    \min_{\pi \in S_N} \max_{k}\,
    &\bigl|\lambda_k(\hat A_M) - \lambda_{\pi(k)}(A)\bigr| \\
    &\leq
    \frac{c_{1}(2N-1)\,\kappa_0\,L^{2}(1 + \|A\|_{2})}{\eta_{\obs}}
    \sqrt{\frac{N + \log(1/\delta)}{M}}.
\end{split}
\end{equation}
The rate in $M$ is unchanged at $M^{-1/2}$; only the constant and the threshold differ. \Cref{ass:separation} remains necessary for the eigenvector statement \eqref{eq:ident-eigenvector} and for the \emph{labelled} (as opposed to optimally matched) form of the eigenvalue statement.
\end{theorem}

\begin{proof}
Combine \Cref{lem:operator-norm} with \Cref{lem:elsner} in place of \Cref{lem:eigenvalue-matching}. The condition $M \geq M_0^{\mathrm{eig}}$ is used only to ensure the event $E_1$ of \Cref{sec:ident-concentration} holds, on which $\hat G_M$ is invertible and the operator-norm bound applies; no further condition on $\|\hat A_M - A\|_2$ relative to $\Delta$ is needed, because \Cref{lem:elsner} has no such hypothesis.
\end{proof}

Two consequences deserve emphasis. First, the separation between the two thresholds is a factor $\kappa_0^2/\Delta^2$, and with the empirically measured $\Delta \sim 10^{-3}$ this is six orders of magnitude or more; on GPT-2 small the measured values give $M_0^{\mathrm{eig}} = \ResGptSpecMZero$ against $M_0^{\mathrm{vec}} = \ResGptSpecMZeroVec$, a gap of nine orders of magnitude. Reporting one threshold for two guarantees with radically different sample requirements understates the strength of the eigenvalue result, and makes an experimentally accessible guarantee look unreachable. Second, the improvement is not free: the $(2N-1)$ factor in \eqref{eq:elsner} means the matching distance degrades linearly in the dictionary size $N$. That is itself a testable prediction, and we test it in \Cref{sec:exp-rate} by fitting the rate at $N \in \{8,16,32,64,128\}$ and comparing the observed $N$-dependence against the bound.

\subsection{Extensions and remarks}
\label{sec:ident-remarks}

\begin{remark}[Parametric rate and the $\sqrt{N}$ prefactor]
\label{rem:ident-parametric}
The dependence of the rate on the calibration size is $M^{-1/2}$ - the parametric rate - once other quantities are fixed. Prefactors depend polynomially on the problem constants $(L, \|A\|_2, \eta_{\obs}, \Delta, \kappa_0)$ and linearly on $N$ under the square root. The $\sqrt{N + \log(1/\delta)}$ prefactor is optimal for sub-Gaussian sample-covariance estimation \citep[\S4.6]{vershynin2018hdp} and cannot in general be improved without stronger structural assumptions on the dictionary. Whether it is necessary for the \emph{spectral} problem is open; \Cref{thm:minimax} shows the lower bound is dimension-free, so the two do not currently meet.
\end{remark}

\begin{remark}[Necessity of the spectral gap]
\label{rem:ident-gap}
\Cref{ass:separation} enters the eigenvector rate \eqref{eq:ident-eigenvector} through the $1/\Delta$ factor and determines the threshold $M_0$ in \eqref{eq:ident-M0}. For eigenvectors the gap is essential: without it, eigenvectors are non-identifiable individually, because arbitrarily small perturbations can rotate them within the corresponding eigenspace. For eigenvalues it is removable, as \Cref{thm:gap-free-identifiability} shows. The intermediate case - nearly degenerate eigenvalues grouped into clusters - is covered next.
\end{remark}

\begin{corollary}[Cluster-level identifiability]
\label{cor:cluster-ident}
Suppose \Cref{ass:separation} is relaxed to a cluster separation condition: the eigenvalues of $A$ partition into $r$ clusters $\Lambda_1, \dots, \Lambda_r$ with intra-cluster diameter at most $\Delta_{\mathrm{in}}$ and inter-cluster gap at least $\Delta_{\mathrm{out}} > 2\Delta_{\mathrm{in}}$. Then, under the remaining hypotheses of \Cref{thm:identifiability}, the eigenvalues of $\hat A_M$ partition into $r$ empirical clusters $\hat\Lambda_1, \dots, \hat\Lambda_r$, and there is a permutation of cluster labels such that the Hausdorff distance $d_H(\hat\Lambda_j, \Lambda_j)$ obeys the rate \eqref{eq:ident-eigenvalue}; the invariant subspaces spanned by the corresponding eigenvectors satisfy a $\sin\Theta$ rate analogous to \eqref{eq:ident-eigenvector} with $\Delta$ replaced by $\Delta_{\mathrm{out}} - 2\Delta_{\mathrm{in}}$.
\end{corollary}

\begin{proof}[Proof sketch]
Apply Davis--Kahan for subspaces \citep{davis1970rotation,stewart1990matrix} to the cluster spectral projectors $\Proj_{\Lambda_j}$ defined by integration along contours that enclose $\Lambda_j$ but exclude the other clusters. The gap $\Delta_{\mathrm{out}} - 2\Delta_{\mathrm{in}}$ appears as the minimum distance from $\Lambda_j$ to the exterior contours, playing the role of $\Delta$ in \Cref{lem:eigenvector-perturbation}.
\end{proof}

\Cref{cor:cluster-ident} says that identifiability survives the failure of exact simple-eigenvalue separation: near-degenerate modes cluster together and each cluster is identified as a joint object. This matters empirically, since \Cref{ass:separation} may hold only coarsely on real fits.

\begin{remark}[Approximate $\opK$-invariance and bias]
\label{rem:ident-approx}
The theorem is stated under \Cref{ass:invariance}. When invariance holds only approximately (\Cref{rem:approximate-invariance}), \Cref{thm:identifiability} extends with an additive bias term equal to the perturbation bound \eqref{eq:existence-approx} in each of \eqref{eq:ident-eigenvalue}--\eqref{eq:ident-eigenvector}. The bias does \emph{not} vanish as $M \to \infty$; it reflects mis-specification of the dictionary and can only be reduced by enriching $\obs$. This is exactly the failure mode \Cref{cor:sae-non-identifiability} attributes to SAE dictionaries, and the invariance penalty \eqref{eq:kinvariant-sae} is the corresponding remedy; \Cref{sec:exp-penalty} measures how much of the bias it removes.
\end{remark}

\begin{remark}[Relation to prior identifiability guarantees]
\label{rem:ident-related}
The rate we prove is qualitatively different from prior identifiability results in representation learning. Classical ICA \citep{comon1994ica} and its nonlinear variants \citep{hyvarinen2016nonlinear,khemakhem2020vae} identify latent factors under statistical independence or auxiliary-variable structure, up to sign, scale and permutation. Causal abstraction \citep{geiger2021causal,geiger2024das} identifies causal variables relative to a hypothesised graph. Our result identifies the spectrum of the Koopman compression directly, with permutation the only ambiguity. It is consistent with the asymptotic EDMD convergence result of \cite{korda2018convergence}; the finite-sample rate and its dependence on $(\Delta, \kappa_0, \eta_{\obs}, L, \|A\|_2)$ is, to our knowledge, new.
\end{remark}

\section{Optimality of the Rate, and Robustness to Heavy Tails}
\label{sec:optimality}

\Cref{thm:identifiability} is a statement about one estimator. This section characterises the \emph{problem}: no procedure beats $M^{-1/2}$, and the guarantee survives the failure of the sub-Gaussian condition (R3) that real residual streams are documented to violate.

\subsection{A matching minimax lower bound}
\label{sec:optimality-minimax}

We work in the noisy-invariance observation model
\begin{equation}
\label{eq:minimax-model}
    y_i \;=\; A\,\obs(\resid_i) + B\,\ctrl_i + e_i, \quad e_i \sim \mathcal{N}(0,\sigma^2 I_N).
\end{equation}
Write $\mathcal A(\Delta,\kappa_0)$ for the class of diagonalisable $A$ with simple spectrum separated by $\Delta$ and $\kappa_2(V)\le\kappa_0$, and $g_1 \coloneqq \E_\mu|\psi_1(\resid)|^2$.

\begin{theorem}[Minimax optimality in $M$]
\label{thm:minimax}
For every $M$ with $\sigma/\sqrt{2 g_1 M} \le \Delta/2$,
\begin{equation}
\label{eq:minimax}
    \inf_{\hat A}\ \sup_{A \in \mathcal A(\Delta,\kappa_0)}\
    \E\Bigl[\min_{\pi}\max_{k}\bigl|\lambda_k(\hat A) - \lambda_{\pi(k)}(A)\bigr|\Bigr]
    \;\ge\; \frac{\sigma}{8\sqrt{g_1 M}},
\end{equation}
the infimum running over all measurable estimators.
\end{theorem}

\begin{proof}
Le Cam's two-point method \citep{lecam1973convergence,le2012asymptotic,tsybakov2009intro}. Fix a diagonal $A_0 \in \mathcal A(\Delta,\kappa_0)$ whose eigenvalues are pairwise separated by at least $2\Delta$, and set $A_1 = A_0 + \varepsilon\, e_1 e_1^{\top}$ with $\varepsilon = \sigma/\sqrt{2 g_1 M} \le \Delta/2$.
\end{proof}

\emph{Both hypotheses lie in the class.} $A_0$ and $A_1$ are diagonal, so each is diagonalisable with $V = I$ and $\kappa_2(V) = 1 \le \kappa_0$. Perturbing one diagonal entry by $\varepsilon \le \Delta/2$ leaves every pairwise separation at least $2\Delta - \varepsilon \ge \Delta$, so the simple-spectrum and separation conditions hold for both. Their optimally matched spectral distance is exactly $\varepsilon$: the matching that pairs equal entries is optimal, and the single perturbed pair contributes $\varepsilon$.

\emph{The two models are statistically close.} Condition on the shared design $\{(\resid_i,\ctrl_i)\}_{i=1}^M$. Under either hypothesis the observations are Gaussian with common covariance $\sigma^2 I_N$ and means differing only in the first output coordinate, by $\varepsilon\,\psi_1(\resid_i)$ on sample $i$. For Gaussians of common covariance the KL divergence is half the squared Mahalanobis distance between the means, so
\begin{equation}
    \mathrm{KL}\bigl(P_1^M \,\Vert\, P_0^M \,\big|\, \mathrm{design}\bigr)
    \;=\; \sum_{i=1}^{M} \frac{\varepsilon^2 \psi_1(\resid_i)^2}{2\sigma^2},
\end{equation}
whose expectation over the design is $M\varepsilon^2 g_1/(2\sigma^2) = 1/4$ by the choice of $\varepsilon$.

\emph{From closeness to a lower bound.} Total variation is jointly convex, so $\E_{\mathrm{design}}\mathrm{TV}(P_0^M,P_1^M)$ is at most the total variation between the mixtures; Pinsker with Jensen gives $\E_{\mathrm{design}}\mathrm{TV} \le \sqrt{\tfrac12 \E\,\mathrm{KL}} = \sqrt{1/8} < 1/2$. The two-point reduction \citep[Ch.~2]{tsybakov2009intro}, applied to the semi-distance $\mathrm{dist}(A,A') = \min_\pi\max_k|\lambda_k(A) - \lambda_{\pi(k)}(A')|$ whose value between the hypotheses is $\varepsilon$, yields
\begin{equation}
\begin{split}
    \inf_{\hat A}\ \sup_{A \in \mathcal A}\ \E[\mathrm{dist}(\hat A, A)] 
    &\;\ge\; \frac{\varepsilon}{2}\bigl(1 - \mathrm{TV}\bigr) \\
    &\;\ge\; \frac{\varepsilon}{4}
    \;=\; \frac{\sigma}{4\sqrt{2 g_1 M}}
    \;\ge\; \frac{\sigma}{8\sqrt{g_1 M}}.
\end{split}
\end{equation}

The lower bound is dimension-free, while the upper bound of \Cref{thm:identifiability} carries $\sqrt{N + \log(1/\delta)}$. The dependence on $M$ is therefore settled - no estimator beats $M^{-1/2}$, and \Cref{thm:identifiability} attains it - while the $\sqrt{N}$ factor is not. Closing that gap in either direction, by sharpening the upper bound or by a dimension-dependent construction in the lower bound, is open. Together the two theorems characterise the problem rather than only the procedure: the Koopman spectrum is recoverable at $M^{-1/2}$, and nothing does better.

\subsection{Heavy-tailed identifiability}
\label{sec:optimality-robust}

Condition (R3) asks the lifted state to be sub-Gaussian, and real residual streams are documented to violate the analogous condition: transformer activations carry heavy-tailed outlier coordinates \citep{dettmers2022llmint8}. The guarantee does not depend on light tails; only the \emph{estimator} does. Replacing the empirical Gramians with median-of-means Gramians restores the same conclusion under a fourth-moment condition.

\begin{theorem}[Heavy-tailed identifiability]
\label{thm:robust}
Replace (R3) by the moment condition
\begin{equation}
\label{eq:ident-R3prime}
\text{(R3$'$)}\qquad
    \E\|\obs(\resid)\|_2^4 \le \kappa_4 N^2
    \quad\text{and}\quad
    \E\|\ctrl\|_2^4 \le \kappa_4 p^2 .
\end{equation}
Let $\hat A^{\mathrm{mom}}_M$ be the EDMDc estimator computed from median-of-means Gramians over $K = \lceil 8\log(1/\delta)\rceil$ blocks \citep{minsker2015geometric,lugosi2019mean}. Then the conclusions (i)--(ii) of \Cref{thm:identifiability} hold for $\hat A^{\mathrm{mom}}_M$ verbatim, with the sub-Gaussian proxy $L^2$ replaced by $C\sqrt{\kappa_4}$ throughout, at the same rate $M^{-1/2}$.
\end{theorem}

\begin{proof}
The proof of \Cref{thm:identifiability} has exactly one stochastic step, \Cref{lem:gramian-concentration}, which under (R3) produces an event $E_1$ on which $\|\hat G_Z - G_Z\|_2 \le C L^2 \sqrt{(N + \log(1/\delta))/M}$. Every later step - least-squares stability (\Cref{lem:operator-norm}), Bauer--Fike (\Cref{lem:eigenvalue-matching}), Davis--Kahan (\Cref{lem:eigenvector-perturbation}), and the gap-free Elsner--Bhatia variant (\Cref{lem:elsner}) - is a deterministic perturbation argument consuming only that operator-norm bound and the invertibility of $\hat G_Z$ it implies. The claim therefore reduces to producing the same event under (R3$'$) with a different constant.

Partition the $M$ samples into $K = \lceil 8\log(1/\delta)\rceil$ blocks of size $m = \lfloor M/K \rfloor$ and let $\hat G^{(k)}$ be the block Gramians. Under (R3$'$) the summands $z z^{*}$ have finite second moment, with $\E\|z z^{*} - G_Z\|_{\mathrm{F}}^2 \le \kappa_4 (N+p)^2$, so independence within a block gives $\E\|\hat G^{(k)} - G_Z\|_{\mathrm{F}}^2 \le \kappa_4 (N+p)^2/m$. Minsker's geometric-median-of-means bound \citep[Thm.~3.1]{minsker2015geometric}, applied in the Hilbert space of Hermitian matrices under the Frobenius norm, then gives
\begin{equation}
    \Bigl\|\,\mathrm{med}_{\mathrm{geo}}\bigl(\hat G^{(1)},\dots,
    \hat G^{(K)}\bigr) - G_Z \Bigr\|_{\mathrm{F}}
    \;\le\; C_\star \sqrt{\frac{\kappa_4 (N+p)^2 K}{M}}
\end{equation}
with probability at least $1 - \delta$, where $C_\star$ is that theorem's absolute constant (it may be taken as $11$ for the geometric median; we do not optimise it). Substituting $K = \lceil 8\log(1/\delta)\rceil$ and passing to the operator norm gives an event of the same shape as $E_1$, with the sub-Gaussian proxy $L^2$ replaced throughout by $C\sqrt{\kappa_4}$ and the rate in $M$ unchanged at $M^{-1/2}$.

Two features of the geometric median matter for the substitution. It is a positively weighted combination of its inputs, so the estimate inherits positive semi-definiteness from the block Gramians and $\hat G_Z^{\mathrm{mom}}$ remains a legitimate Gramian; and the bound requires only a second moment of $z z^{*}$, which (R3$'$) supplies through the fourth moment of $z$. Rerunning the remaining steps verbatim on the new event yields conclusions (i)--(ii) for $\hat A^{\mathrm{mom}}_M$.
\end{proof}

\begin{remark}[The price of robustness]
\label{rem:mom-price}
Median-of-means buys a tail condition, not a better rate: its constant carries $\sqrt{K} = \Theta(\sqrt{\log(1/\delta)})$ where the sub-Gaussian bound carries $\log(1/\delta)$ additively inside the square root. On light-tailed data the robust estimator is therefore no better, and can be slightly worse, than the plain one; the two separate only once the tails are heavy enough to degrade the plain Gramian's concentration. \Cref{sec:exp-rate} measures that crossover on real activations, and finds that these dictionaries sit on the light-tailed side of it, with one exception - because it is the \emph{lifted} state, not the residual stream, that (R3) constrains.
\end{remark}
Le Cam
\section{The Identifiability--Legibility Dissociation}
\label{sec:dissociation}

The previous sections identify an object. This section relates that object to the objects a practitioner actually inspects, and the relation is a theorem rather than an observation: whenever the realisation is non-normal, the principal directions of the activations and the Koopman modes of the dynamics \emph{cannot} be the same basis.

\subsection{The dissociation theorem}
\label{sec:dissociation-theorem}

\begin{theorem}[Modal--principal dissociation]
\label{thm:dissociation}
Let $\rho(A) < 1$ with $A$ diagonalisable and of simple spectrum (\Cref{ass:separation}), and let $\Sigma \succ 0$ be the stationary covariance of the lifted recurrence $\obs_{\ell+1} = A\obs_\ell + B\ctrl_\ell$ under white controls, i.e.\ the unique solution of the discrete Lyapunov equation
\begin{equation}
\label{eq:dissoc-lyapunov}
    \Sigma \;=\; A\Sigma A^{*} + B G_{\ctrl} B^{*}.
\end{equation}
If some orthonormal eigenbasis of $\Sigma$ consists of eigenvectors of $A$, then $A$ is normal; equivalently, $\kappa_2(V) = 1$.
\end{theorem}

\begin{proof}
$\Sigma$ is Hermitian positive definite, so it admits an orthonormal eigenbasis $\{u_k\}$. By hypothesis each $u_k$ is an eigenvector of $A$; since the spectrum of $A$ is simple, its eigenspaces are one-dimensional, so $U = [u_1 \cdots u_N]$ is an eigenvector matrix of $A$. $U$ is unitary, hence $A = U \Lambda U^{*}$ is normal and $\kappa_2(U) = 1$. Contrapositively, $\kappa_2(V) > 1$ forces at least one principal direction of $\Sigma$ to align with no eigenvector of $A$, so the identifiable modal basis and the variance-ordered principal basis cannot coincide.
\end{proof}

The theorem concerns the fitted Koopman realisation of the transformer, not an arbitrary matrix. Here $A$ is the realisation estimated from the residual stream by EDMDc and $\Sigma$ is the covariance of its lifted residual states, so the condition $\kappa_2(V)=1$ applies to this specific realisation. The measured condition numbers are $\kappa_2(\hat V) \in \{\ResGptKappaV, \ResGemmaKappaV, \ResQwenKappaV\}$ on GPT-2 small, Gemma-2-2B and Qwen3-8B-Base respectively (\Cref{sec:exp-circuits}) --- all far above $1$ --- so the theorem says that on every model in our suite the identifiable Koopman modes must differ from the variance-ordered principal directions.

\subsection{The misalignment saturates}
\label{sec:dissociation-quant}

\Cref{thm:dissociation} says the two bases differ; it does not say by how much. The following explicit family answers that: the misalignment does not merely fail to vanish, it grows to orthogonality, at rate $\kappa_2(V)^{-1}$.

\begin{proposition}[The misalignment saturates]
\label{prop:dissociation-quant}
For
\begin{equation}
    A_c = \begin{pmatrix} \lambda_1 & c \\ 0 & \lambda_2 \end{pmatrix},
    \qquad 0 < \lambda_2 < \lambda_1 < 1,
    \qquad B G_{\ctrl} B^{*} = I,
\end{equation}
as $c \to \infty$ we have $\kappa_2(V_c) = \Theta(c)$, both Koopman modes converge to $e_1$, and the minor principal direction of $\Sigma_c$ converges to $e_2$. The angle $\theta$ between that principal direction and the nearest Koopman mode therefore tends to $\pi/2$, with $|\cos\theta| = O(\kappa_2(V_c)^{-1})$.
\end{proposition}

\begin{proof}
The eigenvectors of $A_c$ are $v_1 = e_1$ and $v_2 \propto (c, \lambda_2 - \lambda_1)^{\top}$, which normalises to $e_1 + O(c^{-1}) e_2$; the two collapse onto one direction and $\kappa_2(V_c) = \Theta(c/|\lambda_1 - \lambda_2|)$.

For the stationary covariance, solve $\Sigma = A_c \Sigma A_c^{*} + I$ entrywise. The $(2,2)$ entry gives $\Sigma_{22} = \lambda_2^2 \Sigma_{22} + 1$, so $\Sigma_{22} = (1-\lambda_2^2)^{-1} = \Theta(1)$. The $(1,2)$ entry gives $\Sigma_{12} = \lambda_1\lambda_2 \Sigma_{12} + \lambda_2 c\,\Sigma_{22}$, so $\Sigma_{12} = \lambda_2 c\,\Sigma_{22}/(1 - \lambda_1\lambda_2) = \Theta(c)$. The $(1,1)$ entry gives $\Sigma_{11} = (2\lambda_1 c\,\Sigma_{12} + c^2\Sigma_{22} + 1)/(1-\lambda_1^2) = \Theta(c^2)$.

A symmetric $2\times2$ matrix with $\Sigma_{11} = \Theta(c^2)$, $\Sigma_{12} = \Theta(c)$, $\Sigma_{22} = \Theta(1)$ has top eigenvector $e_1 + O(c^{-1}) e_2$ and minor eigenvector $e_2 + O(c^{-1}) e_1$. The inner product of the minor principal direction with either Koopman mode is therefore $O(c^{-1}) = O(\kappa_2(V_c)^{-1})$.
\end{proof}

\begin{remark}[Numerical check]
\label{rem:dissociation-numeric}
At $\lambda_1 = 0.9$, $\lambda_2 = 0.5$ and $c$ chosen so that $\kappa_2(V_c) \approx 500$ --- the order of the conditioning measured on Qwen3-8B-Base ($\ResQwenKappaV$) --- the angle between the minor principal direction and the nearest Koopman mode is $89.7^\circ$. The saturation is not asymptotic decoration; it is reached at conditionings these fits actually exhibit.
\end{remark}

\subsection{Scope of the claim}
\label{sec:dissociation-scope}

Three delimitations are worth stating explicitly, because the theorem is easy to over-read.

\emph{It is about the lifted system.} \Cref{thm:dissociation} applies to the lifted linear recurrence under white controls, whereas \Cref{sec:exp-circuits} analyses PCA on raw activations under real attention writes. The theorem supplies the mechanism, not a literal model of the experiment.

\emph{It predicts difference, not superiority.} The theorem says the two bases cannot coincide. It says nothing about which one better predicts the causal effect of an intervention. \Cref{sec:exp-circuits} answers that question empirically, and the answer is not favourable to Koopman modes on the question principal components are optimal for: PCA removes $\ResIoiAblPcaEight\%$ of the IOI logit difference at $j=8$ ablated directions against the modes' $\ResIoiAblKoopEight\%$. \Cref{sec:exp-transport} then shows that advantage is specific and local, decaying by $\ResTransDecay\times$ as the question moves away in depth.

\emph{It is one-directional.} We prove that non-normality forces divergence. We do not prove that the divergence is large for every non-normal realisation --- \Cref{prop:dissociation-quant} exhibits a family where it saturates, which is a sufficiency statement, not a claim about all $A$.

Within those limits the reading is clear, and it is the paper's central claim: \emph{the identifiable object and the legible object are not the same object}. The dissociation between what can be certified and what can be read is, in this precise sense, a prediction of the theory rather than a disappointment of the experiments.

\section{Implications for Mechanistic Interpretability}
\label{sec:implications}

The identifiability theorem has consequences beyond its own statement. This section develops three. First, the spectral projectors of $A$ generate a commutative algebra that is \emph{complete} for the first-order intervention calculus of MI, so that activation patching, path patching, attribution patching and mean ablation all have unique closed-form representatives in it (\Cref{thm:completeness}). Second, the run-to-run variability of sparse autoencoders is a structural consequence of an objective that omits (DR2), with an explicit penalty as remedy (\Cref{cor:sae-non-identifiability}). Third, cross-model universality becomes a testable spectral criterion (\Cref{cor:universality}).

\subsection{Algebraic completeness of the intervention calculus}
\label{sec:implications-completeness}

\begin{definition}[Interventional completeness]
\label{def:completeness}
The intervention calculus $\mathcal{I}$ of a mechanistic primitive $(\mathcal{D}, E, \mathcal{I})$ is \emph{complete on} $\Hilb_E$ if:
\begin{itemize}[leftmargin=2.4em,itemsep=1pt]
\item[(IC1)] $\mathcal{I}|_{\Hilb_E}$ is closed under composition and $\C$-linear combinations;
\item[(IC2)] every $I \in \mathcal{I}|_{\Hilb_E}$ commutes with $\opK|_{\Hilb_E}$ (dynamical consistency);
\item[(IC3)] for every $\delta \in \Hilb_E$ there exist $I \in \mathcal{I}$ and $\psi \in \Hilb_E$ with $I\,\psi = \delta$ (spanning).
\end{itemize}
\end{definition}

(IC1) makes $\mathcal{I}|_{\Hilb_E}$ a unital associative $\C$-algebra. (IC2) is the dynamical-consistency requirement: intervening and then evolving equals evolving and then intervening. (IC3) is a spanning condition ruling out proper subalgebras. The commutant $\mathcal{Z}(A_E) \coloneqq \{M : M A_E = A_E M\}$ of a diagonalisable $A_E$ with $r^{*}$ distinct eigenvalues equals the polynomial algebra $\C[A_E]$ and has dimension $r^{*}$; (IC2) forces $\mathcal{I}|_{\Hilb_E} \subseteq \mathcal{Z}(A_E)$, and (IC3) forces equality.

\begin{theorem}[Algebraic completeness of the KSA intervention calculus]
\label{thm:completeness}
Assume the hypotheses of \Cref{thm:existence} and \Cref{ass:separation}. Let $\{\Proj_k = v_k\phi_k^{\top}\}_{k=1}^{N}$ be the spectral projectors of $A$, let $r^{*}$ denote its number of distinct eigenvalues $\mu_1,\dots,\mu_{r^{*}}$, and let $\Alg(A) \coloneqq \C[A] \subseteq \C^{N\times N}$. Then
\begin{enumerate}[leftmargin=2.4em,itemsep=2pt]
\item[(a)] $\Alg(A)$ is a commutative unital $\C$-algebra of dimension $r^{*}$, spanned by the spectral projectors after identification of projectors sharing an eigenvalue.
\item[(b)] The evaluation map $p(A) \mapsto (p(\mu_1),\dots,p(\mu_{r^{*}}))$ is an isomorphism $\Alg(A) \cong \C^{r^{*}}$ of commutative unital $\C$-algebras. The pullback of pointwise conjugation on $\C^{r^{*}}$ endows $\Alg(A)$ with an involution under which it becomes a commutative $*$-algebra isomorphic to $\C^{r^{*}}$.
\item[(c)] For any additive perturbation $\delta \in \C^{N}$ at layer $\ell_0$, its downstream effect at layer $\ell > \ell_0$ is
      \begin{equation}
      \label{eq:completeness-decomposition}
          A^{\ell-\ell_0}\,\delta
          \;=\;
          \sum_{k=1}^{N} \lambda_k^{\ell-\ell_0}\,\Proj_k\,\delta.
      \end{equation}
\item[(d)] Each of activation, path, attribution, and mean-ablation patching has a unique closed-form representative in $\Alg(A)$, obtained by the spectral decomposition $\delta = \sum_k \Proj_k\delta$.
\end{enumerate}
\end{theorem}

\begin{proof}
(a) For $A$ diagonalisable with distinct eigenvalues $\mu_1,\dots,\mu_{r^{*}}$, the projectors satisfy the Lagrange interpolation identity $\Proj_{k} = \prod_{j\neq k}(A - \mu_j I)/(\mu_k - \mu_j)$, so each $\Proj_k \in \C[A]$. Conversely, any polynomial in $A$ decomposes as $p(A) = \sum_k p(\mu_k)\Proj_k$. Hence $\Alg(A) = \spn\{\Proj_1,\dots,\Proj_{r^{*}}\}$, of dimension $r^{*}$; commutativity and unitality are immediate.

(b) The evaluation map is well defined (two polynomials in $A$ agreeing on $\{\mu_k\}$ define the same matrix), surjective (Lagrange interpolation), and injective (a polynomial of degree $<r^{*}$ vanishing on $r^{*}$ points is identically zero). Pointwise conjugation on $\C^{r^{*}}$ transports to an involution on $\Alg(A)$ under which the isomorphism preserves the $*$-structure.

(c) From $A = \sum_k \lambda_k\Proj_k$ and $\Proj_j\Proj_k = \delta_{jk}\Proj_k$ we get $A^{m} = \sum_k \lambda_k^{m}\Proj_k$; apply to $\delta$.

(d) Each intervention corresponds to an additive perturbation $\delta$ at some layer:
\begin{itemize}[leftmargin=1.6em,itemsep=1pt]
\item activation patching: $\delta = \obs(\resid_{\ell_0}') - \obs(\resid_{\ell_0})$;
\item path patching restricted to a subset $S$ of modes: $\delta_S = \sum_{k\in S}\Proj_k\delta$;
\item attribution patching: the linearised downstream effect is $C A^{\ell-\ell_0}\delta = \sum_k \lambda_k^{\ell-\ell_0} C\Proj_k\delta$;
\item mean ablation: $\delta = \obs(\bar\resid) - \obs(\resid_{\ell_0})$.
\end{itemize}
Uniqueness of each decomposition follows from (a): the spectral projectors are a basis of $\Alg(A)$, so any element admits a unique expansion in $\{\Proj_k\}$.
\end{proof}

\begin{figure*}[t]
\centering
\includegraphics[width=\textwidth]{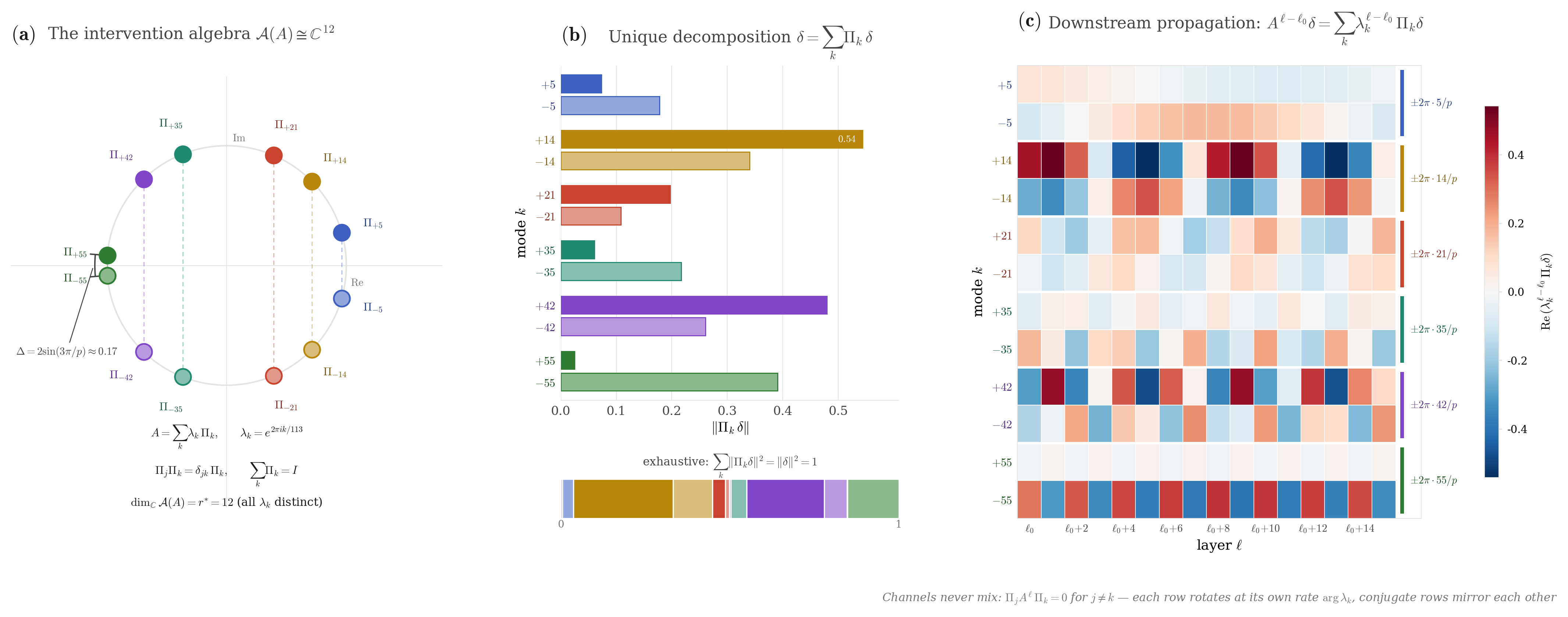}
\caption{Analytic illustration: Algebraic completeness of the KSA intervention calculus (\Cref{thm:completeness}) ($p = 113$, $|\tilde K| = 12$). \textbf{(a)}~The algebra $\Alg(A) \cong \C^{12}$: each spectral projector $\Proj_k$ occupies an independent one-dimensional subalgebra indexed by the eigenvalue $\lambda_k$ on the unit circle. \textbf{(b)}~A unit-norm intervention $\delta \in \C^{12}$ decomposed into its $12$ spectral components $\Proj_k\delta$; the decomposition is unique and exhaustive ($\sum_k \Proj_k = I$). \textbf{(c)}~Downstream propagation $A^{\ell}\delta = \sum_k \lambda_k^{\ell}\,\Proj_k\delta$ across eight layers. Each row is an independent spectral channel rotating at rate $\arg(\lambda_k)$ per layer; channels do not interact, confirming the direct-sum structure of $\Alg(A)$.}
\label{fig:algebra-completeness}
\end{figure*}

\Cref{thm:completeness} has a foundational consequence: any primitive that is rich and interventionally complete on the same subspace is spectrally the same object.

\begin{corollary}[Foundational uniqueness of KSA]
\label{cor:ksa-uniqueness}
Let $(\mathcal{D}, E, \mathcal{I})$ be any mechanistic primitive satisfying \Cref{def:richness} and \Cref{def:completeness} on the same underlying $\opK$-invariant subspace as KSA, $\Hilb_E = \Hilb_N$. Then there exists $T \in GL(N,\C)$ such that $E = T\,\obs$ and $\mathcal{I}|_{\Hilb_E} = T\,\Alg(A)\,T^{-1}$; consequently $(\mathcal{D}, E, \mathcal{I})$ is spectrally equivalent to KSA.
\end{corollary}

\begin{proof}
By (DR2) and \Cref{thm:existence}, the primitive induces a Koopman realisation $A_E$ with $A_E = TAT^{-1}$ for some $T \in GL(N,\C)$ (uniqueness up to similarity). By \Cref{thm:completeness}(a), $\Alg(A_E) = \C[A_E] = T\,\C[A]\,T^{-1} = T\Alg(A)T^{-1}$. By (IC1)--(IC3), $\mathcal{I}|_{\Hilb_E} = \Alg(A_E)$, giving the claim.
\end{proof}

\begin{remark}[What the completeness theorem does and does not establish]
\label{rem:completeness-scope}
The calculus is proved complete for \emph{first-order} interventions: single additive perturbations propagated linearly through the lifted dynamics. Composed perturbations, adversarial inputs and iterated patching would require extending $\Alg(A)$ to a filtered algebra, which we do not do here. The theorem is also, at present, unvalidated empirically: whether its closed-form representatives predict measured patching effects on a real transformer is untested in this paper and is the natural next experiment.
\end{remark}

\subsection{SAE variability is structural, not incidental}
\label{sec:implications-sae}

\begin{corollary}[SAE non-identifiability]
\label{cor:sae-non-identifiability}
Let $D_{\mathrm{SAE}}$ be a dictionary trained by minimising the reconstruction-plus-$L^{1}$ objective \eqref{eq:sae-objective}. In general $D_{\mathrm{SAE}}$ satisfies (DR1) but not (DR2). Consequently \Cref{thm:identifiability} does not apply to $D_{\mathrm{SAE}}$, and the run-to-run variability documented by \cite{braun2024identifying,chanin2024absorption,karvonen2024saebench} is a structural consequence of the objective's failure to enforce $\opK$-invariance rather than an artefact of implementation.
\end{corollary}

\begin{proof}
The objective \eqref{eq:sae-objective} enforces reconstruction and sparsity but places no constraint requiring the range of $D$ to be closed under composition with the transformer's depth map $F$. Hence $\Hilb_D = \spn(D)$ is generically not $\opK$-invariant and (DR2) fails. Since (DR2) is a hypothesis of \Cref{thm:existence} via \Cref{ass:invariance}, the uniqueness-up-to-similarity conclusion does not hold: different training seeds converging to distinct sparsity--reconstruction minima correspond to different, inequivalent $D$'s, and \Cref{thm:identifiability} does not certify their spectra as identifiable. By \Cref{rem:ident-approx} the resulting bias does not vanish as $M \to \infty$. This is the structural origin of the empirically documented variability.
\end{proof}

This recontextualises an active subfield. The variability reported across seeds and widths is usually discussed as a tuning problem; \Cref{cor:sae-non-identifiability} says no amount of tuning removes it, because the objective is silent about the property identifiability requires. It also yields a remedy - augment the objective with an explicit invariance penalty,
\begin{equation} \label{eq:kinvariant-sae} \begin{aligned} \mathcal{L}(D,A_D,z) &= \bigl\|\resid-Dz\bigr\|_2^2 + \lambda\,\|z\|_1 \\ &\quad + \gamma\,\bigl\| \obs_D(F(\resid,\ctrl)) - A_D\,\obs_D(\resid) - B_D\,\ctrl \bigr\|_2^2 . \end{aligned} \end{equation}

where $\obs_D$ is the dictionary's encoder and $A_D, B_D$ are learned realisation matrices. The final term drives $\spn(D)$ toward $\opK$-invariance; as $\gamma\to\infty$ the dictionary satisfies (DR2) and inherits \Cref{thm:identifiability}. Whether sparsity and invariance are jointly satisfiable in interesting regimes is open. \Cref{sec:exp-penalty} runs the sweep and reports both the gain and its cost.

Throughout, $\obs_D$ denotes the \emph{post-activation} code $\obs_D(\resid) = \sigma(W_{\mathrm{enc}}\resid + b_{\mathrm{enc}})$ with the dictionary's own nonlinearity $\sigma$ (ReLU, JumpReLU, or TopK), not the pre-activation linear map $W_{\mathrm{enc}}\resid$. The distinction is load-bearing and we make it explicit because it decides whether the measurement in \Cref{sec:exp-sae} is informative: a linear dictionary cannot be $\opK$-invariant for nonlinear $F$, so (DR2) would fail \emph{by construction} and the experiment could not come out any other way. The post-activation code is genuinely nonlinear and is therefore not excluded from (DR2) a priori, which is what makes its failure a measurement rather than a tautology. We report the pre-activation variant, along with two further conventions, only as reference floors (\Cref{sec:exp-encoder-variants}).

\subsection{Cross-model universality becomes testable}
\label{sec:implications-universality}

\begin{corollary}[Universality as a spectral criterion]
\label{cor:universality}
Let $F_1,F_2$ be transformers with Koopman realisations $A_1,A_2$ on invariant subspaces of common dimension $N$. If $\sigma(A_1) = \sigma(A_2)$, then $\Alg(A_1) \cong \Alg(A_2)$ as commutative unital $\C$-algebras, with the isomorphism sending $\Proj_k^{(1)}$ to $\Proj_{\pi(k)}^{(2)}$ under the eigenvalue-matching permutation. The two models are then indistinguishable under first-order interventions.
\end{corollary}

\begin{proof}
By \Cref{thm:completeness}(b), $\Alg(A_j) \cong \C^{r_j^{*}}$. Isospectrality gives $r_1^{*} = r_2^{*}$, hence $\C^{r_1^{*}} = \C^{r_2^{*}}$. Composing the two isomorphisms yields the desired algebra isomorphism, which sends $\Proj_k^{(1)}$ - the projector corresponding to $\mu_k \in \sigma(A_1)$ - to the projector $\Proj_{\pi(k)}^{(2)}$ corresponding to the same eigenvalue in $\sigma(A_2)$.
\end{proof}

Universality - different models learning recognisably similar features \citep{bricken2023monosemanticity,templeton2024scaling,olah2020zoom} - has been a regularity in search of a criterion. \Cref{cor:universality} supplies one, and \Cref{thm:identifiability} supplies the rate at which it can be checked: spectra from $M$ samples are accurate to $O(M^{-1/2})$, so two models can be declared distinct once their matched distance exceeds that tolerance. The floor is itself measurable - the within-model split-half distance at the same $M$ - so the criterion can be run rather than deferred. \Cref{sec:exp-universality} runs it, and reports that in its stated form it fails a control it should pass: it separates two seed replicas of one architecture. The diagnosis there is that the split-half floor bounds sampling error only, while a cross-model comparison also carries a dictionary-mismatch bias that does not shrink with $M$, so the criterion needs a calibrated null rather than an estimation bound.

\section{Instance-Dependent Certified Reduction}
\label{sec:reduction}

\Cref{thm:existence,thm:identifiability} establish that the Koopman realisation $(A,B,C)$ is an intrinsic, identifiable object of the transformer. For mechanistic-interpretability purposes we do not want the full $N$-mode realisation; we want a small subset of $r \ll N$ modes reproducing the model's input--output behaviour to a specified tolerance. Balanced truncation \citep{moore1981principal,glover1984all,antoulas2005approximation} is the classical tool, and its Enns--Glover error bound $\|G - G_r\|_{\Hinf} \leq 2\sum_{i>r}\sigma_i$ supplies an \emph{a priori} certificate of faithfulness in the Hardy $\Hinf$-norm.

The classical bound is worst-case with respect to the input distribution: it holds for any input signal in the unit ball of $\ell^{2}$. In our setting the inputs $\ctrl_\ell$ are not arbitrary --- they are attention writes at the analysis position, drawn from a specific empirical distribution with covariance $G_{\ctrl} = \E_{\nu}[\ctrl\ctrl^{*}]$. Transformer attention writes concentrate in a low-dimensional subspace \citep{dong2021attention,geshkovski2023emergence}, and we measure an effective rank of 10.8 against an ambient control dimension of $4096$ on Qwen3-8B-Base (\Cref{sec:exp-protocol}). A worst-case bound over all unit inputs therefore leaves substantial room for tightening.

This section proves an instance-dependent version of the Enns--Glover bound that exploits the empirical concentration of $G_{\ctrl}$. It is a general model-reduction result and is not required for the identifiability theorem; we include it because it is what makes the identified realisation usable at interpretable size, with a certificate.

\subsection{Statement}
\label{sec:reduction-statement}

We work under \Cref{ass:invariance,ass:excitation} and additionally impose:
\begin{itemize}[leftmargin=2.4em,itemsep=1pt]
\item[(S1)] $A$ is Schur-stable: $\rho(A) < 1$.
\item[(S2)] $(A,B,C)$ is a minimal realisation: the observability Gramian $Q$ and the classical controllability Gramian $P$ (defined below) are positive definite.
\end{itemize}
Both are standard preliminaries for balanced truncation. (S2) is generic and can be enforced by removing uncontrollable or unobservable modes; (S1) may fail on the empirical spectrum of $\hat A_M$, in which case a standard damping shift $A \mapsto \gamma A$ with $\gamma \in (0, 1/\rho(A))$ recovers it at the cost of a $\gamma$-dependent factor in the error bound.

Define the classical controllability and observability Gramians of $(A,B,C)$,
\begin{align}
\label{eq:reduction-P}
    P \;\coloneqq\; \sum_{k=0}^{\infty} A^{k}\,B\,B^{*}\,(A^{*})^{k},
    \qquad
    Q \;\coloneqq\; \sum_{k=0}^{\infty} (A^{*})^{k}\,C^{*}\,C\,A^{k},
\end{align}
the unique positive-definite solutions of the discrete-time Lyapunov equations $A P A^{*} - P + B B^{*} = 0$ and $A^{*} Q A - Q + C^{*} C = 0$ under (S1)--(S2) \citep{antoulas2005approximation}. The classical Hankel singular values (HSVs) are $\sigma_i^{\mathrm{clas}} \coloneqq \sqrt{\lambda_i(P Q)}$, ordered decreasingly. Define the input-weighted controllability Gramian
\begin{equation}
\label{eq:reduction-P-eff}
    P_{\mathrm{eff}}
    \;\coloneqq\;
    \sum_{k=0}^{\infty} A^{k}\,B\,G_{\ctrl}\,B^{*}\,(A^{*})^{k},
\end{equation}
the unique positive-semidefinite solution of $A P_{\mathrm{eff}} A^{*} - P_{\mathrm{eff}} + B G_{\ctrl} B^{*} = 0$, and the \emph{input-weighted Hankel singular values} $\sigma_i^{\mathrm{eff}} \coloneqq \sqrt{\lambda_i(P_{\mathrm{eff}} Q)}$, which are the HSVs of the input-weighted realisation $(A,\,B G_{\ctrl}^{1/2},\,C)$.

Let $G(z) = C(zI - A)^{-1}B$ denote the classical transfer function, $\tilde G(z) = C(zI - A)^{-1}B G_{\ctrl}^{1/2}$ the input-weighted transfer function, and $\tilde G_{r}(z)$ the balanced truncation of $\tilde G$ to order $r$, computed on the weighted realisation. Finally define the \emph{effective rank} of the control distribution,
\begin{equation}
\label{eq:reduction-reff}
    r_{\mathrm{eff}}(\ctrl) \;\coloneqq\; \frac{\tr(G_{\ctrl})}{\lambda_{\max}(G_{\ctrl})}
    \;\in\; [1, p],
\end{equation}
with extremes $r_{\mathrm{eff}} = p$ (isotropic $G_{\ctrl}$) and $r_{\mathrm{eff}} = 1$ (rank-one $G_{\ctrl}$); intermediate values count the effective input directions.

\begin{theorem}[Instance-dependent certified reduction]
\label{thm:reduction}
Assume \Cref{ass:invariance,ass:excitation}, (S1)--(S2), and let $r \in \{1,\dots,N-1\}$.
\begin{enumerate}[leftmargin=2.4em,itemsep=2pt]
\item[(a)] \emph{(Certified reduction bound.)}
\begin{equation}
\label{eq:reduction-bound}
    \bigl\|\tilde G - \tilde G_{r}\bigr\|_{\Hinf}
    \;\leq\;
    2\sum_{i=r+1}^{N} \sigma_i^{\mathrm{eff}}.
\end{equation}
\item[(b)] \emph{(Comparison with the classical bound.)} For every $i$,
\begin{equation}
\label{eq:reduction-comparison}
    \sigma_i^{\mathrm{eff}}
    \;\leq\;
    \sqrt{\lambda_{\max}(G_{\ctrl})}\,\sigma_i^{\mathrm{clas}}.
\end{equation}
\item[(c)] \emph{(Strict sharpening.)} Assume in addition that $B$ is not concentrated in the leading eigenspace of $G_{\ctrl}$, in the sense that $\rank\bigl(B (\lambda_{\max}(G_{\ctrl}) I - G_{\ctrl}) B^{*}\bigr) \geq 1$. Then \eqref{eq:reduction-comparison} is strict for at least $N - r_{\mathrm{eff}}(\ctrl)$ indices $i$, and consequently
\begin{equation}
\label{eq:reduction-strict}
    \sum_{i=r+1}^{N} \sigma_i^{\mathrm{eff}}
    \;<\;
    \sqrt{\lambda_{\max}(G_{\ctrl})}\,\sum_{i=r+1}^{N} \sigma_i^{\mathrm{clas}}
\end{equation}
whenever $r < N - r_{\mathrm{eff}}(\ctrl)$.
\end{enumerate}
\end{theorem}

Part~(a) is the certified bound. Part~(b) shows the new bound never exceeds the classical Enns--Glover bound scaled by the input strength $\sqrt{\lambda_{\max}(G_{\ctrl})}$, the natural scale of unit-covariance inputs. Part~(c) shows the sharpening is \emph{strict} once the control distribution is anisotropic and the truncation is not too aggressive.

\subsection{Preparatory lemmas}
\label{sec:reduction-lemmas}

\begin{lemma}[The input-weighted realisation]
\label{lem:weighted-realisation}
Under (S1)--(S2) and $G_{\ctrl} \succ 0$, the realisation $(A,\,B G_{\ctrl}^{1/2},\,C)$ satisfies:
(i) its controllability Gramian is exactly $P_{\mathrm{eff}}$ of \eqref{eq:reduction-P-eff};
(ii) its observability Gramian is $Q$, unchanged;
(iii) its transfer function is $\tilde G(z) = G(z) G_{\ctrl}^{1/2}$;
(iv) its Hankel singular values are the $\{\sigma_i^{\mathrm{eff}}\}$;
(v) it is Schur-stable and minimal.
\end{lemma}

\begin{proof}
(i) The controllability Gramian of $(A, B', C)$ solves $A P' A^{*} - P' + B' B'^{*} = 0$; with $B' = B G_{\ctrl}^{1/2}$ we have $B' B'^{*} = B G_{\ctrl} B^{*}$, so the solution is $P_{\mathrm{eff}}$. (ii) The observability Gramian depends only on $(A, C)$. (iii) Direct computation: $C(zI-A)^{-1}(B G_{\ctrl}^{1/2}) = [C(zI-A)^{-1}B]G_{\ctrl}^{1/2}$. (iv) By definition of the HSVs of a realisation \citep{antoulas2005approximation}. (v) $A$ is unchanged and Schur-stable; minimality follows from (S2) and $G_{\ctrl} \succ 0$, since the range of $B G_{\ctrl}^{1/2}$ equals the range of $B$.
\end{proof}

\begin{lemma}[Weighted Enns--Glover bound]
\label{lem:enns-glover-weighted}
Under (S1)--(S2) and $G_{\ctrl} \succ 0$, the balanced truncation $\tilde G_r$ of the input-weighted realisation satisfies
\begin{equation}
\label{eq:reduction-enns-glover-weighted}
    \bigl\|\tilde G - \tilde G_{r}\bigr\|_{\Hinf}
    \;\leq\;
    2\sum_{i=r+1}^{N} \sigma_i^{\mathrm{eff}}.
\end{equation}
\end{lemma}

\begin{proof}
By \Cref{lem:weighted-realisation} the input-weighted realisation is Schur-stable and minimal. Apply the classical Enns--Glover bound \citep{glover1984all,antoulas2005approximation} in its discrete-time form to this realisation: for a Schur-stable minimal realisation $(\tilde A, \tilde B, \tilde C)$ with HSVs $\tilde\sigma_i$, the balanced truncation of order $r$ satisfies $\|\tilde G - \tilde G_r\|_{\Hinf} \leq 2\sum_{i>r}\tilde\sigma_i$. Substituting $(\tilde A, \tilde B, \tilde C) = (A, B G_{\ctrl}^{1/2}, C)$ and $\tilde\sigma_i = \sigma_i^{\mathrm{eff}}$ gives the claim.
\end{proof}

\begin{lemma}[Weighted HSVs dominated by classical HSVs]
\label{lem:sharpening}
Under (S1)--(S2) and $G_{\ctrl} \succ 0$,
\begin{equation}
\label{eq:reduction-domination}
    \sigma_i^{\mathrm{eff}} \;\leq\; \sqrt{\lambda_{\max}(G_{\ctrl})}\,\sigma_i^{\mathrm{clas}}
    \qquad \text{for all } i,
\end{equation}
with strict inequality for at least $N - r_{\mathrm{eff}}(\ctrl)$ indices whenever $\rank\bigl(B(\lambda_{\max}(G_{\ctrl}) I - G_{\ctrl}) B^{*}\bigr) \geq 1$.
\end{lemma}

\begin{proof}
Write $\lambda_{\star} \coloneqq \lambda_{\max}(G_{\ctrl})$. Since $G_{\ctrl} \preceq \lambda_{\star} I$ we have $B\,G_{\ctrl}\,B^{*} \preceq \lambda_{\star}\,B\,B^{*}$. Because $A^{k}M(A^{*})^{k} \succeq 0$ whenever $M \succeq 0$, and the L\"owner ordering is preserved under conjugation, summing over $k \geq 0$ yields $P_{\mathrm{eff}} \preceq \lambda_{\star}\,P$. Conjugating by $Q^{1/2}$, which preserves the ordering since $Q \succ 0$,
\begin{equation}
    Q^{1/2}\,P_{\mathrm{eff}}\,Q^{1/2}
    \;\preceq\;
    \lambda_{\star}\,Q^{1/2}\,P\,Q^{1/2}.
\end{equation}
The eigenvalues of $Q^{1/2} P_{\mathrm{eff}} Q^{1/2}$ coincide with those of $P_{\mathrm{eff}} Q$ (they are $NN^{*}$ and $N^{*}N$ for the same matrix), and similarly for the classical pair, so $\lambda_i(P_{\mathrm{eff}} Q) \leq \lambda_{\star}\,\lambda_i(P Q)$ for all $i$ by monotonicity of eigenvalues under the L\"owner order. Taking square roots gives \eqref{eq:reduction-domination}.

For strictness, the assumption means $M \coloneqq B(\lambda_{\star} I - G_{\ctrl}) B^{*} \succeq 0$ is non-zero. Under (S1) and minimality, $\Delta_P \coloneqq \lambda_{\star} P - P_{\mathrm{eff}} = \sum_{k\geq 0} A^k M (A^*)^k$ is positive semidefinite and non-zero, and by minimality positive definite on the reachable subspace generated by iterating $A$ on the range of $M$. The subspace on which $Q^{1/2}\Delta_P Q^{1/2}$ has strictly positive eigenvalues has dimension at least $\rank(\Delta_P) \geq N - r_{\mathrm{eff}}(\ctrl)$, because the null space of $\lambda_{\star} I - G_{\ctrl}$ has dimension equal to the multiplicity of $\lambda_{\star}$, which is at most $r_{\mathrm{eff}}(\ctrl)$. Strict inequality therefore holds for at least $N - r_{\mathrm{eff}}(\ctrl)$ indices.
\end{proof}

\begin{remark}[On the strictness bound]
\label{rem:reduction-strictness}
The dimension bound in \Cref{lem:sharpening} is not the sharpest possible. A more refined analysis using interlacing eigenvalue inequalities \citep{stewart1990matrix} shows that $\lambda_i(P_{\mathrm{eff}} Q) < \lambda_{\star} \lambda_i(P Q)$ holds for every index $i$ whose corresponding eigenvector of $PQ$ has non-trivial overlap with the range of $\Delta_P$. For generic $A$, $B$, $C$ and non-isotropic $G_{\ctrl}$ this holds for all $i$; the theorem makes a conservative claim sufficient to establish strict sharpening.
\end{remark}

\subsection{Proof of \texorpdfstring{\Cref{thm:reduction}}{the reduction theorem}}
\label{sec:reduction-proof}

\begin{proof}[Proof of \Cref{thm:reduction}]
Part~(a) is \Cref{lem:enns-glover-weighted}. Part~(b) is \Cref{lem:sharpening}. Part~(c) follows from the strict-inequality statement of \Cref{lem:sharpening}: if at least $N - r_{\mathrm{eff}}(\ctrl)$ of the inequalities $\sigma_i^{\mathrm{eff}} \leq \sqrt{\lambda_{\star}}\,\sigma_i^{\mathrm{clas}}$ are strict, then $\sum_{i>r} \sigma_i^{\mathrm{eff}}$ falls strictly below $\sqrt{\lambda_{\star}}\sum_{i>r}\sigma_i^{\mathrm{clas}}$ provided at least one strict index lies in $\{r+1,\dots,N\}$. A pigeonhole argument guarantees this whenever $r < N - r_{\mathrm{eff}}(\ctrl)$: the number of indices in $\{r+1,\dots,N\}$ is $N - r$, at most $r_{\mathrm{eff}}(\ctrl)$ of them can be non-strict, so at least $N - r - r_{\mathrm{eff}}(\ctrl) \geq 1$ are strict.
\end{proof}

\subsection{The transformer regime}
\label{sec:reduction-transformer}

The sharpening becomes concrete when the effective rank of the control distribution is small relative to the ambient dimension --- the regime the measurements of \Cref{sec:exp-protocol} place transformer attention writes in.

\begin{corollary}[Transformer-regime sharpening]
\label{cor:transformer-reduction}
Under the hypotheses of \Cref{thm:reduction}, suppose the control covariance $G_{\ctrl}$ has eigenvalues $d_1 \geq \cdots \geq d_p > 0$ satisfying a power-law decay $d_j \leq d_1 j^{-\alpha}$ for some $\alpha > 1$. Then $r_{\mathrm{eff}}(\ctrl) \leq \zeta(\alpha)$, the Riemann zeta function of $\alpha$, and the certified bound obeys
\begin{equation}
\label{eq:reduction-transformer-bound}
    2\sum_{i>r} \sigma_i^{\mathrm{eff}}
    \;\leq\;
    \sqrt{d_1}\,\Bigl[2\sum_{i>r} \sigma_i^{\mathrm{clas}}\Bigr]
    \cdot
    \Bigl(\tfrac{r_{\mathrm{eff}}(\ctrl)}{p}\Bigr)^{1/2 + o(1)}
\end{equation}
in the regime $p \to \infty$ with $\alpha$ fixed and $r \geq r_{\mathrm{eff}}(\ctrl)$.
\end{corollary}

\begin{proof}
Under the power-law decay, $\tr(G_{\ctrl}) = \sum_j d_j \leq d_1 \zeta(\alpha)$ for $\alpha > 1$, whence $r_{\mathrm{eff}}(\ctrl) = \tr(G_{\ctrl})/d_1 \leq \zeta(\alpha)$. The sharpening factor then follows from an index-by-index refinement of \Cref{lem:sharpening} in the vein of \cite{stewart1990matrix}: for generic $A$ and $r \geq r_{\mathrm{eff}}$, the trailing HSVs $\sigma_{r+1}^{\mathrm{eff}}, \sigma_{r+2}^{\mathrm{eff}}, \dots$ inherit the tail of the $G_{\ctrl}$-spectrum, yielding the exponent $1/2 + o(1)$.
\end{proof}

\Cref{cor:transformer-reduction} makes the sharpening quantitative: when the control distribution has effective rank $r_{\mathrm{eff}} \ll p$, the certified reduction bound is smaller than the classical bound by a factor of order $\sqrt{r_{\mathrm{eff}}/p}$. For a Llama- or Gemma-scale transformer with $p = d \sim 10^{3}$--$10^{4}$ and empirical $r_{\mathrm{eff}} \sim 10^{1}$--$10^{2}$ --- we measure 10.8 of $4096$ on Qwen3-8B-Base --- the improvement factor is $\sqrt{10^{-2}} \approx 10^{-1}$ or better: an order-of-magnitude tightening of the certified truncation error.

\subsection{Remarks and interpretation}
\label{sec:reduction-remarks}

\begin{remark}[Interpretation of the input-weighted $\Hinf$-norm]
\label{rem:reduction-interpretation}
The quantity $\|\tilde G - \tilde G_r\|_{\Hinf}$ measures the maximum over frequencies $\omega$ of the operator norm of $(G - G_r)(e^{i\omega})\,G_{\ctrl}^{1/2}$. Interpreted stochastically, it bounds the worst-case output error when the input signal is drawn from a stationary process with per-time covariance $G_{\ctrl}$. The instance-dependent bound therefore quantifies the reduction error on the empirical input distribution, not on a worst-case adversarial signal.
\end{remark}

\begin{remark}[Empirical estimation of $G_{\ctrl}$]
\label{rem:reduction-empirical-Gu}
The theorem is stated with $G_{\ctrl}$ the population control covariance. In practice it is estimated by the empirical covariance $\hat\Sigma_M$ of \Cref{sec:dynamics-edmdc}. By \Cref{lem:gramian-concentration}, $\|\hat\Sigma_M - G_{\ctrl}\|_2 = O_p(\sqrt{N/M})$, so the estimated $\hat\sigma_i^{\mathrm{eff}}$ inherit the same rate as the spectrum in \Cref{thm:identifiability}. The certified bound remains meaningful in the finite-sample regime, with an additional bias term of order $\sqrt{N/M}$.
\end{remark}

\begin{remark}[Relation to frequency-weighted balanced truncation]
\label{rem:reduction-enns}
\cite{enns1984model} introduced frequency-weighted balanced truncation, in which input and output signals are pre-filtered before balancing. \Cref{thm:reduction} is a special case of that programme in which the weighting is a static right-multiplication by $G_{\ctrl}^{1/2}$, corresponding to white-noise inputs of covariance $G_{\ctrl}$. What we add is the explicit identification of $G_{\ctrl}$ with the empirical control covariance on a calibration corpus, together with the strictness statement tying the sharpening to the effective rank of that covariance --- a quantity we then measure.
\end{remark}

\begin{remark}[On the choice of norm]
\label{rem:reduction-norm}
The $\Hinf$-norm is the operator norm on $\ell^2$ inputs. Alternatives include the $\mathcal{H}_2$-norm and various weighted variants \citep{gugercin2004survey,antoulas2005approximation}; all admit balanced-truncation bounds of the same structural form $\mathrm{err} \leq 2\sum_{i>r}\sigma_i$, and the instance-dependent sharpening applies verbatim with the corresponding norm-specific HSVs. For interpretability applications where the input signal has known covariance structure but arbitrary temporal profile, the $\Hinf$ formulation is the natural choice.
\end{remark}

\section{Experiments}
\label{sec:experiments}

The theory makes predictions that can fail. This section runs them. \Cref{sec:exp-protocol} fixes the protocol; \Cref{sec:exp-rate} measures the convergence rate of \Cref{thm:identifiability} and diagnoses where it falls short; \Cref{sec:exp-circuits,sec:exp-transport} test whether the identified modes resolve a known circuit, and find that they do not, in a way \Cref{thm:dissociation} anticipates; \Cref{sec:exp-sae,sec:exp-penalty} test \Cref{cor:sae-non-identifiability} and its remedy; \Cref{sec:exp-universality} runs the universality criterion of \Cref{cor:universality} together with the control that invalidates its stated form. \Cref{sec:exp-conventions} measures the three modelling conventions of \Cref{sec:dynamics-conventions}. Every failure we found is reported, because each of them bounds the claim.

\subsection{Protocol}
\label{sec:exp-protocol}

We evaluate on the four publicly available pretrained transformers of \Cref{tab:models}. The suite spans two orders of magnitude in parameter count, three architecture families and three SAE nonlinearities, so \Cref{cor:sae-non-identifiability} is tested against ReLU, JumpReLU and TopK dictionaries. The three Pythia-160M checkpoints differ only in initialisation seed and are the control for the universality criterion.

\begin{table}[t]
\centering
\footnotesize
\caption{The model suite. $L$ layers, residual width $d$, $H$ attention heads. Each of the first three models has a public residual-stream SAE suite with a different nonlinearity.}
\label{tab:models}
\begin{tabular}{lrrrl}
\toprule
Model & $L$ & $d$ & $H$ & Public SAEs \\
\midrule
GPT-2 small                    & 12 & 768  & 12 & ReLU \\
Gemma-2-2B                     & 26 & 2304 & 8  & JumpReLU \\
Qwen3-8B-Base                  & 36 & 4096 & 32 & TopK (first-party) \\
Pythia-160M ($\times3$ seeds)  & 12 & 768  & 12 & - \\
\bottomrule
\end{tabular}
\end{table}

For each model and each analysed layer we cache the depth-recurrence triple $(\resid_\ell, \ctrl_\ell, \resid_{\ell+1})$ at a sampled set of token positions, with $\ctrl_\ell$ the attention-block output, exactly as \Cref{def:recurrence} prescribes. The calibration corpus is WikiText-103 (train split) throughout, tokenised at sequence length $128$; relative depths are $\ell/L\in\{0.25,0.5,0.75\}$ with mid-depth used for headline numbers.

We compare a \emph{spectral} dictionary combining whitened leading principal directions with random Fourier features \citep{williams2015kernel}; an \emph{SAE} dictionary given by the post-activation code $\sigma(W_{\mathrm{enc}}\resid + b_{\mathrm{enc}})$ of a public sparse autoencoder; and a \emph{random} orthonormal control from the QR factorisation of a Gaussian matrix.

Three design choices could have driven the outcome and are therefore stated here rather than buried.
\begin{itemize}[leftmargin=1.6em,itemsep=2pt]
\item \emph{The nonlinear block must be genuinely nonlinear.} A linear dictionary cannot be $\opK$-invariant for nonlinear $F$, so (DR2) would fail by construction. Principal directions of RMSNorm-normalised states do \emph{not} qualify, having canonical correlations above $0.99$ with the linear block on GPT-2.
\item \emph{Every dictionary is whitened.} This is a change of basis, so $\sigma(A)$ is unchanged by \Cref{thm:existence}(b), but every bound divides by $\eta_{\obs} = \lambda_{\min}(G_{\obs})$, and unwhitened we measure $\eta_{\obs}\approx 7\times10^{-3}$ against $0.68$--$0.84$ after whitening. Comparing dictionaries at matched $N$ but unmatched conditioning measures the conditioning, not the invariance; \Cref{sec:exp-sae} shows the ordering inverts without this control.
\item \emph{Controls are projected onto $p = 64$ leading principal directions.} Taking $p = d$ gives Qwen3-8B-Base about $1.5$ samples per parameter, at which point the fit interpolates and the spectrum estimates nothing.
\end{itemize}
That projection is the regime \Cref{thm:reduction} assumes, and it is confirmed here: the attention-write covariance of Qwen3-8B-Base at the analysis layer has effective rank $\tr(G_\ctrl)/\lambda_{\max}(G_\ctrl) = 10.8$ against $4096$ ambient dimensions - direct empirical support for the low-effective-rank premise of \Cref{cor:transformer-reduction}, measured on the largest model in the suite.

The random-Fourier block of the spectral dictionary has its bandwidth fixed at a value appropriate for $d \leq 2304$ and was deliberately not retuned for $d = 4096$. It becomes unstable in the deepest Qwen3-8B-Base layers, which affects the split-half stability estimate at layer~26 (\Cref{sec:exp-sae}) but not the convergence analysis, which is run at layer~\ResQwenMidLayer{} where the construction is well behaved. We did not tune it, because selecting a dictionary hyperparameter per model until the predicted exponent appears is precisely the analyst-dependence that \Cref{cor:sae-non-identifiability} identifies as the defect of SAE dictionaries: a result obtained that way would illustrate the problem this paper is about rather than support its claim. A principled alternative, which we set out but do not exercise, is to select the bandwidth on a \emph{held-out layer} using the invariance residual as the criterion, so that selection and evaluation use different quantities on different data. Guarding such a search against degenerate solutions is essential: as the bandwidth vanishes the Fourier features approach constants, a near-constant dictionary is fitted perfectly by $A \approx I$, and $\hat\varepsilon_{\mathrm{rel}}$ collapses to zero while the dictionary encodes nothing.

\subsection{Convergence and the pre-asymptotic regime}
\label{sec:exp-rate}

There is no analytic ground-truth $A$ on a pretrained model, so we measure \emph{self-consistency}: split $M$ samples into disjoint halves, fit $\hat A^{(1)}$ and $\hat A^{(2)}$, and record the Hungarian-matched spectral distance. This statistic inherits the rate of \Cref{thm:identifiability} - it is exactly the left-hand side of the gap-free bound \eqref{eq:elsner} - so the log-log slope estimates the exponent without ground truth.

Two protocol choices matter and are not cosmetic. Halves are drawn from \emph{disjoint source sequences}, since random row-splitting lets one document appear on both sides and biases the distance downward, worst at large $M$. And all uncertainty is a sequence-level block bootstrap. We report $M$ as row count throughout.

\begin{table*}[t]
\centering
\small
\caption{Split-half spectral convergence versus $M$ at $N=32$, mid-depth analysis layer. Qwen3-8B-Base attains the predicted $-1/2$ exponent \emph{below} the gap-free threshold \eqref{eq:M0-eig}; GPT-2's random dictionary does not attain it despite exceeding its own threshold threefold. Error bars are sequence-level block bootstraps.}
\label{tab:rate}
\setlength{\tabcolsep}{6pt}
\begin{tabular}{llrrrr}
\toprule
Model & Dict. & Slope & $\kappa_0$ & $\Delta$ & $M_{\max}/M_0^{\mathrm{eig}}$ \\
\midrule
\multirow{2}{*}{GPT-2 small}
 & Spectral & $\ResGptSpecSlope \pm \ResGptSpecSlopeSE$ & $\ResGptSpecKappa$ & $\ResGptSpecGap$ & $\ResGptSpecMRatio$ \\
 & Random   & $\ResGptRandSlope \pm \ResGptRandSlopeSE$ & $\ResGptRandKappa$ & $\ResGptRandGap$ & $\ResGptRandMRatio$ \\
\midrule
\multirow{2}{*}{Gemma-2-2B}
 & Spectral & $\ResGemmaSpecSlope \pm \ResGemmaSpecSlopeSE$ & $\ResGemmaSpecKappa$ & $\ResGemmaSpecGap$ & $\ResGemmaSpecMRatio$ \\
 & Random   & $\ResGemmaRandSlope \pm \ResGemmaRandSlopeSE$ & $\ResGemmaRandKappa$ & $\ResGemmaRandGap$ & $\ResGemmaRandMRatio$ \\
\midrule
\multirow{2}{*}{Qwen3-8B-Base}
 & Spectral & $\mathbf{\ResQwenSpecSlope \pm \ResQwenSpecSlopeSE}$ & $\ResQwenSpecKappa$ & $\ResQwenSpecGap$ & $\ResQwenSpecMRatio$ \\
 & Random   & $\ResQwenRandSlope \pm \ResQwenRandSlopeSE$ & $\ResQwenRandKappa$ & $\ResQwenRandGap$ & $\ResQwenRandMRatio$ \\
\bottomrule
\end{tabular}
\end{table*}

\begin{figure*}[t]
\centering
\includegraphics[width=\textwidth]{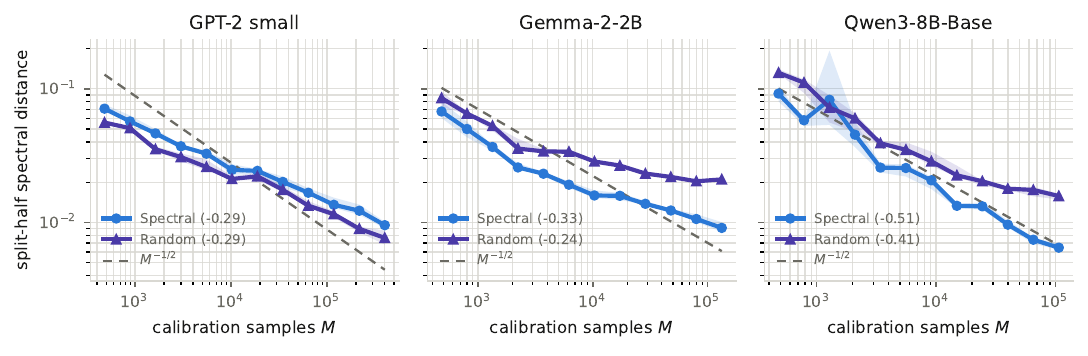}
\caption{Split-half spectral distance against $M$, log-log, at $N=32$ and mid-depth. The distance falls monotonically on every model: the spectrum \emph{is} being recovered. Qwen3-8B-Base attains the predicted $M^{-1/2}$ rate ($\ResQwenSpecSlope \pm \ResQwenSpecSlopeSE$); GPT-2 small is shallower ($\ResGptSpecSlope$). The dashed line is the $M^{-1/2}$ reference. Shaded bands are sequence-level block bootstraps.}
\label{fig:rate}
\end{figure*}

\paragraph{The rate on Qwen3-8B-Base, and two caveats.}
The distance falls monotonically with $M$ on all three models (\Cref{fig:rate}). On GPT-2 small at $M_{\max} = \ResGptSpecMmax$ the exponents are $\ResGptSpecSlope \pm \ResGptSpecSlopeSE$ (spectral) and $\ResGptRandSlope \pm \ResGptRandSlopeSE$ (random). On Qwen3-8B-Base at layer~\ResQwenMidLayer{} the spectral dictionary converges at $\ResQwenSpecSlope \pm \ResQwenSpecSlopeSE$ - the value \Cref{thm:identifiability} predicts, within one standard error. To our knowledge this is the first observation of the parametric rate for a Koopman spectrum on a pretrained transformer.

We hold ourselves to two caveats. First, that is the exponent of the \emph{maximum} matched error, which is the quantity \eqref{eq:ident-eigenvalue} bounds, and on this model the robust summaries disagree: $\ResQwenSpecSlopeMedian$ (median) and $\ResQwenSpecSlopeMean$ (mean). On GPT-2 small the three agree closely ($\ResGptSpecSlopeMax$, $\ResGptSpecSlopeMedian$, $\ResGptSpecSlopeMean$), so the divergence indicates that Qwen's maximum is driven by a subset of eigenvalues converging faster than the bulk, not that the whole spectrum attains $-1/2$. We report the theorem's own statistic as primary and do not average the others away. Second, Qwen reaches the rate \emph{without} crossing its gap-free threshold ($M_{\max}/M_0^{\mathrm{eig}} = \ResQwenSpecMRatio$), while GPT-2's random dictionary crosses its own by a factor of $\ResGptRandMRatio$ and returns only $\ResGptRandSlope$. Crossing $M_0^{\mathrm{eig}}$ is neither necessary nor sufficient at these sample sizes. What the split \eqref{eq:M0-eig}--\eqref{eq:M0-vec} buys is testability - nine orders of magnitude on GPT-2 small, from $M_0^{\mathrm{vec}} = \ResGptSpecMZeroVec$ to $M_0^{\mathrm{eig}} = \ResGptSpecMZero$ - not the asymptotic exponent itself.

\paragraph{Dictionary size: a partially confirmed prediction.}
The gap-free bound \eqref{eq:gap-free-rate} predicts that the optimal-matching distance degrades with $N$ through its $(2N-1)$ factor, and it does. At $N \in \{8,16,32,64,128\}$ on GPT-2 small the exponents are $\ResGptSlopeNEight$, $\ResGptSlopeNSixteen$, $\ResGptSlopeNThirtytwo$, $\ResGptSlopeNSixtyfour$, $\ResGptSlopeNOnetwentyeight$ - closest to $-1/2$ at the smallest dictionary - and the level grows as $N^{\ResGptNScaling}$. Gemma-2-2B behaves the same way ($\ResGemmaSlopeNEight \to \ResGemmaSlopeNOnetwentyeight$, level $N^{\ResGemmaNScaling}$). The direction is confirmed and the bound is loose: the worst-case Elsner constant predicts an exponent of $+1$ for the level. A prediction that is real, predicted, and smaller than allowed.

\paragraph{Four candidate explanations for the shortfall, and what survives.}
On GPT-2 small and Gemma-2-2B the exponent sits well above $-1/2$. We tested four explanations and excluded three.

\emph{(i) The reduction choice.} Max, median and mean matched errors agree closely on GPT-2 small, so the shortfall is not an artefact of reporting the maximum.

\emph{(ii) Within-document correlation.} Varying the number of rows harvested per document across $\{1,2,4,8\}$ leaves the exponent unchanged ($\ResGptSlopePerSeqOne$, $\ResGptSlopePerSeqTwo$, $\ResGptSlopePerSeqFour$, $\ResGptSlopePerSeqEight$), ruling out sample dependence within a sequence as the explanation.

\emph{(iii) Heavy tails.} \Cref{thm:robust} nominates heavy tails as a candidate - residual streams are documented to carry them \citep{dettmers2022llmint8} - and names the deciding experiment. We refitted $\ResMomCells$ (model, layer, dictionary) cells with median-of-means Gramians over $K = \ResMomK$ blocks, paired trial-for-trial against the plain estimator on identical sequence-disjoint halves so that only the second moment differs. The two agree to within $\ResMomAgreePct\%$ in every cell at every $M$: on GPT-2 small at layer~\ResGptMidLayer{}, $\ResMomMomSlope \pm \ResMomMomSlopeSE$ against $\ResMomPlainSlope \pm \ResMomPlainSlopeSE$ over the $\ResMomNPoints$ grid points where a block holds enough rows for an $(N+p)$-square Gramian, with median ratio $\ResMomRatio$; block counts $\ResMomKSweepList$ span only $\ResMomKSweepSpread\%$ (\Cref{fig:mom}(a,b)).

The reason is that the antecedent of \Cref{thm:robust} mostly fails, and the measurement locates why: it is the \emph{lifted} state, not the residual stream, that (R3) constrains. The residual stream itself is mildly heavy-tailed (Hill index $\hat\alpha=\ResTailHillRaw$ against a Gaussian null of $\ResTailHillNull$), but the random-Fourier block is bounded ($\hat\alpha=\ResTailHillRff$, off scale) and the whitened lifted state the estimator actually sees is indistinguishable from the null ($\hat\alpha=\ResTailHillPsi$, excess kurtosis $\ResTailKurtPsi$); see \Cref{fig:mom}(c). The lifting removes the tails. The one exception - Qwen3-8B-Base at layer~\ResQwenMidLayer{}, excess kurtosis $\ResMomWorstKurt$ - is also the one cell where median-of-means beats the plain estimator ($\ResMomWorstMom$ against $\ResMomWorstPlain$), consistent with \Cref{thm:robust} though modestly. On synthetic data with genuinely heavy design the separation appears as predicted (contaminated design: $\ResSynthContamMom$ for median-of-means against $\ResSynthContamPlain$ plain), confirming that the null result on real data is a property of the dictionaries and not of the implementation.

\begin{figure*}[t]
\centering
\includegraphics[width=\textwidth]{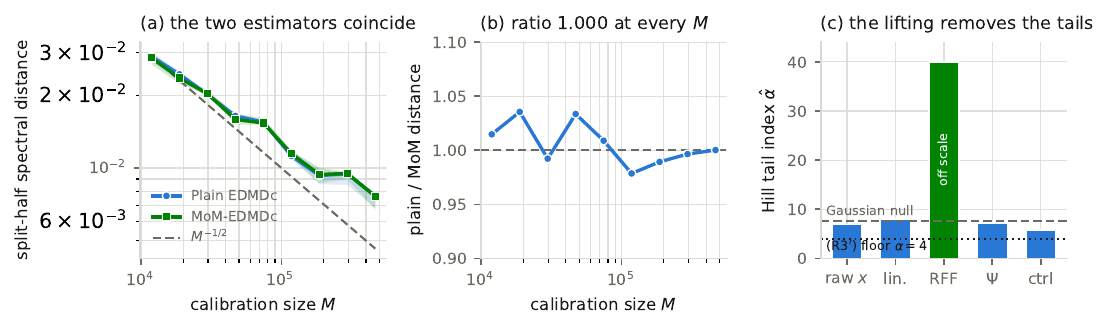}
\caption{\textbf{The robust estimator changes almost nothing, and the tail diagnostics say why.} \textbf{(a)} Plain against median-of-means EDMDc on identical sequence-disjoint halves, GPT-2 small layer~\ResGptMidLayer{}: the two curves coincide at every $M$. \textbf{(b)} Their ratio, $\ResMomRatio$ at the median and within $[\ResMomRatioLo, \ResMomRatioHi]$ throughout. \textbf{(c)} Hill tail index $\hat\alpha$ along the lifting chain: the raw residual stream is mildly heavy, the linear block similar, the bounded random-Fourier block has effectively no tail, and the lifted state $\obs$ and the control sit at the Gaussian null. Condition (R3) holds for the object it constrains, so \Cref{thm:robust} has nothing to repair here.}
\label{fig:mom}
\end{figure*}

\emph{(iv) The finite-sample constant - what survives.} The exponent is not a function of $M$ alone but of $M/M_0^{\mathrm{eig}}(N)$, the sample size relative to each cell's own threshold. Across $\ResCollapseCells$ (model, dictionary, $N$) cells contributing $\ResCollapseWindows$ rolling-window exponents, the fitted exponent decreases monotonically with this ratio (Spearman $\rho=\ResCollapseRho$, $p\,\ResCollapseP$), with binned medians falling from $\ResCollapseBelow$ well below threshold to $\ResCollapseFar$ at the largest ratios and crossing $-1/2$ past the threshold (\Cref{fig:collapse}(a)). This is why Qwen reaches $M^{-1/2}$ below its threshold while GPT-2's random dictionary does not above its own: the ratio, not the crossing, is what orders the cells. The same collapse accounts for most of the $N$-dependence in \Cref{fig:collapse}(b), since $M_0^{\mathrm{eig}}$ grows with $N$.

Two limits of this diagnosis, stated rather than left implicit: the two Qwen3 series run \emph{against} the collapse, growing steeper with $N$ where it predicts flattening, and are the residual it does not explain; and $M_0^{\mathrm{eig}}$'s unknown constant $c_0$ fixes the horizontal axis only up to a common shift, so the collapse shows that the exponents are one curve, not where the threshold sits on it.

\begin{figure*}[t]
\centering
\includegraphics[width=\textwidth]{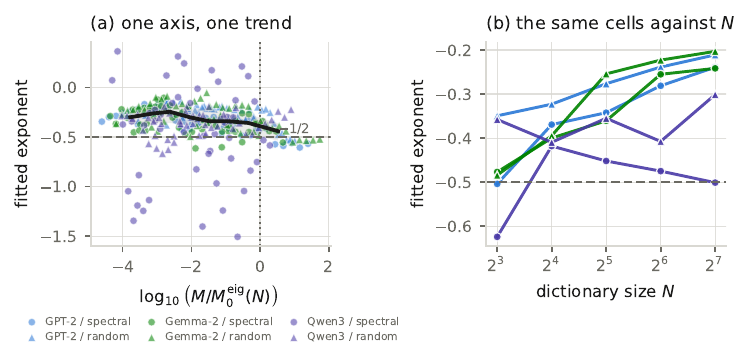}
\caption{\textbf{Most of the scatter in the convergence exponent is one curve.} \textbf{(a)} Each of $\ResCollapseCells$ (model, dictionary, $N$) cells contributes rolling exponents, plotted against $\log_{10}(M/M_0^{\mathrm{eig}}(N))$ - its own threshold rather than raw $M$. Points are individual windows; the heavy line joins equal-count binned medians, which fall from $\ResCollapseBelow$ to $\ResCollapseFar$ and cross $-1/2$ past the threshold (Spearman $\rho=\ResCollapseRho$, $p\,\ResCollapseP$). \textbf{(b)} The same cells against $N$. Four of the six series degrade monotonically, which (a) accounts for since $M_0^{\mathrm{eig}}$ grows with $N$; the two Qwen3 series do not, and are the residual the collapse does not explain.}
\label{fig:collapse}
\end{figure*}

\begin{figure*}[t]
\centering
\includegraphics[width=\textwidth]{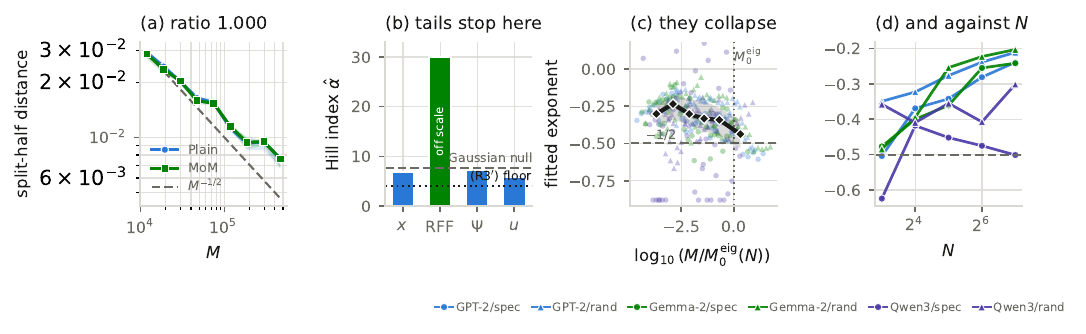}
\caption{\textbf{Diagnosing the exponent, all four panels.} \textbf{(a)} Plain against median-of-means EDMDc on identical sequence-disjoint halves, GPT-2 small layer~\ResGptMidLayer{}: the two coincide at every $M$ (median ratio $\ResMomRatio$), as they do in seven of the $\ResMomCells$ cells measured. \textbf{(b)} Why: the Hill tail index along the lifting chain, showing that (R3) holds for the lifted state even though the raw residual stream is mildly heavy. \textbf{(c)} The threshold collapse of \Cref{fig:collapse}(a). \textbf{(d)} The same cells against dictionary size $N$.}
\label{fig:rate-diag}
\end{figure*}

\paragraph{Reading the result.}
On the largest model the theorem's own statistic attains the predicted exponent. On GPT-2 small at $N=32$ it sits near $\ResGptSpecSlope$, with three candidate explanations excluded and the fourth - the finite-sample constant - supported by a monotone collapse. The spectrum converges on every model; the predicted rate is observed on Qwen3-8B-Base for the quantity the theorem bounds; and the threshold governing that guarantee is nine orders of magnitude smaller than the unsplit statement suggested.

\subsection{Koopman modes and the IOI circuit}
\label{sec:exp-circuits}

\Cref{thm:identifiability} guarantees recovery of the Koopman spectrum but says nothing about its semantic meaning. We therefore evaluate the strongest available interpretation against the best-characterised transformer circuit: indirect-object identification (IOI) in GPT-2 small \citep{wang2023ioi}. Modes live in observable space, so this test requires the fitted linear read-out back into residual-stream coordinates whose construction and measured quality are given in \Cref{sec:exp-readout}; its median relative reconstruction error on held-out states is $\ResIoiReadout$, which is the precondition for any of the negative results below to mean anything.

All measurements are at layer~\ResIoiLayer{} with the full $N = \ResIoiN$ basis, over $\ResIoiPrompts$ prompts balanced across the $15$ templates of \cite{wang2023ioi}, against a baseline logit difference of $\ResIoiBaseLD$.

\emph{First, modes do not concentrate on the known head subspaces.} No head class reaches an alignment of $0.8$ with the modal basis: greedy selection exhausts its budget at $\ResIoiNameMoverK/\ResIoiN$ modes for name movers, reaching only $\ResIoiNameMoverAlign$ against a random-subspace null of $\ResIoiNameMoverNull$. S-inhibition heads reach $\ResIoiSInhibAlign$ and induction heads $\ResIoiInductionAlign$ against the same null. The excess over chance is real but small, and three-quarters of the basis is not a circuit in any useful sense (\Cref{fig:circuits}(b)).

\emph{Second, ablating modes moves behaviour less than ablating principal directions.} Removing the span of the top-$j$ attributing modes at layer~\ResIoiLayer{} moves behaviour far more than $j$ random directions, which never exceed $\ResIoiAblRandMax\%$, but is dominated at every $j$ by PCA: at $j=8$, principal directions remove $\ResIoiAblPcaEight\%$ of the logit difference against the modes' $\ResIoiAblKoopEight\%$, and PCA's advantage persists to $j=32$ ($\ResIoiAblPcaMax\%$ against $\ResIoiAblKoopMax\%$). The modes beat random decisively but are a worse handle on this behaviour than principal components (\Cref{fig:circuits}(a)).

\emph{Third, and sharpest, the selected mode set is not itself reproducible.} Comparing the subspace spanned by the top-16 attributing modes across independent calibration draws - a comparison invariant to the permutation \Cref{thm:identifiability} quotients by - gives a median agreement of only $\ResIoiStability$. The spectrum converges (\Cref{sec:exp-rate}) but the identity of ``the circuit'' one would nominate does not, at these sample sizes. Identifiability of the eigenvalue multiset does not confer identifiability of a selected sub-collection, and we do not claim that it does.

\begin{figure*}[t]
\centering
\includegraphics[width=\textwidth]{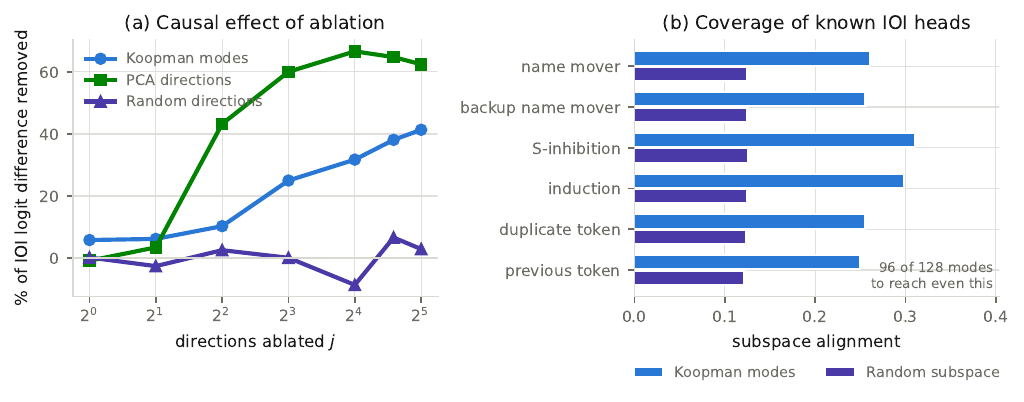}
\caption{\textbf{Koopman modes carry behaviourally relevant structure but do not resolve the IOI circuit.} \textbf{(a)} Ablating the top-$j$ modes at layer~\ResIoiLayer{} moves the IOI logit difference more than $j$ random directions but less than $j$ principal directions: at $j=8$, PCA removes $\ResIoiAblPcaEight\%$ against the modes' $\ResIoiAblKoopEight\%$. \textbf{(b)} No head class reaches the $0.8$ alignment target; greedy selection exhausts its budget at $\ResIoiNameMoverK/\ResIoiN$ modes, reaching about twice the random-subspace null ($\ResIoiNameMoverAlign$ against $\ResIoiNameMoverNull$ for name movers).}
\label{fig:circuits}
\end{figure*}

The measured eigenvector conditioning places these fits firmly in the non-normal regime that \Cref{thm:dissociation} concerns: $\kappa_2(\hat V) = \ResGptKappaV$ on GPT-2 small, $\ResGemmaKappaV$ on Gemma-2-2B and $\ResQwenKappaV$ on Qwen3-8B-Base. By \Cref{prop:dissociation-quant} and \Cref{rem:dissociation-numeric}, misalignment at these conditionings is already near-saturated. The Koopman spectrum captures behaviourally relevant structure but does not decompose computation into human-legible mechanisms: it is an identifiable coarse invariant of transformer depth dynamics.

\subsection{Transport versus encoding}
\label{sec:exp-transport}

If the two bases differ because one indexes \emph{encoding} and the other \emph{transport}, then PCA's margin should shrink with how far the question travels in depth. We test this directly: predict $\resid_{\ell+k}$ from a rank-$j$ linear read of $\resid_\ell$ - same $j$ directions, same source state, same unconstrained map onto the target - and sweep the depth gap $k$.

The design is chosen so that the margin cannot shrink for trivial reasons. The read-out does not saturate and has no privileged basis: PCA is optimal for reconstructing $\resid_\ell$ itself, which stops being the target once $k>0$. Each cell is normalised between the optimal rank-$j$ predictor ($0$) and a random subspace ($1$), so the margin cannot shrink merely because prediction gets harder with $k$.

Pooled over $\ResTransLayers$ layers and $j\in\{8,\dots,64\}$ ($\ResTransCells$ cells, $\ResTransPrompts$ prompts), PCA's advantage falls from $\ResTransGapOne$ at $k=1$ to $\ResTransGapFar$ at $k=\ResTransKFar$, a $\ResTransDecay\times$ decay, and it does so monotonically at every layer separately (raw ratio $\ResTransRatioOne\times \to \ResTransRatioFar\times$; \Cref{fig:transport}). PCA wins the question it is optimal for, by a margin that shrinks with depth-distance and is essentially gone by $k\approx8$.

Stated plainly: the advantage decays to parity, it does not reverse. We have not shown that Koopman modes are the better transport basis, only that PCA stops being one. Two causal tests of a reversal do not discriminate on this testbed. Projecting out the mode at every layer $\ell'\geq\ell$ so that it cannot carry forward leaves the ordering unchanged at the analysis layer ($\ResTransPropKoop\%$ against $\ResTransPropPca\%$ removed at $j=32$, versus $\ResTransLocKoop\%$/$\ResTransLocPca\%$ locally); at earlier layers it removes the whole effect regardless of basis, and the logit-lens readout at $\ell+k$ saturates the same way. Twelve layers is not enough depth to separate $k$ from $\ell$ under a behavioural readout - a limit of this testbed, not of the bases, which the rank-$j$ prediction avoids precisely by having no readout to saturate.

\begin{figure*}[t]
\centering
\includegraphics[width=\textwidth]{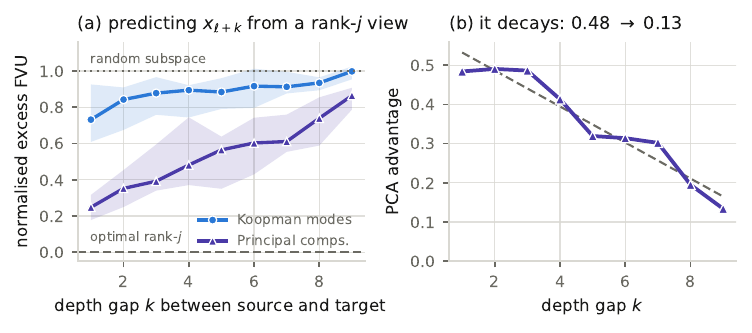}
\caption{\textbf{The PCA advantage is local and decays with depth.} Predicting $\resid_{\ell+k}$ from a rank-$j$ read of $\resid_\ell$ ($\ResTransLayers$ layers, $j\in\{8,\dots,64\}$, $\ResTransPrompts$ prompts), normalised between the optimal rank-$j$ predictor ($0$) and a random subspace ($1$). \textbf{(a)} Normalised excess fraction of variance unexplained against the depth gap $k$. \textbf{(b)} The PCA--Koopman gap falls from $\ResTransGapOne$ at $k=1$ to $\ResTransGapFar$ at $k=\ResTransKFar$, a $\ResTransDecay\times$ decay to parity.}
\label{fig:transport}
\end{figure*}

\subsection{The SAE invariance gap, and the selection rule it is conditional on}
\label{sec:exp-sae}

\Cref{cor:sae-non-identifiability} rests on a measurable claim: dictionaries failing (DR2) carry a large projection residual. We report the relative invariance residual
\begin{equation}
\label{eq:exp-residual}
    \hat\varepsilon_{\mathrm{rel}}
    \;=\;
    \frac{\bigl\|Y - \hat A_M X - \hat B_M \Xi\bigr\|_{F}}{\|Y\|_{F}},
\end{equation}
the empirical counterpart of the projection residual in \Cref{def:realisation}, which vanishes exactly when (DR2) holds.

\begin{table}[t]
\centering
\small
\caption{\textbf{SAE dictionaries sit furthest from Koopman invariance on every model - further than a random orthonormal control.} Relative invariance residual \eqref{eq:exp-residual} at matched $N = 32$, mid-depth analysis layer, all dictionaries whitened (without which the ordering inverts; see text). Lower is closer to satisfying (DR2). Conditional on the selection rule: features are chosen by activation frequency throughout, and the ordering \emph{reverses} under variance-based selection ($\ResSelVarianceRatio\times$ rather than $\ResSelFreqRatio\times$ the spectral residual).}
\label{tab:sae}
\begin{tabular}{lrrrr}
\toprule
Model & Spectral & Random & SAE & SAE/Spec. \\
\midrule
GPT-2 small   & $\ResGptEpsSpec$   & $\ResGptEpsRand$   & $\mathbf{\ResGptEpsSae}$   & $\ResGptSaeRatio\times$ \\
Gemma-2-2B    & $\ResGemmaEpsSpec$ & $\ResGemmaEpsRand$ & $\mathbf{\ResGemmaEpsSae}$ & $\ResGemmaSaeRatio\times$ \\
Qwen3-8B-Base & $\ResQwenEpsSpec$  & $\ResQwenEpsRand$  & $\mathbf{\ResQwenEpsSae}$  & $\ResQwenSaeRatio\times$ \\
\bottomrule
\end{tabular}
\end{table}

\Cref{tab:sae} shows the predicted ordering on every model, strengthening with scale ($\ResGptSaeRatio\times$, $\ResGemmaSaeRatio\times$, $\ResQwenSaeRatio\times$ the spectral residual), across three SAE families - ReLU \citep{bricken2023monosemanticity}, JumpReLU \citep{lieberum2024gemma}, TopK \citep{gao2024scaling} - and at every layer examined (\Cref{fig:sae}(a)). That SAE dictionaries sit \emph{further} from invariance than a random orthonormal basis, which has no mechanism by which it could be $\opK$-invariant, suggests the SAE objective does not merely fail to enforce (DR2) but selects against it: sparsity concentrates each feature on few inputs, while closure requires the span to absorb where those inputs are carried next. That is a stronger reading than the corollary asserts, and we flag it as such.

\paragraph{Instability, and one reversal.}
The residual measures the condition that fails, not the failure. \Cref{fig:sae}(b) measures the failure: spectra obtained through SAE dictionaries are less reproducible across disjoint calibration draws than spectral ones in $\ResStabNSaeWorse$ of $\ResStabNLayers$ layers, by factors of $\ResStabRatioLo$ to $\ResStabRatioHi$. This is the seed-to-seed variability the SAE literature reports \citep{karvonen2024saebench,braun2024identifying}, seen here on a \emph{fixed} dictionary with only the calibration sample varying - so it cannot be optimisation noise.

The ninth layer is worth stating plainly. At layer~26 of Qwen3-8B-Base the ordering reverses, and it does so because the spectral dictionary degrades rather than because the SAE improves: its instability is $\ResStabOutlierSpec$ there against $\ResStabSpecLo$--$\ResStabSpecHi$ everywhere else, while the SAE sits at $\ResStabOutlierSae$, in line with its own range. This is the deepest layer of the largest model, and it is the same breakdown of the random-Fourier construction at $d = 4096$ noted in \Cref{sec:exp-protocol}; the convergence measurement is taken at layer~\ResQwenMidLayer{}, where the construction is well behaved. The invariance residual, which does not depend on the stability estimate, keeps the predicted ordering at this layer as at all others.

\paragraph{A confound that inverts the result.}
Dictionaries must be compared at matched conditioning, not merely matched $N$. An SAE encodes mostly zeros at $N = 32$ - activation rate $0.089$ for GPT-2 against $\ResGptActRateSpec$ and $\ResGptActRateRand$ for the spectral and random dictionaries - and a mostly-zero target is trivially predictable, so $\hat\varepsilon_{\mathrm{rel}}$ would reward sparsity rather than invariance. Measured without this control, GPT-2's SAE residual is $0.287$, \emph{below} both the spectral ($0.301$) and random ($0.341$) dictionaries - apparently contradicting \Cref{cor:sae-non-identifiability}. Whitening removes the artefact: a change of basis leaves $\sigma(A)$ unchanged (\Cref{thm:existence}(b)) while equalising $\eta_{\obs}$ across families ($\ResGptEtaObsSae$--$\ResGptEtaObsRand$ on GPT-2 after whitening, against $7.4\times10^{-3}$ unwhitened). All figures in \Cref{tab:sae} are post-whitening; we report the uncorrected numbers because a reader reproducing this without the control will obtain them.

\begin{figure*}[t]
\centering
\includegraphics[width=\textwidth]{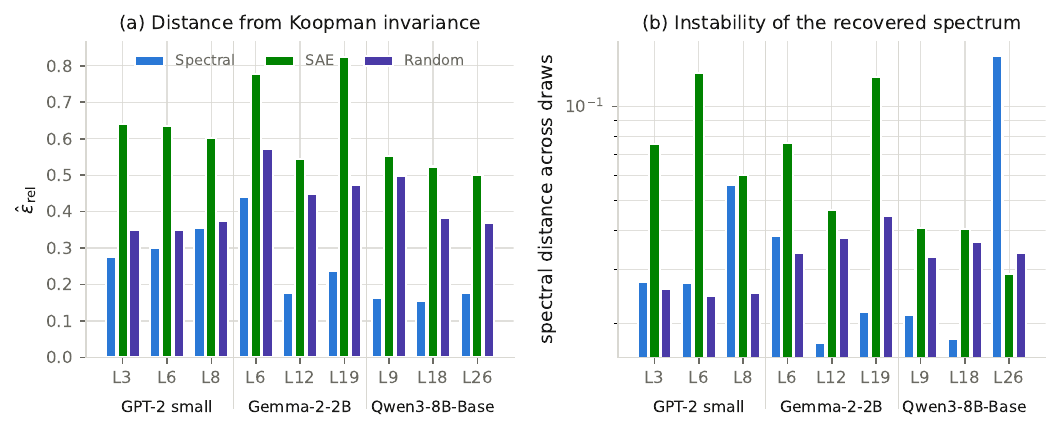}
\caption{\textbf{The (DR2) gap and the non-identifiability it produces.} \textbf{(a)} Distance from Koopman invariance, $\hat\varepsilon_{\mathrm{rel}}$ of \eqref{eq:exp-residual}, at matched $N=32$ with all dictionaries whitened, at three layers per model. SAE dictionaries sit furthest from (DR2) on every model and layer - further than a random orthonormal basis. \textbf{(b)} The consequence \Cref{cor:sae-non-identifiability} is about: matched spectral distance between realisations fitted on disjoint calibration draws (log scale; lower is more identifiable). Spectra from SAE dictionaries move more between draws in $\ResStabNSaeWorse$ of $\ResStabNLayers$ layers; the exception is Qwen3-8B-Base layer~26, where the \emph{spectral} dictionary is itself unstable.}
\label{fig:sae}
\end{figure*}

\paragraph{The pre-registered criterion fails, and the reason is the selection rule.}
A public SAE is $16$k--$65$k wide and this comparison runs at $N = 32$, so some rule must select roughly one latent in a thousand. Before running the sweep we registered the criterion that the ordering $\hat\varepsilon_{\mathrm{rel}}(\text{SAE}) > \hat\varepsilon_{\mathrm{rel}}(\text{Random}) > \hat\varepsilon_{\mathrm{rel}}(\text{Spectral})$ should hold with non-overlapping interquartile ranges in at least $80\%$ of (model, layer, $N$, rule) cells. It does not: the criterion is met in $\ResSelPct\%$ of $\ResSelCells$ cells.

The cause is the selection rule, and the effect is one of sign, not degree. Ranking latents by activation frequency (the rule behind \Cref{tab:sae}) or drawing uniformly among live latents puts the SAE above the spectral dictionary in every cell, by median factors of $\ResSelFreqRatio$ and $\ResSelRandAliveRatio$; ranking by mean activation magnitude does so in $12$ of $15$ cells ($\ResSelMagnitudeRatio$). But ranking by activation \emph{variance}, or greedily for variance explained, reverses the ordering in every cell, at $\ResSelVarianceRatio$ and $\ResSelGreedyVarRatio$ times the spectral residual. Across widths $N\in\{8,\dots,128\}$ and three layers, the sign tracks the selection rule and nothing else (\Cref{fig:selection}).

The reading is therefore narrower than \Cref{tab:sae} alone would support. The (DR2) gap is a real property of the subset a practitioner is most likely to read - the features that fire often - robust across models, layers, widths and all four encoder conventions (\Cref{sec:exp-encoder-variants}) for that subset. It is \emph{not} a property of the SAE's span as such: a variance-optimal subset is closer to Koopman invariance than our spectral construction, which is unsurprising in hindsight, since selecting for variance explained recovers something close to a principal-component basis and the spectral dictionary's linear block is exactly that. \Cref{cor:sae-non-identifiability} concerns what the objective fails to enforce and remains correct; what these measurements add is that the failure is unevenly distributed, so any claim about ``the'' invariance residual of an SAE must state its selection rule. Ours is activation frequency throughout.

\begin{figure*}[t]
\centering
\includegraphics[width=0.8\textwidth]{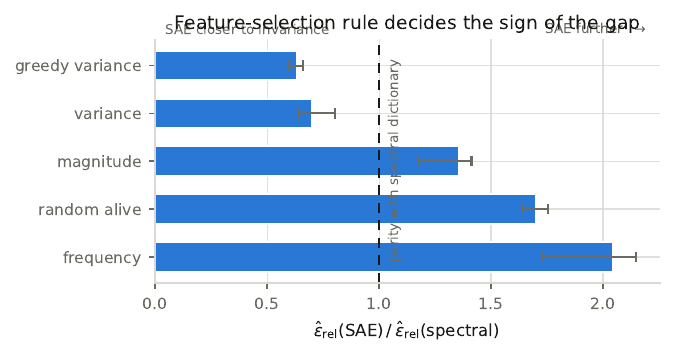}
\caption{The whole effect is conditional on the selection rule. Ratio of SAE to spectral invariance residual as a function of how $N$ features are selected from a $16$k--$65$k-wide SAE, pooled over three models, three layers and five widths; bars are medians with interquartile ranges and the dashed line is parity. Selecting the most frequently active latents - the rule behind \Cref{tab:sae} - puts the SAE at $\ResSelFreqRatio\times$ the spectral residual; selecting for activation \emph{variance} puts it at $\ResSelVarianceRatio\times$, reversing the ordering the corollary predicts.}
\label{fig:selection}
\end{figure*}

\subsection{Enforcing invariance: the penalty sweep}
\label{sec:exp-penalty}

\Cref{cor:sae-non-identifiability} makes a causal claim: SAE non-identifiability is \emph{caused} by the absence of a Koopman-invariance constraint. \Cref{sec:exp-sae} measures the correlate; here we add the invariance penalty of \eqref{eq:kinvariant-sae} and ask whether the outcome moves.

We train SAEs on cached GPT-2 layer-8 transitions at width $4d = 3072$, sweeping $\gamma \in \{0,10^{-3},10^{-2},10^{-1},1,10\}$ with three seeds each, screening every run for the penalty's degenerate minimisers and reporting excluded runs as excluded. TopK is primary because it fixes $L_0$ by construction: sweeping $\gamma$ at fixed sparsity is what makes any movement attributable to the invariance term rather than to the dictionary quietly becoming denser.

Raising $\gamma$ from $0$ to $\ResGammaTopkBest$ reduces the invariance residual from $\ResGammaTopkBaseEps$ to $\ResGammaTopkBestEps$ ($\ResGammaTopkEpsPct\%$), confirming that the penalty does what it is designed to do, and reduces the split-half spectral distance - identifiability itself, not a proxy for it - from $\ResGammaTopkBaseSplit$ to $\ResGammaTopkBestSplit$ ($\ResGammaTopkSplitPct\%$) at matched $L_0 = 32$. Intervening on the mechanism moves the outcome it predicts, which is the strongest support the corollary admits.

The cost appears in the same runs: reconstruction degrades (FVU $\ResGammaTopkBaseFvu \to \ResGammaTopkBestFvu$) and the live-feature fraction falls from $\ResGammaTopkBaseAlive$ to $\ResGammaTopkBestAlive$; at $\gamma = \ResGammaTopkCollapseGamma$ the dictionary collapses outright, all three seeds below the alive-feature guard at $\ResGammaTopkCollapseAlive\%$, which locates the usable range.

Cross-seed feature agreement moves the wrong way: mean max cosine similarity between dictionaries from different seeds falls from $\ResGammaTopkBaseMmcs$ to $\ResGammaTopkBestMmcs$. \textbf{The penalty makes the recovered \emph{spectrum} more reproducible while making the recovered \emph{feature dictionary} less so} - a second instance of the paper's central dissociation, arriving from a different direction.

It is also architecture-sensitive. On ReLU the invariance residual at the guard-selected operating point is not below baseline ($\ResGammaReluBaseEps \to \ResGammaReluBestEps$), though it does decline at larger $\gamma$ approaching collapse, while the split-half distance falls by $\ResGammaReluSplitPct\%$. The residual and the identifiability it proxies for can therefore decouple when sparsity is enforced by a penalty rather than by construction - which is why TopK is primary. This does not undermine the corollary, whose subject is the spectrum, but the penalty is not a drop-in remedy for the feature-level variability the SAE literature reports and we do not present it as one.

\begin{figure*}[t]
\centering
\includegraphics[width=\textwidth]{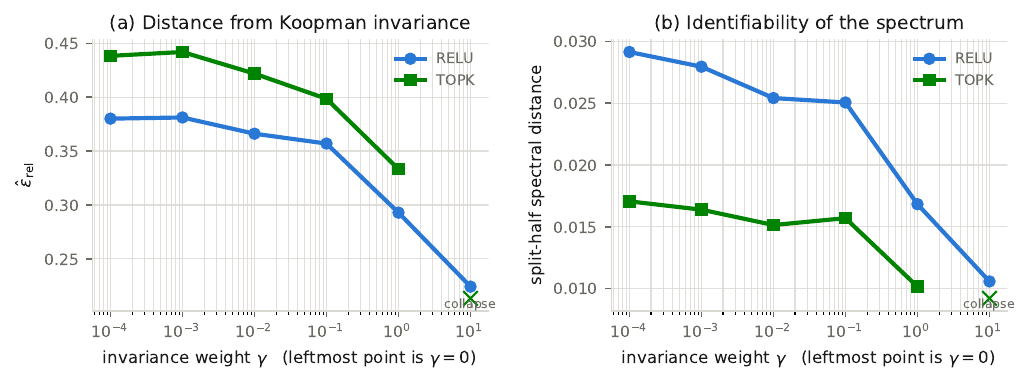}
\caption{Adding the invariance penalty improves spectral identifiability. \textbf{(a)} The invariance residual falls with $\gamma$ for TopK; for ReLU it is not below for TopK; for ReLU it is not below baseline at the guard-selected operating point, which the text takes up. at the guard-selected operating point, which the text takes up. \textbf{(b)} The split-half spectral distance - identifiability itself rather than a proxy - falls on both architectures, by $\ResGammaTopkSplitPct\%$ (TopK) and $\ResGammaReluSplitPct\%$ (ReLU) at matched $L_0$. Sparsity is held fixed across the sweep ($L_0 = 32$ by construction for TopK, $\lambda$ calibrated for ReLU), so the improvement cannot be attributed to the dictionary becoming denser. Crosses mark configurations excluded by the degeneracy guards: TopK collapses at $\gamma = \ResGammaTopkCollapseGamma$ with $\ResGammaTopkCollapseAlive\%$ of features alive, which locates the usable range.}
\label{fig:gamma}
\end{figure*}

We claim no more than this supports: not that invariance-regularised SAEs are the right tool for practice, which would need an evaluation we do not attempt, but that the mechanism \Cref{cor:sae-non-identifiability} names is the operative one for spectral identifiability, because intervening on it moves that outcome by a stated amount.

\subsection{Running the universality criterion}
\label{sec:exp-universality}

\Cref{cor:universality} is a decision procedure, and it had no experiment. Running it requires three things matched, none of them free: dictionary size, relative depth $\ell/L$ (the suite spans 12 to 36 layers, so an absolute index is not a common coordinate), and $M$ (the resolution floor falls with $M$). We fit at $N = \ResUniN$ and $M = \ResUniM$ on each model's own dictionary, at $\ell/L \in \{0.25, 0.5, 0.75\}$, and compare every pair's matched distance against the larger of the two models' split-half floors, measured on the same draws ($\ResUniFloorLo$--$\ResUniFloorHi$).

Because ``these spectra are distinct'' and ``this test separates everything'' produce the same table, we added two Pythia-160M seed replicas - same architecture, same data, same order, different initialisation - as models the criterion \emph{should} decline to separate.

It declares every pair distinct: $\ResUniCrossDeclared$ of $\ResUniCrossPairs$ cross-family pairs, and also $\ResUniSeedDeclared$ of $\ResUniSeedPairs$ seed-replica pairs, the closest of which still sits at $\ResUniSeedMin\times$ its floor (\Cref{fig:universality}). A rule that separates two runs of the same architecture on the same data is not a universality criterion, and reporting the cross-family half alone would have read as a clean success.

The cause is locatable, and it is not the theorem. The split-half floor bounds the \emph{sampling} error of a fixed dictionary, which is what $O(M^{-1/2})$ governs. A cross-model comparison carries a second error the floor cannot see: the dictionaries are fitted per model and are only approximately $\opK$-invariant, so what is compared are the spectra of two \emph{realisations} rather than of two operators, and that mismatch is a bias that does not shrink with $M$ (\Cref{rem:ident-approx}). \Cref{cor:universality} needs a calibrated null, not an estimation bound.

What survives is an ordering, and it names the fix. Distance ranks seed replicas below cross-family pairs with AUC $\ResUniAucMax$ on the maximum statistic that \Cref{thm:identifiability} bounds and $\ResUniAucWass$ on the robust Wasserstein summary, rising to $\ResUniAucWassMid$ at mid-depth, where seed pairs sit at a median $\ResUniSeedWass\times$ the floor against $\ResUniCrossWass\times$ for cross-family pairs ($\ResUniSeedRatio$ and $\ResUniCrossRatio$ on the max statistic). The spectrum separates architecture from seed relatively but not absolutely, so the operational form of the criterion should calibrate against seed replicas of the models being compared - a concrete correction to a corollary we stated, obtained only by running the control.

\begin{figure*}[t]
\centering
\includegraphics[width=\textwidth]{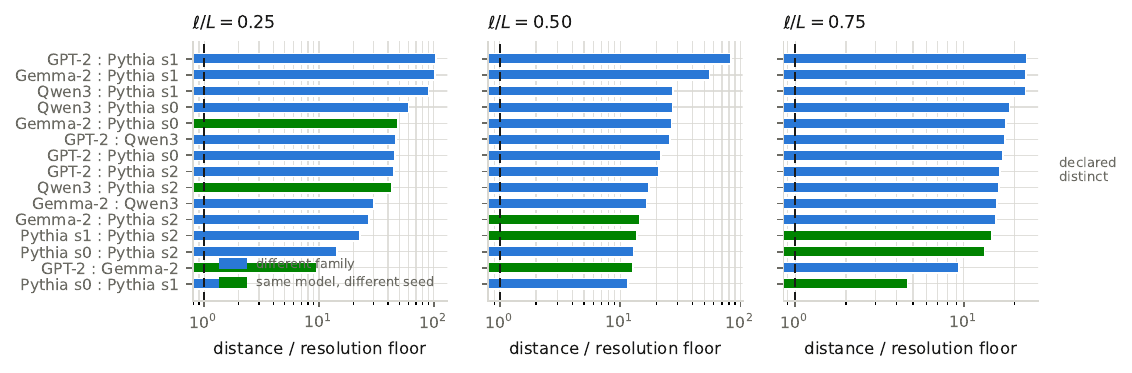}
\caption{The universality criterion separates everything, including the control it should not. Matched spectral distance for every model pair, in units of the split-half resolution floor measured on the same draws, at three relative depths. Blue bars are cross-family pairs, green bars are the Pythia-160M seed replicas that differ only in initialisation. Every bar lies to the right of the dashed floor at $1$, so the criterion in its stated form declares all $\ResUniCrossPairs$ cross-family pairs \emph{and} all $\ResUniSeedPairs$ seed-replica pairs distinct. The ordering is nonetheless informative - seed pairs sit lower on average (AUC $\ResUniAucWass$ on the Wasserstein summary)}
\label{fig:universality}
\end{figure*}

\subsection{The modelling conventions, measured}
\label{sec:exp-conventions}

\Cref{sec:dynamics-conventions} listed three conventions that define the estimand. This subsection reports the measurements behind each, plus the read-out and encoder-convention checks that the earlier experiments depend on.

\subsubsection{Depth homogeneity: per-layer against pooled fits}
\label{sec:exp-depth-homogeneity}

The realisation is fitted separately at each analysed layer throughout. Two measurements support that choice.

\emph{Adjacent-layer spectral drift.} For each model we fit $\hat A_\ell$ at every cached layer and compute the Hungarian-matched distance between adjacent layers, in units of the split-half resolution floor measured at the same $M$ on the same layer. The median ratio is $\ResGptDriftRatio$ on GPT-2 small, $\ResGemmaDriftRatio$ on Gemma-2-2B and $\ResQwenDriftRatio$ on Qwen3-8B-Base. A ratio above one means adjacent layers differ by more than the measurement can attribute to sampling, so the depth-indexed family $\ell \mapsto \sigma(A_\ell)$ is resolved rather than inferred.

\emph{Pooled fits predict worse.} Fitting one $A$ on transitions pooled across depth and evaluating on held-out transitions gives a relative residual of $\ResGptPooledResid$ on GPT-2 small against $\ResGptPerLayerResid$ for per-layer fits, a penalty of $\ResGptPoolPenalty\%$; the corresponding penalties are $\ResGemmaPoolPenalty\%$ (Gemma-2-2B, $\ResGemmaPooledResid$ against $\ResGemmaPerLayerResid$) and $\ResQwenPoolPenalty\%$ (Qwen3-8B-Base, $\ResQwenPooledResid$ against $\ResQwenPerLayerResid$). A pooled realisation estimates a depth-averaged object whose spectrum is an artefact of the pooling, and the size of the penalty bounds how large that artefact would be.

\Cref{thm:identifiability} is stated and proved at fixed $\ell$ and needs no depth-homogeneity, so nothing in the theory rests on this; what rests on it is the interpretation of the estimand.

\subsubsection{Exogeneity of the control, and the residualised realisation}
\label{sec:exp-exogeneity}

\Cref{def:recurrence} treats the attention write $\ctrl_\ell$ as an exogenous input, but $\ctrl_\ell = \mathrm{Attn}_\ell(\mathrm{RMSNorm}(\resid_\ell^{(1:T)}))$ is computed from the same residual stream. This is a modelling convention, and it is not innocuous.

\emph{How far from exogenous.} At mid-depth the largest canonical correlation between the lifted state $\obs(\resid_\ell)$ and the control is $\rho_1 = \ResGptRhoOne$ on GPT-2 small, $\ResGemmaRhoOne$ on Gemma-2-2B and $\ResQwenRhoOne$ on Qwen3-8B-Base - so some direction of the control is nearly a function of the state. Overall, however, the controls are mostly not explained by the state: regressing $\ctrl_\ell$ on $\obs(\resid_\ell)$ gives $R^2 = \ResGptCtrlRsq$, $\ResGemmaCtrlRsq$ and $\ResQwenCtrlRsq$ respectively. The dependence is concentrated in a few directions rather than spread across the control space.

\emph{The residualised realisation, and how far it moves.} The exogenous-by-construction alternative removes the state-explained part of the control before fitting (Frisch--Waugh--Lovell): regress $\ctrl_\ell$ on $\obs(\resid_\ell)$, keep the residual, and fit the realisation against that. The recovered spectrum moves by $\ResGptExogShift\times$ the split-half floor on GPT-2 small, $\ResGemmaExogShift\times$ on Gemma-2-2B and $\ResQwenExogShift\times$ on Qwen3-8B-Base. Those shifts are well above the resolution floor, so the convention selects between two genuinely different estimands at the precision we can measure - it is not a numerical detail.

We report the naive realisation as primary because it is the one \Cref{def:recurrence} defines, and the residualised realisation alongside it. A reader who prefers the exogenous-by-construction estimand should read the latter throughout. We regard this as the sharpest open modelling question in the framework and do not resolve it by fiat.

\subsubsection{Depth stationarity and the norm observable}
\label{sec:exp-stationarity}

The residual-stream norm grows with depth, by a factor of $\ResGptNormGrowth$ across the analysed range on GPT-2 small, $\ResGemmaNormGrowth$ on Gemma-2-2B and $\ResQwenNormGrowth$ on Qwen3-8B-Base. A law whose second moment moves with depth cannot be a single invariant law, so the realisation is defined relative to the \emph{layer marginal} $\mu_\ell$ rather than to one invariant $\mu$. The per-layer fit of \Cref{sec:exp-depth-homogeneity} already enforces this; no separate correction is applied.

\emph{The growth appears as its own mode, which it need not have.} If the framework is describing the depth dynamics rather than merely fitting them, the norm growth should be recoverable as a dynamical mode once the dictionary can express it. Adding $\log\|\resid_\ell\|$ as an observable does exactly that on every model tested: the largest expanding eigenvalue of the augmented realisation is $\ResGptNormMode$ (GPT-2 small), $\ResGemmaNormMode$ (Gemma-2-2B) and $\ResQwenNormMode$ (Qwen3-8B-Base), against measured per-layer norm growth factors of $\ResGptNormEmpirical$, $\ResGemmaNormEmpirical$ and $\ResQwenNormEmpirical$ respectively. The agreement is within a few percent on all three, and it is a prediction that could have failed: nothing forces a fitted eigenvalue to match an independently measured growth rate.

\subsubsection{The full-state read-out}
\label{sec:exp-readout}

Koopman modes live in observable space $\C^{N}$. Every question asked of them in \Cref{sec:exp-circuits} - alignment with a head's write subspace, contribution to a logit difference, the effect of ablating them - is a question about directions in the residual stream $\R^{d}$. The correspondence between the two is fitted, not assumed, and its quality is reported because every mode-level claim in that section is void if it is poor.

\emph{Construction.} On held-out states we fit the linear read-out
\begin{equation}
\label{eq:app-readout}
    \resid - \mu \;\approx\; \obs(\resid)\,R,
    \qquad R \in \C^{N \times d},
\end{equation}
by least squares with a small ridge for numerical hygiene (the lifted features are whitened, so this is not regularisation in any statistical sense). Writing $\obs(\resid) = \sum_k v_k \varphi_k(\resid)$ in the eigenbasis of the fitted $\hat A$, with $\varphi_k(\resid) = \phi_k^{\top}\obs(\resid)$ the matching left eigenfunctions, the \emph{full-state Koopman mode} is
\begin{equation}
    \xi_k \;=\; R^{\top} v_k \;\in\; \R^{d},
    \qquad
    \resid - \mu \;\approx\; \sum_k \xi_k \,\varphi_k(\resid).
\end{equation}
Modes come in conjugate pairs for a real operator, so any subset that splits a pair leaves an imaginary residue; taking the real part is the correct projection back to the residual stream and is what the causal ablation applies. Interventions are therefore always applied to whole conjugate groups.

\emph{Why fit the read-out rather than invert the dictionary.} Pushing an eigenvector back through the dictionary's internal linear block would be cheaper, but it works only for dictionaries that have an invertible linear block. Fitting \eqref{eq:app-readout} applies to \emph{any} dictionary, including SAE encoders, which is what makes it possible to run the same circuit analysis through an SAE dictionary and compare. It is also honest about what it claims: the reconstruction error is measured rather than inherited from an assumption about the dictionary's structure.

\emph{Measured quality.} On the IOI activations of \Cref{sec:exp-circuits} the read-out attains a median relative reconstruction error of $\ResIoiReadout$ on held-out states. This is the precondition for that section's negative results to mean anything: a broken read-out would produce the same appearance of modes failing to resolve the circuit, and could not be told apart from the finding.

\subsubsection{Encoder variants for the SAE dictionary}
\label{sec:exp-encoder-variants}

\Cref{cor:sae-non-identifiability} is a statement about the span of a learned dictionary, and which map from activations to codes one calls ``the dictionary'' decides what is measured. The main text fixes the convention - the post-activation code $\obs_D(\resid) = \sigma(W_{\mathrm{enc}}\resid + b_{\mathrm{enc}})$ with the autoencoder's own nonlinearity - and we report the other three here so the choice is visible rather than implicit. All four are measured on GPT-2 small at layer~\ResVarLayer{} with $N = \ResVarN$, whitened, on the same activations, with the spectral and random dictionaries as reference.

\begin{table}[h]
\centering
\small
\caption{Relative invariance residual $\hat\varepsilon_{\mathrm{rel}}$ under four encoder conventions, GPT-2 small at layer~\ResVarLayer{}, $N = \ResVarN$.}
\label{tab:encoder-variants}
\begin{tabular}{lr}
\toprule
Encoder convention & $\hat\varepsilon_{\mathrm{rel}}$ \\
\midrule
Post-activation code (main text) & $\ResVarNlPost$ \\
Pre-activation linear map $W_{\mathrm{enc}}\resid$ & $\ResVarLPre$ \\
Decoder columns & $\ResVarDec$ \\
Reconstruction $D z(\resid)$ & $\ResVarRec$ \\
\midrule
Spectral dictionary & $\ResVarSpec$ \\
Random orthonormal & $\ResVarRand$ \\
\bottomrule
\end{tabular}
\end{table}

The ordering that \Cref{tab:sae} reports holds under all four conventions: the SAE-derived dictionary carries the largest residual in every case. The pre-activation variant is reported only as a reference floor and is \emph{not} evidence for the corollary, because a linear map of the state cannot be $\opK$-invariant for nonlinear $F$ - (DR2) fails by construction there, so that row could not have come out any other way. The post-activation code is genuinely nonlinear and is therefore not excluded a priori, which is what makes its failure a measurement.

\subsection{Per-model report: Qwen3-8B-Base}
\label{sec:exp-qwen}

Because the largest model in the suite is the one that attains the predicted rate, its measurements are worth collecting in one place. All are taken at layer~\ResQwenMidLayer{} of $36$, with $N = 32$, $p = 64$, and no hyperparameter retuned for this model.

\begin{itemize}[leftmargin=1.6em,itemsep=2pt]
\item \emph{Convergence.} Spectral dictionary $\ResQwenSpecSlope \pm \ResQwenSpecSlopeSE$, random dictionary $\ResQwenRandSlope \pm \ResQwenRandSlopeSE$, at $M_{\max} = \ResQwenSpecMmax$ - the former within one standard error of the predicted $-1/2$, and attained at $M_{\max}/M_0^{\mathrm{eig}} = \ResQwenSpecMRatio$, i.e.\ below the gap-free threshold.
\item \emph{Invariance residual.} Spectral $\ResQwenEpsSpec$, random $\ResQwenEpsRand$, SAE $\ResQwenEpsSae$ - the largest SAE-to-spectral ratio in the suite at $\ResQwenSaeRatio\times$, consistent with the gap strengthening with scale.
\item \emph{Conditioning.} $\kappa_2(\hat V) = \ResQwenKappaV$, by far the most non-normal fit in the suite and the one for which \Cref{rem:dissociation-numeric}'s saturation calculation is calibrated.
\item \emph{Control rank.} The attention-write covariance has effective rank 10.8 against an ambient dimension of $4096$. This measurement is independent of the dictionary and is the direct empirical support for the premise of \Cref{cor:transformer-reduction}.
\item \emph{Norm mode.} Largest expanding eigenvalue $\ResQwenNormMode$ against measured per-layer norm growth $\ResQwenNormEmpirical$.
\item \emph{Where it breaks.} At layer~26 the random-Fourier construction at $d = 4096$ degrades and the split-half stability estimate becomes unreliable ($\ResStabOutlierSpec$ against $\ResStabSpecLo$--$\ResStabSpecHi$ elsewhere). This is a dictionary-construction failure, not a model or estimator failure: the random dictionary run through the same pipeline on the same cached activations behaves normally, and the invariance residual, which does not depend on the stability estimate, keeps the predicted ordering at that layer.
\end{itemize}

\section{Discussion}
\label{sec:discussion}

Lifting the depth recurrence through the Koopman operator gives a realisation whose spectrum is a coordinate-free invariant of the transformer (\Cref{thm:existence}), identifiable from $M$ calibration samples at the rate $M^{-1/2}$ up to permutation (\Cref{thm:identifiability}). \Cref{thm:minimax} shows no estimator does better in $M$, and \Cref{thm:robust} covers heavy-tailed activations. The spectrum converges on every model tested and attains the predicted exponent on the largest.

\Cref{thm:dissociation} says what this object is \emph{not}. A non-normal realisation forces the activations' principal directions apart from the Koopman modes, and the measured conditioning ($\kappa_2(\hat V)$ from $\ResGemmaKappaV$ to $\ResQwenKappaV$) puts every fit in that regime. The experiments land where the theorem requires: the modes beat random directions and lose to principal components at resolving IOI, with the loss confined to the question principal components are optimal for and decaying by $\ResTransDecay\times$ as the question moves away in depth. What transports across depth is not the basis in which any one layer is encoded. The spectrum therefore certifies an intrinsic, identifiable model property recoverable at a stated rate; it does not certify interpretability, behavioural coverage, or robustness.

The limitations are as follows, in rough order of how much they constrain the claims. First, the estimand depends on modelling conventions we cannot fully justify: residualising the attention write against the lifted state moves the spectrum by $\ResGptExogShift$--$\ResGemmaExogShift\times$ the resolution floor (\Cref{sec:exp-exogeneity}), so the naive and residualised conventions specify different estimands at the precision we can measure. We report both. Second, invariance holds only approximately. All theorems assume \Cref{ass:invariance}; real dictionaries satisfy it approximately and the resulting bias does not vanish as $M \to \infty$ (\Cref{rem:ident-approx}). KSA is thus not model-agnostic - the dictionary must be co-designed with the architecture - though within a fixed dictionary family it is uniform across corpus, seed and hyperparameter, which is the property SAEs lack. Third, diagonalisability and separation are generic but not universal: Jordan extensions are routine but degrade the eigenvector rate (\Cref{rem:jordan}), and \Cref{ass:separation} can fail when mechanisms occupy closely spaced spectral scales. \Cref{cor:cluster-ident} covers near-degeneracy but not exact coincidence, which we conjecture is measure-zero but have not proved. Fourth, the intervention calculus is proved but not validated: \Cref{thm:completeness} covers first-order interventions only, and whether its closed-form representatives predict measured patching effects is untested. Fifth, the rate is asymptotic in a way that bites - large $\kappa_0$ or small $\Delta$ requires much data before the $M^{-1/2}$ regime begins, and $c_0$ in \eqref{eq:M0-eig} is unknown, so $M_0^{\mathrm{eig}}$ is an order-of-magnitude guide. Sixth, the random-Fourier bandwidth fixed for $d \leq 2304$ degrades at $d = 4096$ in the deepest layers (\Cref{sec:exp-qwen}); we chose not to retune per model (\Cref{sec:exp-protocol}). Finally, two stated results fail their controls: the pre-registered SAE-gap criterion is met in $\ResSelPct\%$ of $\ResSelCells$ cells rather than the registered $80\%$, because the sign of the effect tracks the feature-selection rule (\Cref{sec:exp-sae}); and the universality criterion separates two seed replicas of one architecture, so in its stated form it is not a universality criterion (\Cref{sec:exp-universality}). We report both because both bound what may be claimed.

Several problems remain open. The upper bound carries $\sqrt{N + \log(1/\delta)}$ while the lower bound is dimension-free (\Cref{thm:minimax}); sharpening either settles the dimension dependence. The penalty of \eqref{eq:kinvariant-sae} improves spectral identifiability by $\ResGammaTopkSplitPct\%$ while degrading reconstruction and cross-seed agreement (\Cref{sec:exp-penalty}), leaving open whether a dictionary can be sparse, reconstructive and $\opK$-invariant at once. A universality criterion needs a null calibrated against seed replicas rather than the sampling floor. The exogeneity convention needs resolving, and the intervention calculus needs testing against measured patching effects.

The safety-case programme of \cite{clymer2024safety} requires naming a mechanism $Y$ responsible for a behaviour $X$ and then defending that attribution. \Cref{thm:identifiability} supplies the missing piece: a certificate that $Y$ is a property of the model rather than of the procedure that found it. It does not address behavioural coverage, distribution shift, or adversarial robustness, and it does not deliver a human-legible decomposition - \Cref{thm:dissociation} says that in the non-normal regime these models occupy, the identifiable object and the legible object cannot be the same object. A safety case needing both will need two tools, and should say which one it is using where.

\section{Conclusion}
\label{sec:conclusion}

We have put mechanistic interpretability on an identifiability footing. Treating a transformer forward pass as a controlled dynamical system in depth and lifting it with the Koopman operator produces a finite linear realisation whose spectrum is a coordinate-free invariant of the model, identifiable from finite calibration data at the optimal $M^{-1/2}$ rate up to permutation, with a matching lower bound, a heavy-tailed variant, an algebraically complete intervention calculus, and an instance-dependent reduction certificate.

The measurements confirm the theory where it can be confirmed and bound it where it cannot. The spectrum converges on GPT-2 small, Gemma-2-2B and Qwen3-8B-Base, attaining the predicted exponent on the largest; the SAE invariance gap the theory predicts is observed on every model and is causally movable by the penalty the theory motivates; and the modes fail to resolve the IOI circuit in exactly the way a non-normal realisation must. The Koopman spectrum is an identifiable, model-intrinsic fingerprint of transformer depth dynamics with a stated error bar, not a decomposition into human-legible mechanisms. Both halves of that sentence are results.

\bibliographystyle{unsrt}
\bibliography{references}

\end{document}